\documentclass{article}

\PassOptionsToPackage{numbers,sort&compress}{natbib}
\usepackage[preprint]{neurips_2026}

\usepackage[utf8]{inputenc} %
\usepackage[T1]{fontenc}    %
\usepackage{url}            %
\usepackage{booktabs}       %
\usepackage{amsfonts}       %
\usepackage{nicefrac}       %
\usepackage{microtype}      %
\usepackage{fvextra}
\usepackage{textcomp}
\input{code/highlight_style.tex}
\usepackage{subfigure}
\usepackage[dvipsnames]{xcolor}
\usepackage[colorlinks=true,allcolors=NavyBlue]{hyperref}
\hypersetup{
  pdftitle={Boule or Baguette? A Study on Task Topology, Length Generalization, and the Benefit of Reasoning Traces},
  pdfauthor={William L. Tong; Ege \c{C}akar; Cengiz Pehlevan}
}

\DeclareUnicodeCharacter{22A2}{$\vdash\!\!$}
\DeclareUnicodeCharacter{2228}{$\!\!\vee\!\!$}
\DeclareUnicodeCharacter{2227}{$\!\!\wedge\!\!$}
\DeclareUnicodeCharacter{271D}{t}

\usepackage{amsmath}
\usepackage{amssymb}
\usepackage{mathtools}
\usepackage{amsthm}
\usepackage{bbm}

\newcommand{\E}{\mathbb{E}}

\newcommand{\N}{\mathbb{N}}

\newcommand{\R}{\mathbb{R}}
\newcommand{\vr}[1]{\mathbf{#1}}

\renewcommand{\d}{\partial}

\newcommand{\Px}[1]{\left(#1\right)}
\newcommand{\Hx}[1]{\left[#1\right]}

\DeclareMathOperator*{\argmax}{arg\,max}

\DeclareMathOperator*{\Var}{\mathrm{Var}}

\theoremstyle{plain}
\newtheorem{theorem}{Theorem}[section]
\newtheorem{proposition}[theorem]{Proposition}

\theoremstyle{definition}

\theoremstyle{remark}

\title{Boule or Baguette? \\ A Study on Task Topology, Length Generalization, and the Benefit of Reasoning Traces}

\author{%
  William L.~Tong\thanks{Contact: \texttt{wtong@g.harvard.edu}}\\
  Harvard University\\
  \And
  Ege \c{C}akar \\
  Harvard University\\
  \And
  Cengiz Pehlevan\\
  Harvard University
}

\begin{document}

\maketitle

\begin{abstract}
Recent years have witnessed meteoric progress in \textit{reasoning models}: neural networks that generate intermediate \textit{reasoning traces} (RTs) before producing a final output. Despite the rapid advancement, our understanding of how RTs support reasoning, and where this paradigm fails, remain incomplete. To promote greater clarity, we introduce PITA: a novel large-scale dataset of over 23 million statements in propositional logic and their corresponding proofs. As a benchmark for robust reasoning, we focus on length generalization: if a model is trained to determine truth or falsity on statements with proofs up to fixed length, how well does it generalize to statements requiring longer proofs? We propose notions of (1) \textit{task depth} and (2) \textit{task breadth}, which measure respectively (1) the number of proof steps required to solve a proposition and (2) the number of unique propositions within a task family. We vary these quantities across subsets of PITA, and find that RT models generalize well on relatively \textit{broad} and \textit{shallow} subsets, while deteriorating on relatively \textit{narrow} and \textit{deep} subsets compared to non-RT baselines. As a controlled point of comparison, we study a separate, tractable transitive inference task that exhibits qualitatively similar behavior. Our accompanying theory explains the scalings observed in this simpler setting, suggesting one mechanism by which breadth can favor RT models while depth can expose long-context weaknesses. Our findings suggest salient benefits and limitations of proof-like reasoning traces in controlled length-generalization settings.
\end{abstract}

\section{Introduction}
In just a few years, neural networks equipped with reasoning traces (RTs) have attained top performance at international math competitions \cite{gemini_imo_gold,alphageometry2}, threatened to displace software engineers \cite{swe_bench,belzner_llm_assist_swe,gpt_labor_market}, saturated dozens of benchmarks designed to measure human-level reasoning \cite{achiam2023gpt_report,team2023gemini,deveci_ouroboros_benchmarks}, and accomplished among many more astonishing feats \cite{chu_navigate_lab_cot}. Yet despite the clear empirical success of these methods, debate abounds regarding how and why reasoning traces improve performance. Perhaps RTs allow a model to decompose a difficult problem into simpler steps \cite{wei_cot_og} or allow the model to catch errors and backtrack \cite{guo2025deepseek}. At the same time, others have observed that the content of generated reasoning traces are not always relevant to solving the task \cite{turpin_lm_not_say_think,wang2023towards_understand_what_matters,yee2024faithful_and_unfaithful}, reasoning models are highly susceptible to distribution shifts \cite{zhao2025cot_mirage,stechly2024cot_planning}, and in some cases, less reasoning paradoxically appears to be better \cite{wu_more_is_less}. With growing awareness around the surprising and sometimes dangerous limitations of reasoning models \cite{song2025survey_llm_fails}, combined with the rapid introduction of this technology across mainstream services, it is more important than ever to characterize when reasoning traces help, and when they introduce new failure modes.

To this end, we introduce \textbf{PITA}: a large-scale \textbf{p}ropositional \textbf{i}nference \textbf{ta}sk consisting of over 23 million statements and 95 billion tokens in propositional logic\footnote{Our dataset is available on Hugging Face at \url{https://huggingface.co/datasets/williamtong105/pita}.}. Propositional logic offers a clean formal setting where proof-like reasoning traces can be generated, verified, and varied at scale. To ensure rigor and avoid confounds associated with natural language, we formalize our dataset in Lean \cite{de_moura_lean}. For each statement in PITA, a model must determine whether the statement is true or false. The corresponding proof presents a natural reasoning trace determining the statement's veracity. We compare RT models that are trained using proofs against \textit{direct prediction} (DP) models, which are not provided proofs. To test the robustness of reasoning, we evaluate length generalization: models are trained on statements with proofs up to a fixed length, then tested on statements with longer proofs. We illustrate PITA in Figure \ref{fig:pita}.

In PITA, we define \textit{task depth} and \textit{task breadth}, which relate to the task's topology. Depth measures the number of steps required to prove a statement. Breadth measures the number of unique propositions for a given size, where size is determined by the number of propositional atoms. These measures are concrete in PITA, and serve as analogues for more general properties of reasoning tasks: how many steps are required to solve an example, and how diverse the space of possible examples is. 

We divide PITA into several splits with varying depth and breadth, and find that RT models generalize most effectively on \textit{broad} and \textit{shallow} splits (``boule"-shaped), while deteriorating on \textit{narrow} and \textit{deep} splits (``baguette"-shaped) relative to direct prediction models. Given that RTs are presumed to benefit most on tasks with step-wise structure \cite{scratchpad,prystawski_step_by_step}, it is particularly surprising that an RT model \textit{underperforms} the DP baseline on deep splits. 

To determine whether these outcomes are idiosyncratic to PITA, we compare our results to a simpler task based on \textit{transitive inference} (TI). TI captures a minimal form of multi-step symbolic inference while being substantially more tractable to study, offering a controlled companion domain for reasoning about task topology and reasoning traces. Using arguments based on margin maximization in neural networks \cite{chizat_implicit_bias_of_gd,morwan_feature_learning}, we explain empirically observed scalings that relate task breadth, depth, and model width in TI. We do not treat this analysis as a faithful mechanistic model of the LLMs trained on PITA. Rather, TI demonstrates that a similar topology-dependent RT/DP tradeoff can arise in a separate setting where the mechanism can be analyzed directly.

\subsection{Related work}
Reasoning traces entered mainstream practice with the introduction of chain-of-thought (CoT) prompting \cite{wei_cot_og,scratchpad}. Originally, CoT was an inference-time technique where a language model is prompted to ``think step-by-step," producing a reasoning trace that guides it towards the correct answer. It has since grown to encompass a diversity of both inference- and training-time interventions that equip a neural network with some capacity to use or generate reasoning traces \cite{chu_navigate_lab_cot}. In common use, RT and CoT are almost synonymous, but in this study we prefer ``RT" to avoid the connotation of inference-time prompting.

Several theoretical studies have sought to explain how RTs benefit neural networks. Allowing a Transformer to generate intermediate tokens increases its expressivity, an advantage that can be described precisely using circuit complexity theory \cite{feng_og_revealing,li_serial_probs,merrill_expressive}. Others have observed that reasoning traces confer statistical advantages when learning tasks that would otherwise require a large number of training examples or large model size \cite{abbe_globality_barrier,kim_transformers_parity_cot,hu_stat_found_cot,wies_sub_task_decomp}. Generating RTs may also be interpreted as a form of search, enabling the model to more effectively explore the task space \cite{kim_metastable_cot,gandhi_stream_search,prystawski_step_by_step}.

RTs are also thought to benefit length generalization \cite{huang_prov_cot_len_gen,ahuja_prov_len_and_comp_gen,xiao_theory_len_gen,anil_exp_len_gen_llm}. In general, length generalization is challenging for Transformer models, and depends highly on training procedure and architectural choices like positional encodings \cite{zhou_alg_len_gen,zhou_transformers_len_gen_non_robust,kazemnejad_nope}. Transformers also struggle with very large input contexts \cite{kuratov_babilong,hsieh_ruler,roberts_needle_threading,tay_long_range_arena}, potentially hindering their ability to generalize when tasks require long reasoning traces. Indeed, recent evidence suggests that prolonged reasoning traces can be harmful \cite{wu_more_is_less,hassid_short_chains,ghosal_more_thinking,feng_eff_reasoning_len}. We continue this discussion by connecting length generalization with task topology, clarifying cases where RTs benefit or harm.

Our tasks focus on propositional logic and transitive inference, domains that probe controlled forms of multi-step symbolic reasoning. Logical reasoning has been studied across many settings in both natural language \cite{saparov_pronto_qa,tafjord_proofwriter,clark_soft_reason,han_folio} and formal language \cite{pan_logic_lm,xu_sym_cot,pan_3sat,an_propl}. Some have compared logical reasoning to graph operations \cite{abbe_globality_barrier,minegishi_topology,sanford_graphs,wang_comp_fun,prystawski_step_by_step}. Related to our work, \cite{abbe_globality_barrier} formalize an in-context graph connectivity task, finding that CoT interventions enable greater learning efficiency. \cite{minegishi_topology} interpret the reasoning process as graph traversal. By inferring graphs from hidden representations, they identify correlations between topological properties of the graph and the models' performance. We continue this analogy between logical reasoning and graph operations, with a focus on properties of the task-induced graph that support or hinder length generalization.

Transitive inference (TI) is a specific case of graph-based approaches to reasoning. A classic cognitive task in humans and animals \cite{piaget_ti,apa}, TI presents a set of stimuli together with an ordering relation. During training, the subject observes pairwise relations between stimuli. At test time, the subject must transitively infer the correct relation between unobserved pairs. TI is a well established test of relational reasoning in psychology \cite{vasconcelos_ti,apa}, and is witnessing increasing application to neural networks \cite{lippl_ti,kay_ti,stachenfeld_ti}. We build on this tradition by employing TI as a tractable yet phenomenologically rich setting for measuring reasoning in complex models.

\section{\texorpdfstring{PITA: a \underline{p}ropositional \underline{i}nference \underline{ta}sk}{PITA: a propositional inference task}}
PITA is a large-scale synthetic dataset of propositional logic, formalized in Lean. To the best of our knowledge, PITA is the largest dataset of Lean theorems to date, with over 23 million statements and 95 billion tokens. We use PITA to evaluate logical reasoning in LLMs at scale, uncovering surprising generalization differences between RT models and direct prediction baselines. Figure \ref{fig:pita} illustrates the format of PITA. The dataset consists of four splits based on syntactic families of propositions, described below. To construct an example, we sample a statement under the syntactic constraints, construct a proof using a \textit{focused proof search} procedure \cite{liang2009focused_proof_search,an_propl}, translate into Lean code, and then tokenize. For a brief primer on propositional logic and proof systems, see Appendix \ref{app:prop_logic}.

\begin{figure*}
\vspace{-0.05in}
    \centering
    \includegraphics[width=\linewidth]{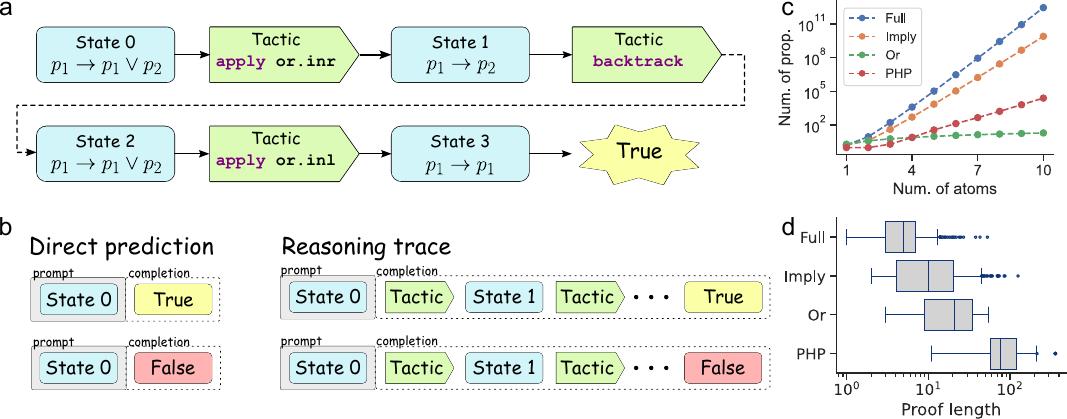}
    \caption{\textbf{The PITA dataset.} \textbf{(a)} Statements and proofs are expressed in Lean. Tactics are Lean commands that transform sets of propositions. A proof is an alternating sequence of tactics and proof states. \textbf{(b)} Inputs are formatted as prompt-completion pairs, where losses are computed only on completions. The DP model's completion consists only of the classification token. The RT model's completion includes the whole proof. \textbf{(c)} Breadth of each split, plotted as the number of unique examples per statement size (the number of atoms in that statement). \textbf{(d)} Depth of each split, plotted as the number of unique proof states. Boxplots are constructed from 500 samples from each split, and illustrate the median and quartiles. Outliers are determined from 1.5 times the inter-quartile range.}
    \label{fig:pita}
    \vspace{-0.05in}
\end{figure*}

\subsection{PITA splits}
We organize PITA into four dataset splits, with varying \textit{topology} determined by the underlying reasoning process. The two metrics we use to describe topology are \textit{task depth} and \textit{task breadth}. Depth measures the number of steps required to prove a proposition in a split, counted as the number of unique proof states between the proposition and final classification. This count includes proof states added by backtrack demonstrations\footnote{The proof-search procedure is a deterministic depth-first search with a fixed order of alternatives. The same statements therefore always generate the same proofs, with potential backtracks.}. Doing so more accurately reflects the real volume of computation required to prove a proposition, as compared to an idealized measure based on the shortest path (which is unknown at the start).  Breadth measures the number of unique propositions for a given size within a split, where size is counted by the number of atoms in the proposition. For example, consider the \textsc{Imply} split, which consists of propositions that have only the \textit{implication} ($\rightarrow$) connective. For statements with two atoms, $p \rightarrow q$ and $q \rightarrow p$ count as the same statement under variable renaming, while $p \rightarrow p$ is distinct because the variable-occurrence pattern differs. \footnote{We quotient variable renamings and, for commutative/associative connectives \{$\vee,\wedge$\}, permutations and associations. This avoids artificial breadth changes from notational variants; see Appendix \ref{app:pita_splits} for the precise equivalence relation.}

Depth and breadth correspond to natural descriptions of topology in general problem-solving tasks. Depth relates to the number of steps required to solve a task. Breadth relates to the size of the task space. These are general and intuitive properties of any task with step-wise structure, and historically studied in great depth for classic search and planning algorithms in AI \cite{russell2016ai}. Outside PITA, their exact measurement will depend on the domain; in our propositional logic task, we make both quantities explicit and controllable.

We investigate how task topology influences length generalization in reasoning trace models, compared against direct prediction baselines. We construct four dataset splits delineating different subsets of propositional logic, with varying depth and breadth:

\begin{itemize}
    \item \textbf{\textsc{Full}}. The full space of propositional statements with no restrictions. The broadest split.
    \item \textbf{\textsc{Imply}}. The implicational fragment of propositional logic: only implication-connectives are permitted. Implication is functionally complete \cite{enderton2001logic}: every proposition in \textsc{Full} can also be expressed in \textsc{Imply}. Also relatively broad.
    \item \textbf{\textsc{Or}}. Statements that consist of only or-connectives in the consequent. Propositions of this form are equivalent to a membership check: does the proposition in the antecedent occur in the consequent? The narrowest split.
    \item \textbf{\textsc{PHP}}. The Pigeonhole Principle (PHP) states that given $m$ pigeons and $n$ holes, if $m > n$, then at least one hole contains two or more pigeons (see Eq \eqref{eq:php} for the exact statement in propositional logic). We use a harder variation in which pigeons are constrained to occupy subsets of available holes. These propositions require proofs that can be exponentially long \cite{razborov2001php,haken1985php_hard,balsiger2000php_benchmarks}. This split is the deepest, and relatively narrow.
\end{itemize}

We plot the depths and breadths of each split in Figure \ref{fig:pita}. \textsc{Full} and \textsc{Imply} are \textit{boule-shaped} splits, significantly broader and generally shallower than the other two. \textsc{Or} and \textsc{PHP} are \textit{baguette-shaped} splits, substantially narrower than the first two, with \textsc{PHP} notably deeper than any other split.

A task becomes more difficult as it deepens and broadens. Our boule-shaped splits highlight the impact of breadth, testing how models handle an increasing diversity of propositions. Our baguette-shaped splits highlight the impact of depth, testing how models handle an increasing number of reasoning steps. See Appendix \ref{app:details} for example propositions in each split, along with additional details on their construction.

\subsection{Model}
We test Qwen2.5-Coder, Gemma 3, and Llama 3 on PITA, across many sizes. We start with pretrained models, and finetune on PITA.

Input formats for our RT and DP models are illustrated in Figure \ref{fig:pita}. For each statement in PITA, the model determines whether the statement is true or false. Each statement also has a corresponding proof establishing its veracity or falsity. RT models are trained on statements with proofs before rendering their final classification. DP models are trained only on statements, and output a classification directly without proof. In this way, we employ proofs as natural, verifiable reasoning traces, reflecting the role of proofs in math to establish the veracity of theorems. Our setting is similar to cases where models are trained directly on datasets that embed reasoning traces, where RTs are hand-constructed \cite{ye2025limo}, generated by large reasoning models \cite{guha2025openthoughts}, or constructed from a verifiable source \cite{yang2023leandojo}. This supervised setting deliberately controls trace quality, and should not be read as a complete model of RL-trained systems that learn to generate their own natural-language chains of thought. All models are finetuned on next-token prediction with cross-entropy loss, using QLoRA \cite{dettmers2023qlora}. Hyperparameters and implementation details are listed in Appendix \ref{app:details}.

We evaluate length generalization: models are trained on statements with proofs up to the median proof length of the split, and evaluated on statements with longer proofs. During evaluation, the model generates tokens autoregressively until producing a true/false classification token. By virtue of their training, DP models always classify immediately, whereas RT models generate an intermediate proof first. Since our focus is on measuring whether reasoning traces influence final-answer accuracy, rather than on certifying generated proofs, we primarily measure accuracy only from the classification token\footnote{We separately evaluate generated reasoning trace validity by replaying tactics and checking emitted states against Lean (Appendix \ref{app:trace_validity}), finding that correct final labels do not always imply valid traces. We also offer representative examples of generated traces in Appendix \ref{app:ex_out}. These outcomes are interesting in their own right, and we look forward to a more thorough evaluation of LLM-generated proofs in future work.}.

\subsection{Experimental results}

\begin{figure*}
    \vspace{-0.05in}
    \centering
    \includegraphics[width=0.9\linewidth]{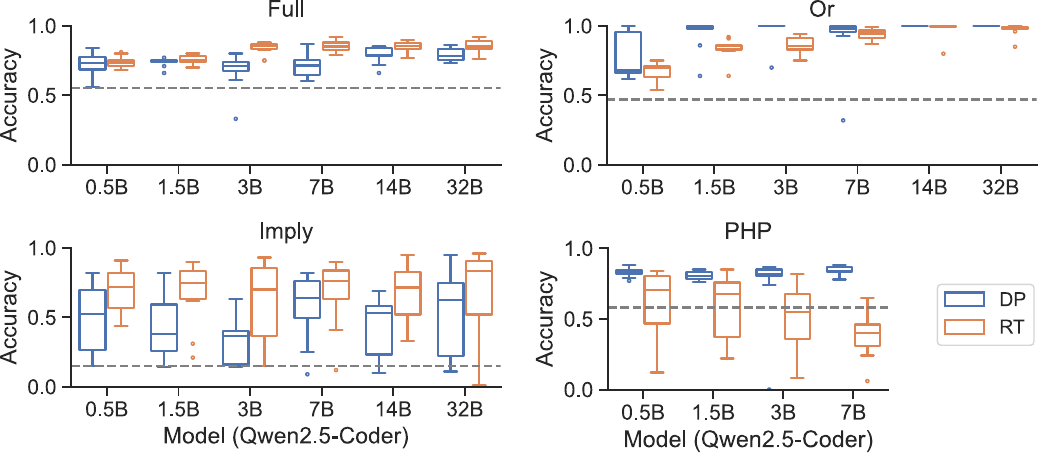}
    \caption{\textbf{Generalization accuracy on PITA splits.} Models are trained up to median proof length, then evaluated on longer examples. The dashed line marks chance level, where chance is calculated as the accuracy of random guessing with probability equal to the proportion of true/false examples in the training distribution. Because true/false proportion may vary widely between train and test distributions, chance level is sometimes substantially below 50 percent, as in \textsc{Imply}. RT models typically outperform DP models on breadth-dominated splits (\textsc{Full} and \textsc{Imply}), while the reverse is true for depth-dominated splits (\textsc{Or} and \textsc{PHP}). Due to computational constraints, the largest trainable model is 7B for \textsc{PHP} and 32B for all others. Boxplots are constructed from 10 runs, where each model is evaluated on 100 test samples. Box lines illustrate the median and quartiles. Outliers are determined from 1.5 times the inter-quartile range.}
    \label{fig:pita_res}
    \vspace{-0.05in}
\end{figure*}

Figure \ref{fig:pita_res} plots the generalization accuracy of RT and DP models across the Qwen 2.5 Coder family. Additional plots for Gemma 3 and Llama 3 are provided in Appendix \ref{app:additional_fig}, along with a collection of supporting results including balanced class accuracy, in-distribution length evaluations, and a detailed error breakdown. Reasoning trace models tend to outperform their direct prediction counterparts on the breadth-dominated \textit{boule-shaped} splits \textsc{Full} and \textsc{Imply}, particularly for larger models. The reverse is true for the depth-dominated \textit{baguette-shaped} splits, where DP models outperform RT models by as much as 50 percentage points on \textsc{PHP} for the 7B model. The same trends hold consistently across all other model families.

These results indicate a prevalent reversal in the generalization gap between RT and DP models on breadth-dominated tasks versus depth-dominated tasks. One intuitive explanation for this outcome is that long reasoning traces sometimes hurt model performance. Because deeper tasks require more steps to solve, they also require a longer reasoning trace. Recent findings suggest that long RTs incur a number of significant issues: long generation may introduce cascading errors \cite{bengio_scheduled_sampling}, autoregressive prompting is susceptible to an exposure bias \cite{schmidt_generalization}, Transformers have difficulty processing long contexts \cite{kuratov_babilong}, and empirical performance sometimes correlates \textit{negatively} with reasoning trace length \cite{wu_more_is_less}. This outcome is consistent with the emerging narrative that long RTs can hurt performance.

However, this explanation glosses over the many arguments suggesting that RTs are helpful precisely for tasks with salient step-wise structure \cite{scratchpad}. Indeed, Figure \ref{fig:app_breakdown} suggests RT models commit fewer false positive and false negative errors generally, even as non-termination becomes a serious failure mode. In related settings, direct prediction models are shown to either learn a short-cut solution that generalizes poorly \cite{liu2022transformers_shortcuts_automata}, or fail to learn a solution at all within a reasonable training budget \cite{feng_og_revealing,abbe_globality_barrier}. This literature may lead one to expect RT models to outperform DP models on tasks like PITA, particularly when the underlying problem has explicit step-wise structure. Our results instead suggest a more delicate tradeoff: traces may improve ordinary classification errors while also creating long-generation failures on depth-dominated examples.

It remains unclear from PITA alone how this tradeoff operates, and whether these results are idiosyncratic to our dataset. To provide a point of comparison, we next consider a simple task based on \textit{transitive inference} that reproduces these trends in a minimal, tractable setting. The role of this task is not to mechanistically explain every PITA failure mode, but to show that a similar topology-dependent RT/DP reversal can arise outside PITA under assumptions that can be analyzed directly.

\section{Transitive inference} \label{sec:ti}
Transitive inference (TI) has a long, rich history in developmental and comparative psychology, where it is frequently employed to describe the ability to infer novel relations from observed pairwise comparisons \cite{piaget_ti}. Given the facts ``A implies B" and ``B implies C," a subject may transitively infer that ``A implies C." TI is also witnessing increasing application in neural networks to measure reasoning \cite{lippl_ti,kay_ti,stachenfeld_ti}.

TI is simpler than propositional proof search, but it preserves a useful feature for our purposes: solving longer examples requires composing more intermediate relations. In the remainder of this paper, we present an analysis of TI in Transformers that connects task topology to generalization tradeoffs between RT and DP models. This analysis should be read as a tractable companion to PITA, rather than as a faithful surrogate for the multi-layer LLMs used in Section 2.

\subsection{Task}
Our TI task consists of a set of symbols $S = \{s_i\}$ and a set of transitive relations $R = \{s_i \rightarrow s_j\}$. Transitivity requires that if $s_i \rightarrow s_j \in R$ and $s_j \rightarrow s_k \in R$, then $s_i \rightarrow s_k \in R$. The model is then presented pairs $s_i, s_j$ with label $y = 1$ if $s_i \rightarrow s_j \in R$ and $y = -1$ otherwise.

A relation $s_i \rightarrow s_k$ is \textit{derived} if there exists one or more intermediate symbols $j_1, j_2, \ldots$ such that $s_i \rightarrow s_{j_1}; \, s_{j_1} \rightarrow s_{j_2}; \, s_{j_2} \rightarrow \ldots \rightarrow s_k$. If a relation is not derived, it is \textit{direct}. A \textit{positive relation} is any relation in $R$, including both direct and derived relations.

To construct a simple analog of task depth and breadth, we consider the following structure of transitive relations. Symbols are organized into $B$ distinct branches, each consisting of $D$ symbols. We denote the $n$th branch as a set $\mathcal{B}^n = \{ s_1^n, s_2^n, \ldots, s_D^n\}$ with direct relations $s_1^n \rightarrow s_2^n; \, s_2^n \rightarrow s_3^n \, \ldots \, s_{D-1}^n \rightarrow s_D^n$. There do not exist relations between the symbols of two different branches. 

To determine whether $s_i^m \rightarrow s_j^n$ is a positive relation, we need to enumerate the intermediate relations $s_i^m \rightarrow s_{i+1}^m \rightarrow \ldots \rightarrow s_{D}^m$. If $s_j^n \in \{s_{i+1}^m \ldots s_D^m\}$, then the relation is valid. Performing this check requires up to $D$ reasoning steps, so $D$ naturally parameterizes the task depth. For a fixed depth, the number of branches $B$ determines the number of possible symbol pairs, and therefore the size of the task space. Hence, $B$ parameterizes the task breadth.

We evaluate length generalization in the following way. The position along a branch $n$ may be defined by the number of direct relations leading from the first symbols $s_1^n$ (which conveniently aligns with our indexing). We denote the distance between two symbols as $d(s_i^n, s_j^n) = |j - i|$. Since each branch has the same number $D$ of symbols, we also define distance between symbols of different branches as $d(s_i^n, s_j^m) = |j - i|$. During training, we present symbol pairs up to a fixed distance $k$. During testing, we measure generalization accuracy on symbol pairs with distance greater than $k$.

Examples are sampled uniformly at random and presented to the model online. Exactly half of the examples have positive labels, and half have negative. Given a candidate relation $s_i^m \rightarrow s_j^n$, a negative example may be constructed either by ensuring $m \neq n$ or $i > j$. We concentrate on negative examples where branches differ, finding that this setting is already sufficient to reproduce the generalization characteristics observed in PITA, and defer a more detailed analysis of $i > j$ inverted negative examples for future work. We illustrate our task in Figure \ref{fig:ti}.

\subsection{Model} \label{sec:model}
We study single-layer, single-head, decoder-only Transformer models. Input symbols $s_1^n, s_2^n, \ldots$ are mapped to embeddings $\vr{x}_1^n, \vr{x}_2^n, \ldots \in \R^H$, sampled as $\vr{x}_i^n \sim \mathcal{N}(\vr{0}, \vr{I}_H)$ and fixed throughout training.

Embeddings pass through a self-attention layer. For an input sequence consisting of embeddings $\vr{x}_1, \vr{x}_2, \ldots, \vr{x}_L$, the $L$th activation $\vr{v}_L$ is given by $\vr{v}_L = \vr{V} \left(a_{1L} \,\vr{x}_1  + a_{2L} \, \vr{x}_2 +  \ldots + a_{LL}\, \vr{x}_L \right)$,
where attention weights are given by $a_{iL} = \text{softmax} \left(\vr{x}_i^\intercal \vr{K}^\intercal \vr{Q}\, \vr{x}_L  / H\right)$. Note, we use the $\mu$P scaling prefactor $1 / H$ rather than the original $1 / \sqrt{H}$ \cite{yang2022mup_hyperparam,bordelon2024infinite_head}. Weight matrices $\vr{K}, \vr{Q}, \vr{V} \in \R^{H \times H}$ are initialized in $\mu$P.

Activations then pass through an MLP. The $L$th logit $y_L$ is given by $y_L = \vr{w}_r^\intercal \phi (\vr{W}_1 \vr{v}_L)$. Weight matrices $\vr{W}_1 \in \R^{H\times H}, \vr{w}_r \in \R^H$ are again initialized in $\mu$P.

This minimal model is chosen for tractability, but captures our essential phenomenology. For simplicity, we omit residual connections, layer normalization, dropout, and biases.  We also omit positional encodings. Consistent with \cite{kazemnejad_nope}, we find that model performance remains the same or \textit{improves} when positional encodings are omitted.

\begin{figure*}
    \vspace{-0.05in}
    \centering
    \includegraphics[width=\linewidth]{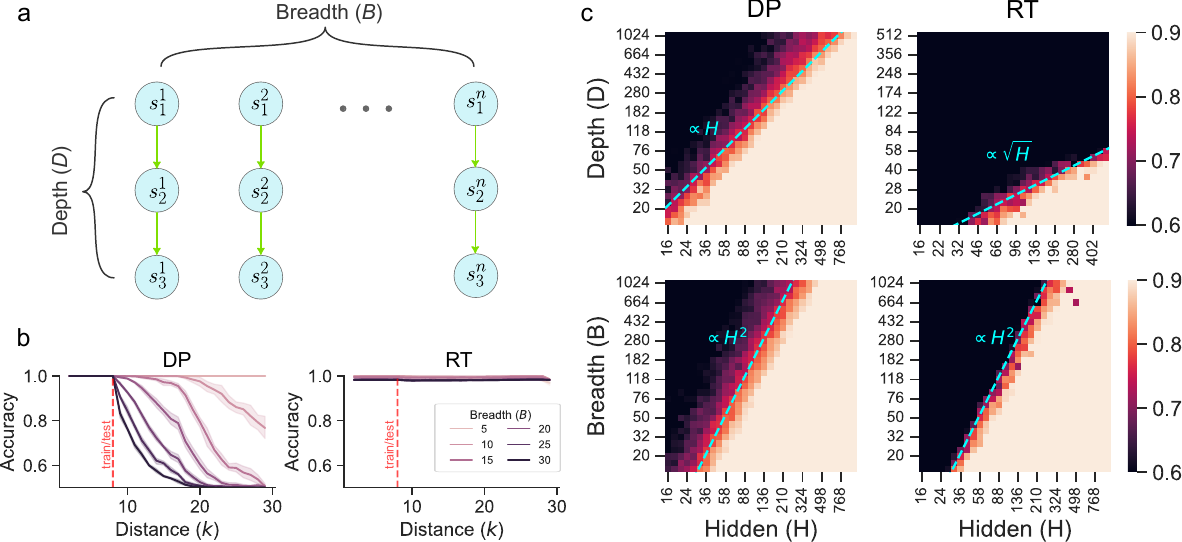}
    \caption{\textbf{Transitive inference task.} \textbf{(a)} The TI task. Symbols are arranged in parallel branches, each a line of inferences. Breadth is parameterized by the number of branches $B$, while depth is parameterized by the number of symbols in a branch $D$. \textbf{(b)} Generalization accuracy for fixed depth $D = 30$ and varying breadth. The red dashed line indicates the max training length. Generalization accuracy for the DP model decays quickly with breadth, while remaining consistently high for the RT model. Error regions correspond to 95 percent confidence intervals estimated from six seeds. \textbf{(c)} Heatmaps of training accuracy for varying depth, breadth, and model size for the full model described in Section \ref{sec:model}. Scalings are plotted in cyan, and agree closely with the high accuracy contours.}
    \label{fig:ti}
\end{figure*}

\subsubsection{Input format}
For the \textbf{direct prediction} model, a classification is immediately required after presenting the input pairs. Given embeddings $\vr{x}_i^m$ and $\vr{x}_j^n$, the input is $(\vr{x}_i^m, \vr{x}_j^n)$, producing an input in $\R^{H \times 2}$ \,. Using the output logit $y_2$, the model is trained using gradient descent on binary cross entropy loss.

For the \textbf{reasoning trace} model, inputs $\vr{x}_i^m, \, \vr{x}_j^n$ are followed by a reasoning trace that enumerates \textit{all} symbols in the $n$th branch. Once the enumeration completes, the model generates a classification token, denoted here as ``\checkmark". For example,
\[ \begin{pmatrix}
    \vr{x}_i^m & \vr{x}_j^n & \vr{x}_{1}^n & \cdots & \vr{x}_{j-1}^n & \vr{x}_{j+1}^n \cdots \vr{x}_{D}^n & \checkmark 
\end{pmatrix} \,.\]
This input produces an embedding in $\R^{H \times (D+2)}$. There are other reasonable input schemes, like terminating early if $\vr{x}_i^m$ is recovered in the reasoning trace. Our scheme holds total sequence length fixed across queried distances and simplifies our theoretical analysis. In our experiments it also generalizes substantially better than the short RT that terminates when the first queried symbol is recovered (Figure \ref{fig:app_variation}). The model is trained on autoregressive next-token prediction using gradient descent on cross entropy loss. All hyperparameters and setup details are listed in Appendix \ref{app:details}.

\subsection{TI results}
We identify empirical scalings that relate generalization accuracy with task topology. We also characterize scalings between topology and model width required to maintain high training accuracy. These scalings are summarized in Table \ref{tab:ti_scales}, and plotted in Figure \ref{fig:ti}. 

To explore where these scalings originate, we conduct a theoretical analysis based on max-margin considerations. Neural networks trained on a classification objective often converge to a set of weights that maximize the margin over their training data \cite{soudry_implicit_bias_of_gd,wei_regularization}. We characterize the max margin solution for our task, in the process identifying constraints that model weights must satisfy to solve transitive inference. These constraints limit the largest task depth $D$ and breadth $B$ that a model can learn, as a function of its hidden width $H$. We also use these constraints to infer the models' generalization properties within TI. Our analysis accounts for the empirical scalings observed in this simplified setting, and is detailed in Appendix \ref{app:ti_theory}.

To explore where these scalings originate, we conduct a theoretical analysis based on max-margin considerations. Neural networks trained on a classification objective often converge to a set of weights that maximize the margin over their training data \cite{soudry_implicit_bias_of_gd,wei_regularization}. We use a max-margin solution as a normative model of the weights selected by training, supported by the measured weight structure, rather than merely exhibiting weights that can express the task. The resulting constraints limit the largest task depth $D$ and breadth $B$ that a model can learn, as a function of its hidden width $H$. The learned structure also predicts different behavior on unseen relational distances: branch-specific correlations in DP can fail outside the training range, whereas the RT final-token classifier transfers between training and test inputs. Our analysis accounts for the empirical scalings observed in this simplified setting, and is detailed in Appendix \ref{app:ti_theory}. 

\begin{table}[t]
\vspace{-0.05in}
  \caption{\textbf{Empirical scalings for the transitive inference task}.
  We observe that generalization accuracy $\alpha_g$ decays proportional to $B^{-2}$ for DP, but remains flat for RT (see also Figure \ref{fig:app_dp_gen_acc}). The depth and breadth scalings characterize the high training accuracy contour in depth- and breadth-dominated settings. For example, the scaling $B \propto H^2$ can be interpreted as: if model width $H$ doubles, then the model maintains high accuracy on tasks up to 4x broader. Our theoretical analysis suggests the origin of these scalings, and is detailed in Appendix \ref{app:ti_theory}.}
  \label{tab:ti_scales}
  \begin{center}
    \begin{small}
      \begin{sc}
        \begin{tabular}{cccc}
          \toprule
          Prompt      & Gen. acc.  & Depth         & Breadth    \\
          \midrule
          DP    & $\alpha_g \propto B^{-2}$  & $D \propto H$        & $B \propto H^2$ \\
          RT   & $\alpha_g \propto 1$ & $D \propto \sqrt{H}$ & $B \propto H^2$ \\
          \bottomrule
        \end{tabular}
      \end{sc}
    \end{small}
  \end{center}
  \vskip -0.1in
\end{table}

These TI results qualitatively echo our findings on PITA. Reasoning trace models perform better on breadth-dominated \textit{boule-shaped} tasks: as breadth increases, generalization in the RT model remains unchanged while generalization accuracy in the DP model decays. On the other hand, direct prediction models perform better on depth-dominated \textit{baguette-shaped} tasks: the depth necessary to maintain high accuracy scales faster with hidden width in DP models compared to RT models.

Our theoretical analysis suggests that the generalization differences between DP and RT models may be understood through a difference in inductive bias. By enumerating over intermediate symbols, the RT model benefits from an inductive bias for a solution that transfers directly from train to test examples (Proposition \ref{prop:cot_test_to_train}). As a result, when training accuracy is high, generalization is perfect.

Meanwhile, generalization performance decays in the DP model as breadth increases. We find that the DP model overfits to the training distribution (Section \ref{app:generalization}), an effect exacerbated by increasing $B$. This outcome resonates with the current understanding that Transformers tend to discover ``short-cut" solutions that approximate correct outputs with brittle heuristics \cite{liu2022transformers_shortcuts_automata}. Although both DP and RT models attain a fast $B \propto H^2$ scaling in training accuracy, the superior generalization performance of reasoning trace models for large breadth differentiates their strength in a breadth-dominated setting.

When increasing task depth, the DP scaling $D \propto H$ is faster than the RT scaling $D \propto \sqrt{H}$. We ultimately trace the slower RT scaling in TI to an issue with long contexts: as depth increases, the reasoning trace also lengthens, making it more difficult to discern the correct label (Proposition \ref{prop:stat_indist_ar_length}). If breadth is small, the DP model may attain high generalization accuracy. Thus, direct prediction can outperform reasoning traces on baguette-shaped TI tasks that are deep and narrow.

\paragraph{Implications for PITA.} Our transitive inference task is substantially simpler than PITA, and different mechanisms may drive the topology/RT relationships we observe in these two settings. In particular, the PITA error breakdown points strongly toward non-termination as an important failure mode for RT models on deep splits, whereas the TI analysis isolates a signal-dilution effect in long contexts under simplifying assumptions. We therefore view TI as evidence that the PITA findings are more than an artifact of propositional proof search, but not a complete mechanistic explanation of the LLM results. Still, the comparison suggests useful shared intuitions. Reasoning traces may narrow the gap between train and test propositions, improving generalization on breadth-dominated splits. Direct prediction may be prone to spurious signals in the training set, particularly when the set of propositions broadens and becomes more complex. Deep PITA splits necessitate longer proofs, introducing long-generation and long-context issues that can handicap RT models. On the other hand, DP models benefit from narrower tasks and avoid the pitfalls of long generation, preserving their length generalization performance on baguette-shaped splits.

\section{Discussion}
We introduced PITA, a large-scale dataset of statements in propositional logic. We defined two metrics that describe task topology: (1) depth, the number of steps required to solve a task, and (2) breadth, the number of unique examples across a task. Across several model families and sizes, we identified a persistent reversal in length generalization performance between reasoning trace and direct prediction models, where RT models performed well on \textit{boule-shaped} tasks, and DP models excelled on \textit{baguette-shaped} tasks.

To explore these trends further, we proposed a transitive inference task that exhibits the same qualitative phenomena. Our theoretical analysis suggests that DP models discover suboptimal minima that generalize poorly for increasing task breadth in this simplified setting. Reasoning traces narrow the generalization gap, supporting high accuracy for shallow tasks as breadth increases. However, long contexts hinder RT performance on deep TI tasks, pulling it below DP models when the task remains narrow.

The surprising performance of DP models indicates a need for more nuance when using RTs. Past work have observed that reasoning traces are minimally beneficial (and occasionally harmful) in domains like text classification and commonsense reasoning \cite{sprague2024_2cot_or_not_2cot,liu2024mind_your_step}. Even in propositional logic, a setting that should complement the advantages of reasoning traces, we identify a consistent generalization difference between RT and DP models related to task topology. Our notions of task depth and breadth are concrete for PITA and suggestive more broadly; we anticipate future work characterizing topology more precisely and across a wider range of settings to suggest where direct prediction may excel.

In many domains, the task is both very \textit{broad} and very \textit{deep}. For these especially challenging reasoning tasks, improving the long-context and long-generation behavior of RT models may be more feasible than relying on DP models to extrapolate without intermediate computation. Our results are consistent with growing evidence that current Transformers can struggle with long contexts \cite{kuratov_babilong,hsieh_ruler,roberts_needle_threading,tay_long_range_arena}. Besides issues with autoregressive generation, our theory on TI additionally highlights the statistical difficulty of processing a long context where the signal is weakly dispersed across many relevant tokens. As reasoning trace techniques grow increasingly popular, we highlight long-context reasoning as an important obstacle for current Transformer-based paradigms.

With the general increasing interest in formal reasoning, we add PITA to the growing collection of testbeds for promoting and evaluating logical reasoning in neural networks. As a large-scale, extensible framework for generating Lean-formalized statements in propositional logic, PITA presents many additional opportunities for rigorously testing and developing new models. We hope that the topology perspective introduced here helps guide future frumentaceous work on when reasoning traces should be trusted, shortened, verified, or avoided.

\paragraph{Scope and limitations.} Our study deliberately isolates length generalization for proof-like traces under supervised finetuning. This choice enables clean comparisons between RT and DP models, but it leaves several important questions open for future work. 1) PITA uses formal proofs rather than natural-language chains of thought, so our results do not directly answer how free-form natural-language rationales behave. 2) Our models are finetuned with fixed gold traces, whereas many modern reasoning systems use reinforcement learning to shape their own traces. This restriction is intentional: learned traces introduce additional variation in trace quality and training dynamics, whereas our setting isolates the effect of providing the same gold trace to every RT model. 3) PITA has binary final answers; broader answer spaces may change the relative difficulty of DP and RT formats, and may themselves contribute to task breadth. 4) Length is only one kind of distribution shift. Shifts in symbolic representation, logical structure, or trace quality would probe different robustness questions. 5) Finally, the TI theory should be understood as a tractable companion analysis, not a full mechanistic account of Qwen, Gemma, or Llama models trained on PITA.

\section*{Acknowledgements}

We thank Mike Douglas, Mike Freedman, and members of the Pehlevan Group for many helpful discussions during this project. WLT is supported by a Kempner Graduate Fellowship. CP is supported by an NSF CAREER Award (IIS-2239780), DARPA grants DIAL-FP-038 and AIQ-HR00112520041, the Simons Collaboration on the Physics of Learning and Neural Computation, and the William F. Milton Fund from Harvard University. This work has been made possible in part by a gift from the Chan Zuckerberg Initiative Foundation to establish the Kempner Institute for the Study of Natural and Artificial Intelligence. The computations in this paper were run on the FASRC cluster supported by the FAS Division of Science Research Computing Group at Harvard  University. LLM assistance was partly used in implementing experiments, checking mathematical results, and polishing the writing.

\newpage
\appendix
\onecolumn

\section{Primer on propositional logic and proof systems} \label{app:prop_logic}

Propositional logic is a calculus of truth. Given propositions with intrinsic truth value and connectives that join them, propositional logic studies the ultimate truth of their combination. This appendix offers a brief primer to provide some background for our PITA dataset. This topic is deep and nuanced; for a more in-depth treatment of this topic, we recommend \cite{enderton2001logic}.

A \textit{statement} in propositional logic (or simply, a \textit{proposition}) consists of \textit{atoms} and \textit{connectives}. An \textit{atom} is a symbol that denotes truth ($\top$), falsity ($\bot$), or a variable (e.g. $p, q, r, \ldots$). A connective is a function that takes two propositions and returns a corresponding truth assignment. For instance, the connective \textit{and} ($\wedge$) evaluates to true only if both inputs are true. Hence, $\top \wedge \top$ is true but $\top \wedge \bot$ is false. Other common connectives are \textit{or} ($\vee$) and \textit{imply} ($\rightarrow$). The truth assignment of each connective we use in PITA is summarized below in Table \ref{tab:truth_assignments}

\begin{table}[h]
  \caption{\textbf{Truth assignments for logical connectives in PITA}. For propositional variables $p$ and $q$ with the indicated truth assignment, we record the output of each connective.}
  \label{tab:truth_assignments}
  \begin{center}
    \begin{small}
      \begin{sc}
        \begin{tabular}{ccccc}
          \toprule
          $p$  & $q$ & $p \wedge q$  & $p \vee q$ & $p \rightarrow q$  \\
          \midrule
          T    & T   &  T            & T          & T  \\
          T    & F   &  F            & T          & F  \\
          F    & T   &  F            & T          & T  \\
          F    & F   &  F            & F          & T  \\
          \bottomrule
        \end{tabular}
      \end{sc}
    \end{small}
  \end{center}
  \vskip -0.1in
\end{table}

We say that a proposition is true if it is a \textit{tautology}: the proposition remains true for every possible truth/falsity assignment of its variables. Otherwise if the proposition is not a tautology, then we say the proposition is false. To demonstrate whether a propositional statement is true, one may simply enumerate over all possible truth assignments and verify that the statement remains true. However, this approach requires $2^n$ evaluations for a statement with $n$ variables, which becomes impractical for large propositions. To demonstrate truth more efficiently, we may instead construct a \textit{proof}.

\subsection{Proofs and proof search}
A \textit{proof} is a sequence of transformations that converts a proposition into other (usually simpler) propositions while preserving tautology. These transformations are called \textit{inference rules}. For example, a common inference rule is \textit{modus ponens} (MP), which requires two propositions as input: a proposition $p$, and a proposition $p \rightarrow q$. MP then allows us to transform these two propositions into the single proposition $q$.

There are a variety of valid inference rules beyond MP, and different proof systems permit different sets of inference rules. In PITA, we use rules of \textit{natural deduction}, which are a common choice with desirable properties for both reasoning about proofs and engineering automated proof search procedures. For a more detailed discussion of natural deduction and examples of its inference rules, see \cite{pelletier2021natural_deduction}.

Armed with inference rules, the process of constructing a proof can be thought of as the process of identifying the correct sequence of rules that translates the starting proposition into the proposition $\top$. Under this interpretation, inference rules induce a graph-like topology: sets of propositions are nodes, and inference rules are directed edges that connect them. If a path exists from a proposition $p$ to the proposition $\top$, then $p$ must be true. Otherwise if no such path exists, $p$ must be false\footnote{Since inference rules may introduce additional propositions that need to be concurrently proved, proofs technically constitute \textit{trees} through this graph rather than paths. A proof is true if all leaves of this tree terminate in $\top$. Otherwise, the proof is false. In this discussion, we have also ignored nuances related to syntactic truth versus semantic truth, and falsity versus unprovability. These issues are important but not immediately relevant in our setting. For additional discussion on these topics and more, we recommend \cite{enderton2001logic}}.

In this way, proving a proposition reduces to implementing a search procedure over the graph of propositional statements and inference rules. Common approaches to performing this search procedure are sensitive to symmetries inherent to the different logical connectives, and apply inference rules in deliberate order to control the size of the search space. To construct the proofs used in PITA, we employ an algorithm from the family of methods called \textit{focused proof search} \cite{liang2009focused_proof_search}, a popular choice for constructing efficient and intuitive proofs to arbitrary propositional statements.

\section{Theoretical analysis of transitive inference} \label{app:ti_theory}

In this appendix, we detail our analysis of a Transformer learning the transitive inference task, introduced in Section \ref{sec:ti}. Our analysis is organized in three stages.
\begin{enumerate}
    \item We begin by simplifying our Transformer model further. We argue that attention becomes uniform over inputs for our task, reducing our model to a two-layer MLP. Building on \cite{tong_learning_richness}, we introduce a \textit{sample-coordinate representation} of our model weights, in which we characterize our weights in the span of training samples. 
    \item We identify the set of weights that maximize the classification margin in our transitive inference task. Because neural networks frequently discover a max-margin solution \cite{soudry_implicit_bias_of_gd,wei_regularization,chizat_implicit_bias_of_gd,morwan_feature_learning}, we use the max-margin weights to describe a normative solution that the model implements to solve this task.
    \item Under the max-margin solution, we establish the scaling relationships and generalization patterns observed in Figure \ref{fig:ti}.
\end{enumerate}

\subsection{Simplified model}
We use a single-layer single-head Transformer model, summarized in Section \ref{sec:ti}. To facilitate our analysis, we simplify this model in the following ways. 

We merge the key-query matrices into a single $H \times H$ matrix $\mathbf{M} \equiv K^\intercal Q$, and merge the value matrix $\mathbf{V}$ with the hidden weights $\vr{W}_1$ into a single $H \times H$ matrix $\vr{W} \equiv \vr{W_1} \vr{V}$. For inputs $\vr{x}_1, \vr{x}_2, \ldots, \vr{x}_L$, our reparameterized Transformer computes the $L$th logit $y_L$ as
\begin{equation} \label{eq:transformer_mlp}
y_L \equiv f(\vr{x}_1, \ldots, \vr{x}_L) = \sum_{i = 1}^H r_i \phi (\vr{w}_i \cdot \vr{v}_L)
\end{equation}
where
\[\vr{w}_r =\begin{pmatrix}
    r_1 \\ r_2\\ \vdots \\r_H
\end{pmatrix}
\quad 
\vr{W} =\begin{pmatrix}
    \Big|    & \Big|    &   & \Big| \\
    \vr{w}_1 & \vr{w}_2 & \cdots & \vr{w}_H \\
    \Big|    & \Big|    &   & \Big| \\
\end{pmatrix}^\intercal
 \,.\]
and
\begin{align*}
    \vr{v}_L &= a_{1L} \, \vr{x}_1 + \ldots + a_{LL}\, \vr{x}_L \\
    a_{ij} &= \text{softmax}\, \Px{\frac{1}{H} \vr{x}_i^\intercal \vr{M} \vr{x}_j}
\end{align*}
To simplify the analysis and unify the treatment of the direct prediction and reasoning trace cases, we ignore autoregressive generation in the RT model and consider classification accuracy only on the final logit, after a correct chain has been presented. Doing so constitutes an upper bound on the RT model's performance. We nonetheless identify generalization weaknesses in the RT model independent of the autoregressive failures commonly posited as the source of failure for chain-of-thought prompting \cite{bengio_scheduled_sampling,schmidt_generalization,wu_more_is_less,feng_eff_reasoning_len}.

\paragraph{Sample-coordinate representation.} We next describe a \textit{sample-coordinate representation} of the hidden weights $\vr{w}_i$. From Eq \eqref{eq:transformer_mlp}, it is evident that the gradient of $y_L$ with respect to $\vr{w}_i$ is
\begin{align*}
    \frac{\d y_L}{\d \vr{w}_i} &\propto \vr{v}_L \\
    &= a_{1L} \, \vr{x}_1 + \ldots + a_{LL}\, \vr{x}_L
\end{align*}
which is a sum over input embeddings $\vr{x}_\ell$. Hence, if we train our model using gradient descent, all updates to $\vr{w}_i$ can be expressed as a linear combination of input embeddings. At the end of training, we may decompose $\vr{w}_i$ as
\[\vr{w}_i = \vr{w}_i^0 + c_1^i \vr{x}_1 + c_2^i \vr{x}_2 + \ldots + c_V^i \vr{x}_V \,,\]
where $\vr{w}_i^0$ is its value at initialization and $\vr{x}_1, \ldots, \vr{x}_V$ correspond to the $V$ distinct embeddings in the task vocabulary. If we assume the model begins with a zero or very small initialization, then $\vr{w}_i$ can be wholly characterized by the \textit{sample coordinates} $c_1, c_2, \ldots c_V$, so called because they correspond to the coordinates of $\vr{w}_i$ in the span of the training samples. We study solutions to the TI task in this \textit{sample-coordinate space}, ignoring the effect of the initialization. Indeed, we train our Transformers in a \textit{rich} learning regime, where the influence of initialization is diminished. Sample coordinates are superficially similar to the \textit{dual coefficients} that weight the contribution of training examples in kernel methods.

Note that
\begin{equation} \label{eq:dual_dot_x}
\vr{w}_i \cdot \vr{x}_1 = \vr{w}_i^0 \cdot \vr{x}_1 + c_1^i ||\vr{x}_1||^2 + \sum_{\ell=2}^V c_\ell^i (\vr{x}_1 \cdot \vr{x}_\ell) \,.
\end{equation}
Since input embeddings are sampled from an isotropic standard Gaussian, we have that
\begin{align*}
    \E [||\vr{x}_1||^2] &= \Theta \Px{H} \\
    \E [\vr{w}_i^0 \cdot \vr{x}_1] &= \Theta (\sqrt{H}) \\
    \E [\vr{x}_1 \cdot \vr{x}_\ell] &= \Theta (\sqrt{H})
\end{align*}
Hence, for very wide models $H \gg V$, we conclude that $\frac{1}{H} \vr{w}_i \cdot \vr{x}_1 \approx c_1^i$ (we consider this condition more carefully below starting with Proposition \ref{prop:dual_signal}). We may therefore interpret our model as
\begin{equation*}
y_L \approx H \sum_{i=1}^H r_i \phi(a_{1L} \,c_1^{i} + \ldots + a_{LL}\, c_L^i) \,.
\end{equation*}
When we analyze the margin of our model, we will frequently drop the extra factor of $H$, which serves only to change the scale of the margin without impacting the classification. Later, when we report results based on estimated sample coordinates, we estimate the sample coordinate for $\vr{x}$ from a weight vector $\vr{w}_i$ precisely as $\hat{c}^i = \frac{1}{H} \vr{w}_i \cdot \vr{x}$.

\paragraph{Uniform attention weights.} Our final simplification is to assume that the attention weights are uniform: $a_{1L} = a_{2L} = \ldots = a_{LL} = \frac{1}{L}$. This rather drastic simplification is nonetheless supported by empirical measurements (Figure \ref{fig:app_att}). We also give the following intuition for why the model settles on uniform attention on our TI task, based on max-margin considerations.

Let $\mathcal{X} = \{\vr{x}_{\ell_1}, \vr{x}_{\ell_2}, \ldots, \vr{x}_{\ell_L}\}$ be the set of training examples with length $L$, labeled by the function $y(\vr{x}_{\ell_1}, \ldots, \vr{x}_{\ell_L}) = \pm 1$. Then the margin is given by
\[\delta = \min_{\ell_1, \ldots, \ell_L} \, y(\vr{x}_{\ell_1}, \ldots, \vr{x}_{\ell_L}) \sum_i r_i \phi(a_{{\ell_1} \ell_L} \, c_{\ell_1}^i + \ldots + a_{{\ell_L} \ell_L} \, c_{\ell_L}^i) \,,\]
subject to $\sum_i r_i^2 = 1$ and $\sum_{ij} \Px{c_j^i}^2 = 1$ (or equivalent norm constraints). Suppose there exists a unique training example that lies on the margin, given by indices $m_1, \ldots, m_L$. Let $\mathcal{I}$ be the set of hidden activation indices $i$ such that $a_{{m_1} {m_L}} c_{m_1}^i + \ldots + a_{{m_L} {m_L}} c_{m_L}^i > 0$. Then the magnitude of the margin is given by
\[|\delta| = a_{{m_1} {m_L}} \sum_{i \in \mathcal{I}} r_i c_{m_1}^i + \ldots + a_{{m_L} {m_L}} \sum_{i \in \mathcal{I}} r_i c_{m_L}^i \,. \]
If the model attains perfect training accuracy (as it does empirically), then $\delta = |\delta|$. Attention weights must sum to 1, so they form a convex sum over elements $\sum_{i \in \mathcal{I}} r_i c_{m_t}^i \equiv e_t$. Given free choice over the attention weights, to maximize the margin, we could allocate all weight to the index $\argmax_t e_t$. Similarly, given free choice over the magnitudes of elements $e_t$, we could allocate all available weight (allowed under our norm constraints) to the index $\argmax_t a_{m_t, m_L}$.

Suppose we set $a_{m_1, m_L} = 1$ and increased the weight of $e_1 = \sum_{i \in \mathcal{I}} r_i c_{m_1}^i$. Then the margin on the training example indexed by $m_1, \ldots, m_L$ is now large, but it may no longer lie on the margin. The \textit{new} example that minimizes $\delta$ may now be very small. Indeed, due to the symmetry of the transitive inference task, any symbol may appear in any position and any class with even probability across all training examples. Hence, if $a_{m_1, m_L}$ and $e_1$ are large, then there exists a training example $m_1, m_2', \ldots, m_L'$ in the opposite class as $m_1, \ldots, m_L$, with a very small or very negative margin because of our weight allocation.

Hence, because of this symmetry, we conjecture that the magnitude of each $e_t$ tends to remain about the same in the max margin set, and no position or symbol is favored above any other. The margin cannot be increased by concentrating attention on one position or symbol, only shrunk. Thus, uniform attention weights are favorable, and we confirm empirically in Figure \ref{fig:app_att} that a trained transformer indeed exhibits uniform attention weights on our task.

With uniform attention weights $a_{1L} = \ldots = a_{LL} = \frac{1}{L}$, our model simplifies down to
\begin{equation*}
y_L = \frac{1}{L} \sum_{i=1}^H r_i \phi(\vr{w}_i \cdot (\vr{x}_1 + \cdots + \vr{x}_L)) \,.
\end{equation*}

\begin{figure}
    \centering
    \includegraphics[width=0.7\linewidth]{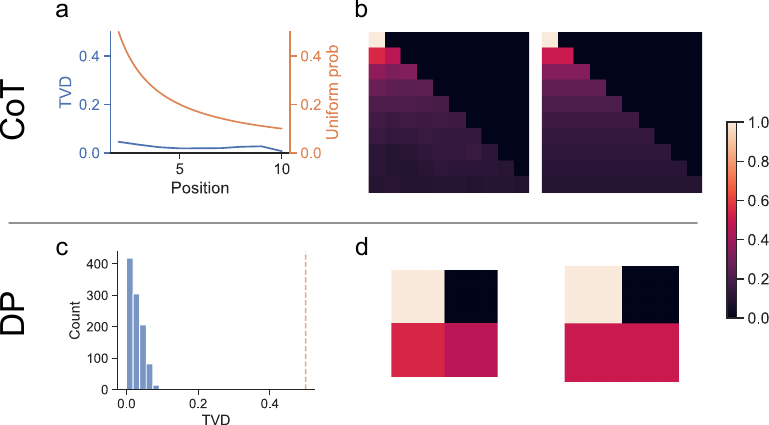}
    \caption{\textbf{Attention weights are uniform.} We show empirically that the attention weights become uniform in a Transformer trained on our transitive inference task. \textbf{(a)} \textit{In blue}, total variation distance (TVD) between a uniform distribution and the attention weights in an RT model, measured across 1000 examples. The x-axis indicates the position of the query token. \textit{In orange}, the probability assigned to each position by a uniform distribution, plotted for comparison. The TVD is very close to zero across all positions, indicating that the attention weights are close to uniform. \textbf{(b)} \textit{Left}, an example attention matrix from the RT model. \textit{Right}, attention matrix with uniform entries, plotted for comparison. \textbf{(c)} Histogram of TVD between a uniform distribution and attention weights in a DP model at position 2 (the output position for a DP model), measured across 1000 examples. \textit{In orange}, the probability assigned to position 2 by a uniform distribution (which is 0.5). As in (a), the TVD is fairly close to 0. \textbf{(d)} The same as (b), plotted with the DP model.}
    \label{fig:app_att}
\end{figure}

\subsection{Max-margin solution} \label{sec:max_margin_sol}
Neural networks trained on a classification objective routinely discover the max-margin solution \cite{soudry_implicit_bias_of_gd,wei_regularization,chizat_implicit_bias_of_gd}. This observation has been productively applied to explaining the features in trained models \cite{morwan_feature_learning,tong_learning_richness}. We study the structure of the max margin weights for our transitive inference task, and use it to form the basis of a normative solution our model should learn. We then use this solution to reveal how its properties constrain the scaling relationships between model width $H$ and task parameters $B,D$.

Following \cite{morwan_feature_learning}, we consider the max \textit{average} margin as a reasonable proxy to the strict max margin. If training examples are sampled from the distribution $(\vr{x}_1, \ldots, \vr{x}_L), y \sim \mathcal{D}$, then the average margin is given by
\[\overline{\delta} = \E_\mathcal{D} \left[ \frac{y}{L} \sum_{i=1}^H r_i\,\phi(\vr{w}_i \cdot (\vr{x}_1 + \ldots + \vr{x}_L)) \right] \,.\]
Denote the max margin solution to be
\begin{align*}
\{r_i^*, \vr{w}_i^*\} = &\argmax_{r_i, \vr{w}_i}\, \overline{\delta} \\ &\text{subject}\,\text{to}\; |r_i| = ||\vr{w}_i|| = 1
\end{align*}

Note, we apply our norm constraint to individual activations $|r_i|$ and $||\vr{w}_i||$. This scheme is more restrictive than a global constraint where for instance $\sum_i r_i^2 = 1$ and $\sum_i ||\vr{w}_i||^2 = 1$. Our choice of constraints reflects the assumption that each activation contributes equally to the final classification, and that the margin cannot be improved by weighting some activations heavier than others. We also find empirically that the sign of each $r_i$ tends to remain the same throughout training. We later show in Figure \ref{fig:app_max_margin} that the structure of the max margin solution under these assumptions is descriptive of the actual weights learned by a model after training.

Suppose that training examples with positive label $y = 1$ have distribution $\mathcal{D}^+$ and examples with negative label $y = -1$ have distribution $\mathcal{D}^-$. Denote sets of indices $\mathcal{I}^+ = \{i : r_i^* > 0\}$ and $\mathcal{I}^- = \{i : r_i^* < 0\}$. Denote the input embeddings by $\overline{\vr{x}} = \frac{1}{L}(\vr{x}_1 + \ldots + \vr{x}_L)$. Then by linearity of expectation,
\begin{align}
\overline{\delta} &= \frac{1}{2}\sum_{i \in \mathcal{I}^+} \left( \E_{\mathcal{D}^+}\left[ \phi(\vr{w}_i \cdot \overline{\vr{x}}) \right] - \E_{\mathcal{D}^-} \left[ \phi(\vr{w}_i \cdot \overline{\vr{x}})\right] \right) \nonumber \\
&- \frac{1}{2}\sum_{j \in \mathcal{I}^-} \left( \E_{\mathcal{D}^+}\left[ \phi(\vr{w}_j \cdot \overline{\vr{x}}) \right] - \E_{\mathcal{D}^-} \left[ \phi(\vr{w}_j \cdot \overline{\vr{x}})\right] \right) \,. \label{eq:margin}
\end{align}

Under the sample-coordinate representation, we have that
\[\frac{L}{H} \vr{w}_i \cdot \overline{\vr{x}} \approx c_1^i + \ldots + c_L^i \,\]
so
\[\frac{L}{H}\phi(\vr{w}_i \cdot \overline{\vr{x}}) \approx (c_1^i + \ldots + c_L^i) \mathbbm{1}_{c_1^i + \ldots + c_L^i > 0} \,.\]
We characterize the structure of the max margin solution in our sample-coordinate space defined by $c_t^i$.

We begin with an intuitive discussion before presenting a formal result. We use a superscript on $\vr{x}_t^n$ to denote that symbol embedding $\vr{x}_t$ lies on branch $n$ (while the index $t$ would now correspond to a particular depth within branch $n$). Given input embeddings $\vr{x}_1^m, \vr{x}_2^n, \ldots, \vr{x}_L^n$, our scheme for constructing the reasoning trace requires that all symbols $\vr{x}_2^n, \vr{x}_3^n, \ldots, \vr{x}_L^n$ lie on the same branch. The symbol $\vr{x}_1^m$ shares the same branch $m = n$ only if the input has a positive label $y = 1$. Otherwise when $y = -1$, the symbol $\vr{x}_1^m$ lies on a different branch $m \neq n$. In cases where $y = 1$ (and all symbols lie on the same branch), our margin over these examples is
\[\frac{1}{2} \sum_{i \in \mathcal{I}^+} \E_{\mathcal{D}^+} [\phi(\vr{w}_i \cdot \overline{\vr{x}})] - \frac{1}{2} \sum_{j \in \mathcal{I}^-} \E_{\mathcal{D}^+} [\phi(\vr{w}_j \cdot \overline{\vr{x}})] \,.\]
Hence for $i \in \mathcal{I}^+$, the margin is maximized when the sum $c_1^{i,n} + \ldots + c_L^{i,n}$ is large. However, for indices $j \in \mathcal{I}^-$, the margin is maximized when the sum $c_1^{j,m} + \ldots + c_L^{j,n}$ is very small or negative. In this way, we want the quantity $c_1^{i,n} + \ldots + c_L^{i,n}$ to be large when $\vr{x}_1^n$ is on the \textit{same} branch $n$ as the remaining symbols $\vr{x}_2^n, \ldots, \vr{x}_L^n$, suggesting that $c_1^{i,n} > 0$ when $c_2^{i,n}, \ldots, c_L^{i,n} > 0$. However, when the first symbol belongs to a different branch $m \neq n$, we desire for $c_1^{i,m} + \ldots + c_L^{i,n}$ to be small or negative, suggesting that the sign of symbols $c_t^{i,m}$ should be the opposite of the sign for symbols $c_t^{i,n}$ lying on a different branch $n$. We might therefore expect that activations for $i \in \mathcal{I}^+$ may have a structure such that all sample coordinates sharing a branch should also share the same sign, some positive and some negative.

The opposite is true when considering the case $y = -1$:
\[-\frac{1}{2} \sum_{i \in \mathcal{I}^+} \E_{\mathcal{D}^-} [\phi(\vr{w}_i) \cdot \overline{\vr{x}}] + \frac{1}{2} \sum_{j \in \mathcal{I}^-} \E_{\mathcal{D}^-} [\phi(\vr{w}_j) \cdot \overline{\vr{x}}] \,.\]
For activations corresponding to an index $i \in \mathcal{I}^+$, the margin is now maximized when the sum $c_1^{i,n} + \ldots + c_L^{i,n}$ is small or negative; for index $j \in \mathcal{I}^-$, the margin is maximized when the sum $c_1^{j,m} + \ldots + c_L^{j,n}$ is large. The first condition on $c_1^{j,n} + \ldots + c_L^{j,n}$ suggests that all coordinates sharing the same branch $n$ should be negative or balanced in sign, whereas the second condition on $c_1^{j,m} + \ldots + c_L^{j,n}$ suggests that at least some coordinates should be large and positive. Hence, we might expect that activations for $j \in \mathcal{I}^-$ should exhibit a mixed proportion of positive and negative sample coordinates across branches.

\begin{figure}
    \centering
    \includegraphics[width=0.7\linewidth]{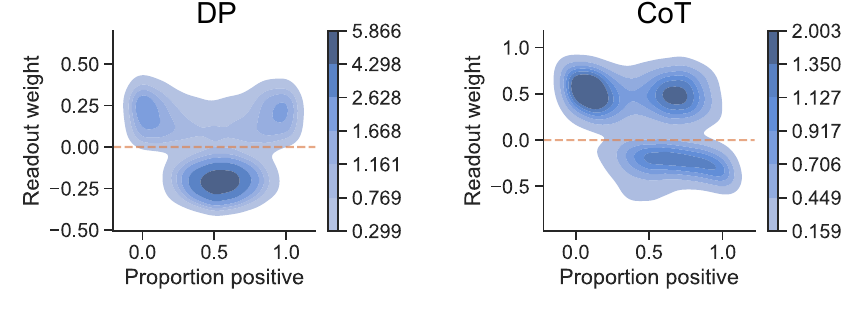}
    \caption{\textbf{Trained models learn the max margin solution.} We plot the kernel density estimate (KDE) over the proportion of positive, estimated sample coordinates per task branch in (\textit{left}) a direct-prediction model and (\textit{right}) reasoning trace model. The proportion is indexed by the x-axis. The y-axis indexes the readout weight of the corresponding weight vector, from which the sample coordinates are estimated. For positive readout weights, we see that proportions are bimodal around 0 and 1. As the readout weight becomes negative, proportions tend to coalesce in a single range. These observations are consistent with implementing a max margin solution. Given a weight vector $\vr{w}_i$ and symbol embedding $\vr{x}_j$, the corresponding sample coordinate is estimated as $c_j^i = \vr{w}_i \cdot \vr{x}_j$.}
    \label{fig:app_max_margin}
\end{figure}

We can formalize these intuitions in the case where $L = 2$ (corresponding to the direct prediction model) and the magnitude of $|c_t^n|$ is the same across all symbols, in Proposition \ref{prop:max_margin} below. We find empirically that the structure of weights in trained models are consistent with this analysis (Figure \ref{fig:app_max_margin}).

\begin{proposition} \label{prop:max_margin}
Under the sample-coordinate representation, suppose that the magnitude of sample coordinates $|c_k^{i,n}|$ is the same across all depth indices $k$, activation indices $i$, and branch indices $n$. If $L = 2$ and the task parameters $B, D \gg 1$, then the following properties hold for the set of sample coordinates $c_k^{i,n}$ that maximize the margin in Eq \ref{eq:margin}:
\begin{itemize}
    \item For activation indices $i \in \mathcal{I}^+$, $\text{sign} (c_k^{i,n}) = \text{sign} (c_{k'}^{i,n})$ for all $k, k' \in [D]$. In words, all coordinates that share the same branch also have the same sign.
    Furthermore, exactly half of the branches have coordinates with positive sign, and half of branches have coordinates with negative sign.
    \item For activation indices $j \in \mathcal{I}^-$, $\sum_k \mathbbm{1} \{ {c_k^{j,n} > 0} \} = \sum_k \mathbbm{1} \{c_k^{j, n'} > 0\}$ for all $n, n' \in [B]$. In words, the proportion of positive coordinates is the same across all branches.
\end{itemize}
\end{proposition}
\begin{proof}
Given a training example that consists of the two input embeddings $\vr{x}_k^m, \vr{x}_\ell^n$, in sample coordinates we have
\[\frac{1}{H} \phi(\vr{w}_i \cdot (\vr{x}_k^m + \vr{x}_\ell^n)) = (c_k^{i,m} + c_\ell^{i,n}) \mathbbm{1}\{c_k^{i,m} + c_\ell^{i,n} > 0\} \,.\]
Since $|c_k^{i,m}| = |c_\ell^{i,n}|$, we can rewrite our indicator as $\mathbbm{1}\{(c_k^{i,m} > 0) \,\cap\, (c_\ell^{i,n} > 0)\}$, and
\[\E_{\mathcal{D}}[\phi(\vr{w}_i \cdot  (\vr{x}_k^m + \vr{x}_\ell^n)) \propto \Pr (c_k^{i,m} > 0 \,\cap\, c_\ell^{i,n} > 0) \,.\]

Define $\mu_{i, m} = \Pr(c_k^{i,m} > 0)$ for an index $k$ sampled uniformly at random from branch $m$. Recall that examples with positive label $y = 1$ are sampled from the same branch. Hence for large task depth $D$, we have
\[\Pr\nolimits_\mathcal{D^+} (c_k^{i,m} > 0 \,\cap\, c_\ell^{i,n} > 0) = \frac{1}{B} \sum_{m=1}^B \mu_{i,m}^2 \,.\]
Recall that examples with negative label $y = -1$ are sampled from different branches. Hence for large task breadth $B$, we have
\[\Pr\nolimits_\mathcal{D^-} (c_k^{i,m} > 0 \,\cap\, c_\ell^{i,n} > 0) = \overline{\mu}_i^2 \equiv \Px{\frac{1}{B} \sum_{m=1}^B \mu_{i,m}}^2 \,.\]
Altogether then, we have
\begin{equation} \label{eq:margin_half}
E_{\mathcal{D}^+}\left[ \phi(\vr{w}_i \cdot \overline{\vr{x}}) \right] - \E_{\mathcal{D}^-} \left[ \phi(\vr{w}_i \cdot \overline{\vr{x}})\right] \propto \Px{\frac{1}{B} \sum_{b=1}^B \mu_{i,m}^2} - \overline{\mu}_i^2 \,.
\end{equation}
If $i \in \mathcal{I}^+$, maximizing the margin in Eq \eqref{eq:margin} requires maximizing Eq \eqref{eq:margin_half}. Otherwise if $i \in \mathcal{I}^-$, then maximizing the margin requires minimizing Eq \eqref{eq:margin_half}. We would therefore like to characterize the extrema of this expression.

We first characterize the maximum, revealing the structure of activations $i \in \mathcal{I}^+$. Note that $\mu_{i,m} \in [0, 1]$, so we must have that $\mu_b^2 \leq \mu_b$, with equality when $\mu_{i,m} = 1$ or $\mu_{i,m} = 0$. This suggests that
\[\Px{\frac{1}{B} \sum_{m=1}^B \mu_{i,b}^2} - \overline{\mu}_i^2 \leq \overline{\mu}_i - \overline{\mu}_i^2 \,, \]
again with equality if all $\mu_{i,m}$ equal either $0$ or $1$. The quantity $\overline{\mu}_i - \overline{\mu}_i^2$ is maximized for $\overline{\mu}_i = \frac{1}{2}$. In this way, the outcome that $\mu_{i,m}$ be either $0$ or $1$ corresponds to the condition that all coordinates within a task branch share the same sign. The outcome that $\overline{\mu}_i = \frac{1}{2}$ corresponds to the condition that half the task branches consist of positive coordinates, and the other half consist of negative coordinates.

Now for the minimum (relating to activations $i \in \mathcal{I}^-$), we observe that Eq \eqref{eq:margin_half} is always non-negative. Indeed, if we interpret $\mu_{i,m}$ as draws from a random variable, then Eq \eqref{eq:margin_half} is simply its variance. This expression is minimized when $\Px{\frac{1}{B} \sum_{m=1}^B \mu_{i,b}^2} = \overline{\mu}_i^2$, which occurs when $\mu_{i,b} = \overline{\mu}_i$. This result corresponds to the condition that the proportion of positive coordinates is the same across all task branches.
\end{proof}

Proposition \ref{prop:max_margin} holds only under specific conditions, but as our discussion above illustrates, we expect the general structure of the max margin solution to continue holding for large $L$ and unconstrained coordinate magnitudes. Indeed, we find empirically that trained models tends to obey this max margin structure (Figure \ref{fig:app_max_margin}).

\subsection{Scaling relationships}
Armed with a general understanding of the max margin solution, we turn our attention to understanding the scaling relationships uncovered in Figure \ref{fig:ti} and summarized in Table \ref{tab:ti_scales}, which illustrates the scalings between task depth, breadth, and model width required to saturate the training accuracy on the transitive inference task.

Our approach for studying these scalings proceeds as follows. Given a hidden weight $\vr{w}_i$ and symbol embedding $\vr{x}_1$, Eq \eqref{eq:dual_dot_x} suggests that $\frac{1}{H} \vr{w}_i \cdot \vr{x}_1 \approx c_1^i$ (or for any other index $j \neq 1$, without loss of generality), provided that $H \gg V$. We formalize this intuition below in Proposition \ref{prop:dual_signal}.

\begin{proposition} \label{prop:dual_signal}
    Let $\vr{x}_1, \vr{x}_2, \ldots, \vr{x}_V$ be iid draws from a standard Gaussian with dimension $H$. Suppose $c_1, c_2, \ldots, c_V$ are sample coordinates such that $|c_1| = |c_2| = \ldots = |c_V| = c$, for some fixed finite $c$, and $\vr{w} = c_1 \vr{x}_1 + c_2 \vr{x}_2 + \ldots + c_V \vr{x}_V$. If $V = o(H)$, then as $H \rightarrow \infty$, $\frac{1}{H} \vr{w} \cdot \vr{x}_1$ converges in probability to $c_1$.
\end{proposition}

\begin{proof}
    We observe that
    \[\frac{1}{H}\vr{w} \cdot \vr{x}_1 = \frac{c_1}{H}||\vr{x}_1||^2 + \frac{1}{H} \sum_{j=2}^V c_j (\vr{x}_1 \cdot \vr{x}_j) \,.\]
    The Law of Large Numbers suggests that the quantity $c_1 ||\vr{x}_1||^2 / H$ converges almost surely to $c_1$. 
    
    Denote the sum over cross products as
    \begin{equation} \label{eq:cross_term_sum}
    S_H = \frac{1}{H}\sum_{j=2}^V c_j (\vr{x}_1 \cdot \vr{x}_j)
    \end{equation}
    We observe that
    \begin{align*}
        \E\Hx{S_H} &= 0 \\
        \Var\Px{S_H} &< c^2\frac{V}{H}
    \end{align*}
    Under Chebyshev's inequality, for any $\varepsilon > 0$, we must have that
    \[\Pr\Hx{|S_H| > \varepsilon} \leq \frac{c^2 V}{\varepsilon^2 H} \,.\]
    Since $V = o(H)$, the quantity $|S_H|$ must converge in probability to 0. In this way, we have overall that $\frac{1}{H} \vr{w} \cdot \vr{x}_1 \rightarrow c_1$ in probability.
\end{proof}
This proposition indicates that, so long as the number of symbols $V$ remains small with respect to the hidden dimension $H$, we do indeed have that $\frac{1}{H} \vr{w}_i \cdot \vr{x}_1 = c_1^i$ as $H$ increases. Similar to our max margin argument formalized in Proposition \ref{prop:max_margin}, we require that the magnitude of all coordinates $c$ are the same. As before, the symmetry of the transitive inference task suggests that we generally expect the magnitudes of these coordinates to be similar, though we explore a notable exception below in Section \ref{app:breadthwise_scale}. We next apply Proposition \ref{prop:dual_signal} to understand scaling relationships in our transitive inference task.

\subsubsection{Depth-wise scaling} \label{sec:depth_wise_scale}
Our TI task is parameterized by two numbers: \textit{breadth} $B$ (the number of branches) and \textit{depth} $D$ (the length of each branch). The total number of unique symbols $V$, ignoring special tokens for padding and termination, must therefore be $V = BD$.

\paragraph{Direct prediction model.} To explain the scaling relationships in Figure \ref{fig:ti}, we first fix $B$ while increasing $D$ and $H$ to understand depth-wise scaling with model width. Recall from Section \ref{sec:max_margin_sol}: the max margin solution requires that, for weights $\vr{w}_i$ corresponding to a positive activation index $i \in \mathcal{I}^+$, coordinates for symbols that share the same branch must also share the same sign. That is, for symbol embeddings $\vr{x}_k^{n}, \vr{x}_\ell^{n}$, we require that $\text{sign}(\vr{w}_i \cdot \vr{x}_k^n) = \text{sign}(\vr{w}_i \cdot \vr{x}_\ell^n)$. However, since $V = \Theta(D)$, if $D = \omega(H)$, then $V = \omega(H)$. Proposition \ref{prop:dual_signal} no longer holds, and we cannot guarantee that $\text{sign}(\vr{w}_i \cdot \vr{x}_k^n) = \text{sign}(\vr{w}_i \cdot \vr{x}_\ell^n)$. Indeed, if the magnitudes of all sample coordinates remain approximately the same, then our sum over cross terms in Eq \eqref{eq:cross_term_sum} explodes, leaving us with chance-level output: $\Pr(\text{sign}(\vr{w}_i \cdot \vr{x}_k^n) = \text{sign}(\vr{w}_i \cdot \vr{x}_\ell^n) ) = 1/2$. For this reason, $D = \mathcal{O}(H)$ poses a natural upper bound on scaling $D$, agreeing with the linear scaling we observe empirically for direct prediction models in Figure \ref{fig:ti}.

In general, we expect this requirement for sign-sharing among weights with positive activation indices to be the bottleneck that determines the scaling. For weights with negative activation index, our max-margin solution only expects that in each branch, a portion of sample coordinates are positive, and a portion are negative. This mild condition is satisfied at initialization, and we conjecture that its impact on scaling is minimal.

\paragraph{RT model.} 
Given this argument for linear scaling, why does the RT model settle on a markedly slower depth-wise scaling? The reason may be that the model's context length also grows with $D$. To construct a chain for a task with depth $D$, the RT model requires $L = \Theta(D)$ context tokens $\vr{x}_1, \vr{x}_2, \ldots, \vr{x}_L$, of which $\vr{x}_2, \ldots, \vr{x}_L$ enumerate over all symbols belonging to a branch. Hence, an increasing segment of the context is dedicated to this enumeration, potentially diluting the desired signal: whether $\vr{x}_1$ belongs in the same branch as the others.

We may analyze this setup as a distinguishability problem. Recall that symbol embeddings are sampled from a standard Gaussian $\vr{x} \sim \mathcal{N}(\vr{0}, \vr{I}_H)$ and fixed throughout training. If $H$ is the embedding dimension, then with respect to the randomness of the embeddings, the distribution of a negative example $F$ may be $\frac{1}{L} (\vr{x}_1 + \vr{x}_2 + \ldots \vr{x}_L) \sim F \equiv \mathcal{N}\Px{\vr{0}, \frac{1}{L} \vr{I}}$. For a positive example, $\vr{x}_1$ is duplicated at some point in the enumeration $\vr{x}_2, \ldots, \vr{x}_L$, so the distribution of a positive example $\vr{T}$ becomes $\frac{1}{L}(2\vr{x}_1 + \vr{x}_2 + \ldots \vr{x}_{L-1}) \sim T \equiv \mathcal{N}\Px{\vr{0}, \frac{2 + L}{L^2}\vr{I}_H}$. The model's ability to learn the task becomes a problem of \textit{statistical indistinguishability}: given a sample from either $F$ or $T$, with what probability can the model determine its origin? We answer this question with help from Proposition \ref{prop:stat_indist_ar_length}, which establishes an upper bound on the total variation distance (TV) between $F$ and $T$. 

This description is precise only if we resample the embeddings between each draw of a positive or negative example. It may, however, continue to be accurate when the vocabulary size is large, and inputs are not frequently repeated. We indeed observe that this framing is predictive for the empirically observed scaling.

\begin{proposition} \label{prop:stat_indist_ar_length}
    Define the distributions $F \equiv \mathcal{N}(\vr{0}, \frac{1}{L} \vr{I}_H)$ and $T \equiv \mathcal{N}(\vr{0}, \frac{2 + L}{L^2} \vr{I}_H)$. Let $\text{TV}(F, T)$ be the total variation distance between $F$ and $T$. Then for $L \rightarrow \infty$, $\text{TV}(F, T) = \mathcal{O}(\sqrt{H} / L)$.
\end{proposition}

\begin{proof}
    From Pinsker's inequality, we know that
    \begin{equation} \label{eq:tv_pinsker}
    \text{TV}(F, T) \leq \sqrt{\frac{1}{2} D_{KL} (F \,||\, T)} \,.
    \end{equation}
    From the definitions of $F$ and $T$,
    \[D_{KL}(F \,||\, T) = \frac{H}{2} \Px{\frac{L}{2 + L} + \log \frac{2 + L}{L} - 1} \,.\]
    Taylor expanding this quantity around $L \rightarrow \infty$ reveals that
    \begin{align*}
        D_{KL}(F \,||\, T) &= \frac{H}{2} \Px{\frac{2}{L^2} - \frac{16}{3L^3} + \mathcal{O}\Px{\frac{1}{L^4}}} \\
        &= \mathcal{O}\Px{\frac{H}{L^2}} \,.
    \end{align*}
    In this way, we see that the KL divergence is bounded asymptotically by $\mathcal{O}(H / L^2)$. In fact, the function $\frac{H}{L^2}$ upper bounds the KL divergence for all $L$. To see how, define $u = 2 / L$ so that 
    \[g(u) \equiv \frac{2}{H} D_{KL} (F\,||\,T) = \frac{1}{1+u} + \log (1 + u) - 1 \,.\]
    Let
    \[f(u) = \frac{u^2}{2} - g(u) \,.\]
    Then $f(0) = 0$ and
    \begin{align*}
        f'(u) = \frac{u^2 (u + 2)}{(1 + u)^2},
    \end{align*}
    so $f'(u) > 0$ for $u > 0$. Hence, we must have $\frac{u^2}{2} \geq g(u)$ for all $u > 0$. Substituting $u = 2/L$, we must have that $D_{KL}(F\,||\,T) \leq H/L^2$ for all $L$. 
    
    Substituting our upper bound for the KL divergence into Eq \eqref{eq:tv_pinsker}, we conclude that
    \[\text{TV}(F, T) = \mathcal{O}(\sqrt{H} / L) \,.\]
\end{proof}

Returning to our original question, our model is presented with a sample from either $F$ or $T$, and must determine its origin. A classic result in statistical theory asserts that the probability of success is given by $\frac{1}{2}(1 + \text{TV}(F, T))$ (see for instance Theorem 2.2 in \cite{Tsybakov_nonparametric}). Since $L = \Theta(D)$, we expect that the model's accuracy is also bounded by $\frac{1}{2} + \mathcal{O}(\sqrt{H} / D)$. The high-accuracy contour should therefore exhibit the scaling $D = \mathcal{O}(\sqrt{H})$, which agrees with the square-root scaling we observe in Figure \ref{fig:ti}.

\subsubsection{Breadth-wise scaling} \label{app:breadthwise_scale}
A direct application of our argument in Section \ref{sec:depth_wise_scale} suggests that breadth-wise scaling should also be linear. However, linear scaling may be \textit{exceeded} if we also consider the possibility that the magnitude of sample coordinates differ between symbols.

Suppose we fix the task depth $D$, and scale the task breadth $B$ with model width $H$. Since the size of the vocabulary is $V = BD$, as we increase $B$, we note that $V = \Theta(B)$. As before, our max margin solution suggests that for a positive activation with index $i \in \mathcal{I}^+$ and any two symbol embeddings $\vr{x}_k^n, \vr{x}_\ell^n$ which belong to the same branch $n$, we require $\text{sign}(\vr{w}_i \cdot \vr{x}_k^n) = \text{sign}(\vr{w}_i \cdot \vr{x}_\ell^n)$. If the magnitude of all sample coordinates $c$ is about the same, then Proposition \ref{prop:dual_signal} suggests we should expect the scaling to be $B = o(H)$.

However, when scaling breadth, perhaps the structure of the task would allow us to bypass the assumption that all sample coordinates have the same magnitude. In particular, suppose that some fraction of the sample coordinates are zero. Their contribution to the cross sum in Eq \eqref{eq:cross_term_sum} would be removed, and the vocabulary may be increased.

Our max margin construction in Section \ref{sec:max_margin_sol} requires that all coordinates sharing a branch are uniformly positive or negative. What if we additionally allow some of them to be relatively small or zero? Since there are $H$ activations in our model, and $B$ total branches, suppose we allowed each activation to prioritize $B / H$ branches, and maintain nonzero coordinates only for these $B / H$ choices. Then Proposition \ref{prop:dual_signal} would suggest that if we now scale as $B / H = o(H)$ (or equivalently, $B = o(H^2)$), then we continue to have $\frac{1}{H}\vr{w}_i \cdot \vr{x}_1 = c_1$, and $\text{sign}(\vr{w}_i \cdot \vr{x}_k^n) = \text{sign}(\vr{w}_i \cdot \vr{x}_\ell^n)$. This is consistent with our observed quadratic scaling in both DP and RT models.

If our models indeed distribute sample coordinates in this fashion, then we should observe a linear relationship between the number of large, positive sample coordinates and the task breadth. We empirically observe precisely this trend in Figure \ref{fig:app_breadth_scale}.

\begin{figure}
    \centering
    \includegraphics[width=0.7\linewidth]{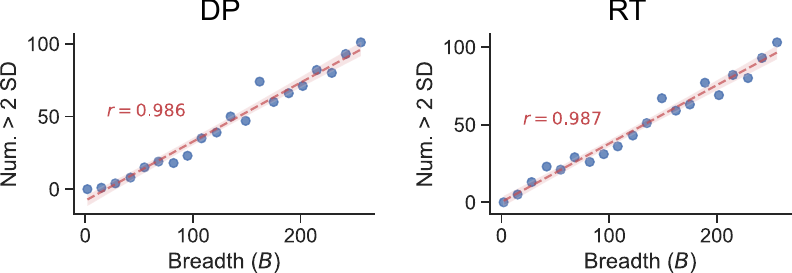}
    \caption{\textbf{Number of large sample coordinates increases with task breadth.} We compare across 20 models trained with varying task breadth, and estimate the sample coordinates from the weight vector corresponding to the largest positive readout weight. We determine that a sample coordinate is large if its magnitude is greater than 2 standard deviations from zero, compared across all other sample coordinates extracted from the weight vector. In both DP and RT models, we measure a strong linear relationship between task breadth and the number of large coordinates, supporting our conjecture about how models distribute sample coordinates to attain faster breadth-wise scaling.}
    \label{fig:app_breadth_scale}
\end{figure}

\paragraph{Distributing depth?} By distributing branches across $H$ activations, we attain a faster scaling than the linear rate suggested by our argument in Section \ref{sec:depth_wise_scale}. Could the same strategy be employed to improve our rate when scaling by depth?

The answer turns out to be no, since doing so blocks us from classifying positive examples correctly. Suppose we distribute $D$ depths among $H$ activations, with each activation implementing nonzero coordinates for $D / H$ depths. Then the proportion of nonzero coordinates per activation is $\frac{D / H}{D} = \frac{1}{H}$, which approaches zero as $H$ increases. Hence, there no longer exists a branch for which the proportion of positive coordinates is close to 1, which our max margin solution in Proposition \ref{prop:max_margin} requires to classify a positive example. For this reason, depth scaling remains limited to $D = o(H)$.

\subsection{Generalization outcomes} \label{app:generalization}
We next turn our attention to length generalization, and study why length generalization deteriorates quickly in direct-prediction models, but remains intact their reasoning trace counterparts.

\subsubsection{DP model}
We observe in Figure \ref{fig:ti} that length generalization decays in our DP model as task breadth $B$ increases. In the following section, we consider why length generalization deteriorates, and ballpark estimate the rate at which it decays with $B$.

We begin by providing a broad intuition for why test performance decays at all. Suppose we restrict our training distribution to relations with distance $ \leq k$, then test on distances $> k$. Let $\vr{w}_j$ be a weight vector with negative activation index $j \in \mathcal{I}^-$, and sample coordinates indexed by branch $n$ and depth $k$, $c_k^{j,n}$. Our max margin solution (Section \ref{sec:max_margin_sol}) suggests that a proportion of coordinates within a branch will be positive.

However, if $k \ll D$, then the \textit{placement} of these positive coordinates need not be uniform. Indeed, we should expect that $c_p^{j,n} + c_q^{j,n}$ is small or negative for $d(x_p^n, x_q^n) \leq k$, increasing the margin on these points in the training set. On the other hand, $c_p^{j,n} + c_r^{j,n}$ has no such constraints when $d(x_p, x_r) > k$, and may be correspondingly large. Indeed, some proportion of coordinates in each branch must be large and positive in order to produce correct negative classifications. Hence, positive coordinates will tend to have small or negative neighbors within $k$ distance. When longer distances are prompted at test time, these large, positive coordinates may be selected, increasing the contribution from negative activations. Activations with positive index are uniformly positive or negative within a branch, so the contribution from positive activations should remain the same for larger test distances as they are for shorter distances. The overall result is a false negative. Indeed, we find empirically that the vast majority of test error are false negatives (Figure \ref{fig:app_dp_gen_acc}).

\begin{figure}
    \centering
    \includegraphics[width=0.75\linewidth]{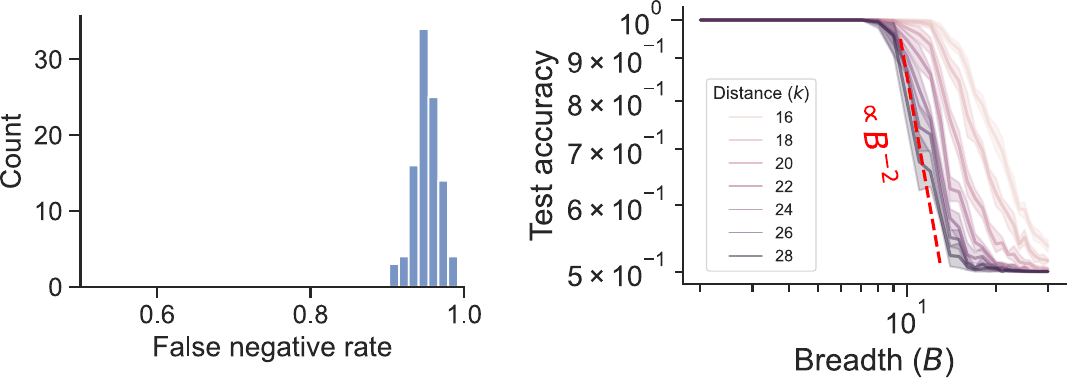}
    \caption{\textbf{Generalization in the direct prediction model.} \textit{Left}, we train 100 DP models with $B = 50$, $D = 10$, and training distance $k = 5$, then measure the false negative rate (FNR) on classifications for $k > 5$. FNR is very close to 1, supporting our assertion that logits become very negative when the DP model operates on longer distances. \textit{Right}, we measure the rate at which test accuracy falls as a function of task breadth, across varying test distances. Particularly for large test distances, the rate is close to our predicted $B^{-2}$ rate. Shading corresponds to 95 percent intervals estimated from 6 independently trained models each evaluated on 1000 examples.}
    \label{fig:app_dp_gen_acc}
\end{figure}

\subsection{Estimating the error.} Using these observations, we offer a back-of-the-envelope estimate for how test accuracy falls with increasing task breadth. We first conjecture how increasing breadth causes output logits to decrease, then use the rate of decrease to ballpark the classification accuracy, finding that test accuracy diminishes as $\mathcal{O}(B^{-2})$.

For a direct prediction model, the gradient of the logit $\hat{y} = \sum_{i=1}^H r_i \, \phi(\vr{w}_i \cdot (\vr{x}_p^{m} + \vr{x}_q^{n}))$ with respect to the hidden weights $\vr{w}_i$ is
\[\frac{\d y}{\d \vr{w}_i} \propto \vr{x}_p^{m} + \vr{x}_q^{n} \,.\]
This quantity is nonzero only if $\vr{w}_i \cdot (\vr{x}_p^{m} + \vr{x}_q^{n}) > 0$, or equivalently in sample coordinates, $c_{p}^{i,m} + c_{q}^{i,n} > 0$. Consider also the contribution from the loss $\mathcal{L}(y, \hat{y})$, for logit $\hat{y}$ and label $y$. If $\mathcal{L}$ is binary cross entropy, then $-\text{sign}\Px{\frac{\d \mathcal{L}}{\d \mathcal{\hat{y}}}} = \text{sign}(y)$. A gradient descent step in $\vr{w}_i$ on symbol embeddings $\vr{x}_p^m, \vr{x}_q^n$ therefore takes the following form, omitting the learning rate:
\[\Delta \vr{w}_i = \text{sign}(y)\, r_i (\vr{x}_p^m + \vr{x}_q^{n}) \mathbbm{1}\{c_p^{i,m} + c_q^{i,n} > 0\} \,.\]
If $c_p^{i,m} + c_q^{i,n} > 0$, then the resulting change in the output logit for $\vr{x}_p^m, \vr{x}_q^n$ is simply
\[\Delta\hat{y} = r_i (\Delta \vr{w}_i \cdot (\vr{x}_p^m + \vr{x}_q^n)) = \text{sign}(y) \, r_i^2 \,||\vr{x}_p^m + \vr{x}_q^n||^2 \,. \]
Since $r_i^2$ and $||\vr{x}_p^m + \vr{x}_q^n||^2$ are positive, we see that $\text{sign}(\Delta \hat{y}) = \text{sign}(y)$. Intuitively, output logits increase for positive examples ($y = 1$) and decrease for negative examples ($y = -1$).

If an example is negative, then the component symbols must belong to different branches. For a symbol $\vr{x}_p^m$, there are $B - 1$ other branches $n \neq m$ from which to form a valid negative example. If a consistent proportion of these $\mathcal{O}(B)$ other pairings produce a negative contribution on the output logit, then we conjecture that the negative contribution to a logit over the course of training may also be of order $\mathcal{O}(B)$. Specifically, we assume that a logit decomposes as $\hat{y} = \hat{y}_0 - \beta$, where $\hat{y} = \hat{y}_0$ when $B = 2$ (the smallest possible task breadth $B$) and $\beta = \mathcal{O}(B)$.

In this way, as $B$ increases, all logits decrease, resulting in the false negative classifications we observe in Figure \ref{fig:app_dp_gen_acc}. To estimate the impact on classification accuracy, we need to estimate $\Pr(\hat{y} > 0)$. If the distribution of $\hat{y}$ over test examples has a fixed variance $\sigma^2$, then Cantelli's inequality suggests that
\[\Pr(\hat{y}> 0) = \Pr(\hat{y} - (\hat{y}_0 - \beta) > -(\hat{y}_0 - \beta)) \leq \frac{\sigma^2}{\sigma^2 + (\hat{y}_0 - \beta)^2} \,.\]
If $\sigma^2$ does not vary with task breadth and $\beta^2 \gg \sigma^2 + \hat{y}_0^2$, then $\Pr(\hat{y} > 0) = \mathcal{O}\Px{B^{-2}}$. In Figure \ref{fig:app_dp_gen_acc}, we confirm that this scaling is indeed accurate.

\subsubsection{RT model}
In Figure \ref{fig:ti}, we observe that generalization in RT models remains strong regardless of task breadth $B$. This outcome is due to an inductive bias for a generalizing solution on our TI task facilitated by the reasoning trace. Specifically, we argue that for any test example, there exists a permutation that transforms it into a training example. In the absence of positional encodings, the model is invariant with respect to permutations in input position (up to the query token); if the model attains perfect training accuracy, then it attains perfect test accuracy also.

\begin{proposition} \label{prop:cot_test_to_train}
    Consider a task with training distance $k \geq 2$. Denote a test example as any sequence of symbols $x_1, x_2, \ldots, x_L$ such that $d(x_1, x_2) > k$ and $x_3, \ldots, x_L$ enumerate the remaining symbols that belong in the same branch as $x_2$. Denote a training example similarly, except $d(x_1, x_2) \leq k$. Let $\pi: \N^{L - 1} \rightarrow \N^{L-1}$ denote a permutation over $L - 1$ indices. Then for any test example, we may construct a permutation $\pi$ such that $x_{\pi(1)}, x_{\pi(2)}, \ldots, x_{L}$ is a training example. Note, the query token $x_L$ is not permuted.
\end{proposition}

\begin{proof}
Given a test example $x_1, x_2, \ldots, x_L$, the set of symbols $x_2, \ldots, x_L$ enumerate \textit{all} nodes in a branch. Then there must exist an index $2 < j < L$ such that $d(x_1, x_j) \leq k$. Because the test distance must be at least 2, there are at least 2 symbols that satisfy $d(x_1, x_j) \leq k$, allowing us to proceed even if $d(x_1, x_L) \leq k$. Construct a permutation by swapping $x_2$ with $x_j$: $\pi(2) = j$ and $\pi(j) = 2$, while preserving the remaining indices. The set $x_{\pi(2)}, \ldots, x_{\pi(L - 1)}$ continues to enumerate all nodes in a shared branch, while $d(x_{\pi(1)}, x_{\pi(2)}) < k$. Thus, $x_{\pi(1)}, x_{\pi(2)}, \ldots, x_{L}$ must be a training example.
\end{proof}

Proposition \ref{prop:cot_test_to_train} demonstrates that there always exists a permutation that transforms a test example into a train example. Our RT model is trained without positional encodings, and is therefore permutation invariant in positions up to the query token. If the model attains perfect training accuracy, then it should attain perfect test accuracy also, irrespective of task topology. Hence, we conclude that the RT model generalizes well on our TI task. 

\clearpage
\section{Task and model details} \label{app:details}
Our code is available on GitHub at \url{https://github.com/wtong98/boule-or-baguette/tree/main}. Exact configurations used to reproduce all plots can be found there.

The PITA dataset is available on Hugging Face at \url{https://huggingface.co/datasets/williamtong105/pita}.

Experiments were conducted on the Harvard FASRC cluster. Each LLM finetuning run took approximately 2 - 10 hours on a single H100 GPU. Results for the transitive inference task were conducted on a mixed cluster of A100 and H100 GPUs, requiring on the order of 1000 GPU hours total to finish. 

Below, we summarize the details of our task setups and model configurations.

\subsection{Making a PITA}
The PITA dataset consists of statements in propositional logic together with a proof that demonstrates whether the statement is true or false. To convert these inputs into tokens that an LLM can process, we proceed through the following three steps:
\begin{enumerate}
    \item Construct a proof
    \item Translate the proof into Lean tactics and states
    \item Structure the Lean output with XML delimiters
\end{enumerate}
Appendix \ref{app:prop_logic} provides a brief primer on propositional logic, and the process by which proofs may be constructed for propositional statements. Our proofs are constructed using a standard \textit{focused proof search} procedure \cite{liang2009focused_proof_search} based on the implementation from \cite{an_propl}. We then translate the resulting sequence of inference rules discovered by the search procedure into their corresponding Lean tactics, and poll the Lean proof assistant for intermediate proof states following each tactic application. 

Following \cite{an_propl}, we also include \texttt{backtrack} tactics. These are a non-standard extension that directs Lean to resume tracking from the last unresolved proof state. Including \texttt{backtrack} allows us to ``undo" erroneous proof steps. If the current sequence of tactics does not terminate in a tautology, \texttt{backtrack} brings us to the last state where we can apply a different tactic to explore a different branch of the proof. Figure \ref{fig:pita} illustrates an example with a \texttt{backtrack} application. \cite{an_propl} find that including these ``mistake" demonstrations increases the model's theorem-proving success by a substantial margin. Including \texttt{backtrack} also allows us to treat proofs for both true and false statements in a uniform way: if the model explores all reasonable tactics through \texttt{backtrack} and discovers no valid tautologies, then the proof must be false.

The \texttt{backtrack} tactic itself is implemented using the same procedure as in \cite{an_propl}. The proof search procedure constructs a tree of possible inference rule applications and their resulting proof states. A proof of a true statement may be extracted as the shortest path between the starting state and a successful terminating leaf. Our traces instead use the same deterministic depth-first search over this tree, with alternatives expanded in a fixed order. Once the traversal encounters a failing terminal leaf, we inject a \texttt{backtrack} tactic to represent a return back to the last node with alternative paths, and proceed down the next viable branch.

Finally, given the sequence of tactics and states, we prepare the final input to the LLM by wrapping the tactics and state components in XML tags. The XML has the following structure:

\begin{Verbatim}[commandchars=\\\{\}]
\PYG{n+nt}{\PYGZlt{}state} \PYG{n+na}{id=}\PYG{l+s}{\PYGZdq{}0\PYGZdq{}}\PYG{n+nt}{\PYGZgt{}}
  \PYG{n+nt}{\PYGZlt{}if}\PYG{n+nt}{\PYGZgt{}} ... \PYG{n+nt}{\PYGZlt{}/if\PYGZgt{}}       \PYG{c+cm}{\PYGZlt{}!\PYGZhy{}\PYGZhy{} propositions that are given \PYGZhy{}\PYGZhy{}\PYGZgt{}}
  \PYG{n+nt}{\PYGZlt{}if}\PYG{n+nt}{\PYGZgt{}} ... \PYG{n+nt}{\PYGZlt{}/if\PYGZgt{}}
  \PYG{n+nt}{\PYGZlt{}then}\PYG{n+nt}{\PYGZgt{}} ... \PYG{n+nt}{\PYGZlt{}/then\PYGZgt{}}   \PYG{c+cm}{\PYGZlt{}!\PYGZhy{}\PYGZhy{} goal proposition to demonstrate \PYGZhy{}\PYGZhy{}\PYGZgt{}}
\PYG{n+nt}{\PYGZlt{}/state\PYGZgt{}}
\PYG{n+nt}{\PYGZlt{}tactic}\PYG{n+nt}{\PYGZgt{}} ... \PYG{n+nt}{\PYGZlt{}/tactic\PYGZgt{}}
\PYG{n+nt}{\PYGZlt{}state} \PYG{n+na}{id=}\PYG{l+s}{\PYGZdq{}1\PYGZdq{}}\PYG{n+nt}{\PYGZgt{}} ... \PYG{n+nt}{\PYGZlt{}/state\PYGZgt{}}
\PYG{n+nt}{\PYGZlt{}tactic}\PYG{n+nt}{\PYGZgt{}} ... \PYG{n+nt}{\PYGZlt{}/tactic\PYGZgt{}}
\PYG{n+nt}{\PYGZlt{}state} \PYG{n+na}{id=}\PYG{l+s}{\PYGZdq{}2\PYGZdq{}}\PYG{n+nt}{\PYGZgt{}} ... \PYG{n+nt}{\PYGZlt{}/state\PYGZgt{}}

\PYG{c+cm}{\PYGZlt{}!\PYGZhy{}\PYGZhy{} ... and so on ...  \PYGZhy{}\PYGZhy{}\PYGZgt{}}

\PYG{n+nt}{\PYGZlt{}backtrack} \PYG{n+na}{to=}\PYG{l+s}{\PYGZdq{}m\PYGZdq{}} \PYG{n+nt}{/\PYGZgt{}}  \PYG{c+cm}{\PYGZlt{}!\PYGZhy{}\PYGZhy{} possible backtrack to state `m` \PYGZhy{}\PYGZhy{}\PYGZgt{}}

\PYG{c+cm}{\PYGZlt{}!\PYGZhy{}\PYGZhy{} ... and so on ...  \PYGZhy{}\PYGZhy{}\PYGZgt{}}

\PYG{n+nt}{\PYGZlt{}state} \PYG{n+na}{id=}\PYG{l+s}{\PYGZdq{}n\PYGZdq{}}\PYG{n+nt}{\PYGZgt{}}
   \PYG{n+nt}{\PYGZlt{}complete} \PYG{n+nt}{/\PYGZgt{}}        \PYG{c+cm}{\PYGZlt{}!\PYGZhy{}\PYGZhy{} indicates that the current subgoal is complete \PYGZhy{}\PYGZhy{}\PYGZgt{}}
\PYG{n+nt}{\PYGZlt{}/state\PYGZgt{}}
\PYG{n+nt}{\PYGZlt{}success} \PYG{n+nt}{/\PYGZgt{}}            \PYG{c+cm}{\PYGZlt{}!\PYGZhy{}\PYGZhy{} or `\PYGZlt{}failure /\PYGZgt{}` \PYGZhy{}\PYGZhy{}\PYGZgt{}}
\end{Verbatim}

If the propositional statement is true, we terminate in a \texttt{<success />} tag. Otherwise, we terminate in a \texttt{<failure />} tag. These tags constitute the special classification token that the model must output to indicate whether the statement is true or false. The entire sequence is then passed to the model-specific tokenizer and processed.

Models are trained on a prompt/completion input format, where loss is computed only on the completion tokens. The prompt is delineated by all tokens up to \texttt{<state id="0">...</state>}. The completion is the remainder of the proof for the RT model, or only the success/failure tag for the DP model. To ensure that the model tokenizer cleanly separates the prompt from the completion, we include an addition double bar \texttt{||} separating the prompt from completion, so the input looks like \texttt{<state id="0">...</state>||<tactic>...}. 

\subsection{PITA split construction} \label{app:pita_splits}
PITA is composed of four different splits.

We discuss the details of each split and provide examples below.

\paragraph{\textsc{Full.}} This split corresponds to the full space of propositional statements, enumerating over all 3 connectives $\rightarrow, \vee, \wedge$ as well as propositional atoms $\top,\bot$ and 3 variables $p, q, r$. We exhaustively enumerate all statements up to 5 atoms in size (including enumerations over parenthetical groups), resulting in 3.6 million statements and 1.4 billion tokens. All examples are well under the 32K max context length of our models. Examples from this split are
\begin{align*}
    &p \rightarrow p \vee q \\
    (&p \wedge q) \rightarrow p \vee (r \rightarrow \bot) \\
    (&p \vee q) \wedge r \wedge \top
\end{align*}

\paragraph{\textsc{Imply.}} This split corresponds to the implicational fragment of propositional logic. The setting is the same as \textsc{Full}, except we restrict to a single connective $\rightarrow$. Somewhat surprisingly, doing so does not reduce the space of expressible propositions. The implicational fragment is \textit{functionally complete}, so every propositional statement (interpreted as an input/output mapping between truth values) expressible in \textsc{Full} is also expressible in \textsc{Imply}. In this sense, the task is essentially the same as \textsc{Full}, but formatted differently. We exhaustively enumerate all statements up to 7 atoms in size, resulting in 11 million statements and 11.6 billion tokens. Virtually all examples are under the 32K max context length of our models. Examples from this split are
\begin{align*}
    &p \rightarrow q \rightarrow p \\
    (&p \rightarrow \bot) \rightarrow (p \rightarrow \top) \\
    &p \rightarrow ((q \rightarrow r) \rightarrow q)
\end{align*}

\paragraph{\textsc{Or.}} This split corresponds to a membership checking task expressed in propositional form. Each statement has the form $p \rightarrow q_1 \vee q_2 \vee \ldots \vee q_n$, where exactly half of the statements have $p \in \{ q_1 \ldots q_n\}$ (corresponding to true statements) and half the statements do not (corresponding to false statements). A random parenthetical grouping is sampled for the variables on the right hand side. The variable names for $q_1 \ldots q_n$ are forced to be unique, and drawn from a pool of 100 thousand possible names to minimize the chance of name collisions. We sample 8.7 million of these statements, resulting in 78.5 billion tokens. A little less than 1 percent of these examples exceed the 32K max context length of our model; these examples are removed from the dataset. Examples from this split are
\begin{align*}
    &p \rightarrow q_1 \vee p \vee q_2 \vee q_3 \vee q_4 \\
    &p \rightarrow q_1 \vee q_2 \vee q_3 \vee q_4 \\
    &p \rightarrow q_1 \vee q_2 \vee q_3 \vee q_4 \vee p \vee q_5 \\
\end{align*}

\paragraph{\textsc{PHP.}} The pigeonhole principle is a classic counting argument that states the following: given $m$ pigeons and $n$ holes, if $m > n$, there there must be at least one hole occupied by two or more pigeons. This statement can be framed in propositional logic in the following way. Let $p_{ij}$ be the proposition ``pigeon $i$ occupies hole $j$". Then the pigeonhole principle can be written as
\begin{equation} \label{eq:php}
\bigwedge_{i=1}^m \Px{p_{i1} \vee p_{i2} \ldots \vee p_{in}} \rightarrow \bigvee_{j=1}^n \, \bigvee_{i_1 \neq i_2} (p_{i_1 j} \wedge p_{i_2j}) \,.
\end{equation}
The antecedent corresponds to the condition that every pigeon must occupy a hole, and the consequent corresponds to the condition that at least one hole contains at least two pigeons. Proving these propositions reduces essentially to a constraint-satisfaction problem. Many such problems feature exponential time complexity to solve. Similarly, proving PHP propositions may take exponentially many steps in $m$ and $n$, making them a common choice for stress-testing proof systems \cite{haken1985php_hard,balsiger2000php_benchmarks}. We use a slightly harder version of this task (still commonly used a stress-test) where each pigeon may occupy only a subset of holes, and only a subset of pairs are checked for double-occupancy. We sample 36 of these conditions, and exhaustively enumerate over all possible parenthetical groupings, resulting in 51 thousand statements and 1.0 billion tokens. Examples from this split are
\begin{align*}
    (&p_{11} \vee p_{13}) \wedge (p_{23}) \rightarrow (p_{12} \wedge p_{22}) \vee (p_{13} \wedge p_{23}) \\
    (&p_{12}) \wedge (p_{31} \vee p_{32} \vee p_{33}) \rightarrow (p_{11} \wedge p_{21}) \\
    (&p_{21} \vee p_{22}) \rightarrow (p_{11} \wedge p_{21}) \vee (p_{11} \wedge p_{31}) \vee (p_{22} \wedge p_{32})
\end{align*}

Each split is characterized by its \textit{depth} and \textit{breadth}. Depth is calculated as the unique number of proof states required to prove the statement, taken as the largest \texttt{id} number on the \texttt{<state>} tags. Hence, if the model backtracks to an earlier state, it is not double-counted.

Breadth is counted as the number of unique statements that can be enumerated for propositions of a given size, where size is measured by the number of atoms in the proposition. The \textit{or} ($\vee$) and \textit{and} ($\wedge$) connectives are commutative and associative, so we consider statements like $(p \vee q) \vee r$ to be the same as $q \vee (p \vee r)$, and do not count them as two unique statements. Furthermore, we distinguish propositional variables only by identity, rather than name. Hence, the statement $p \rightarrow q \rightarrow p$ is considered the same as $p' \rightarrow q' \rightarrow p'$, though the variable names are different. More precisely, we quotient by bijective variable renaming, preserving the equality pattern of variable occurrences and leaving $\top$ and $\bot$ fixed: $p \rightarrow q$ and $q \rightarrow p$ count once, while $p \rightarrow p$ is distinct. Reassociation and permutation of operands of $\vee$ or $\wedge$ are also identified. We do not quotient arbitrary logical equivalence: distributivity, implication elimination, and other semantic identities do not merge templates unless they follow from these stated symmetries. This quotient defines the breadth count; different parenthesizations and variable names can still appear as training examples and yield different search traces.

\subsection{PITA model configurations}
All models are finetuned using Hugging Face \cite{wolf2019huggingface}. Inference is performed using vLLM \cite{kwon2023vllm}. Specific training and inference setup are available on GitHub: \url{https://github.com/wtong98/boule-or-baguette}

Models are finetuned using QLoRA \cite{dettmers2023qlora} with rank $r = 16$, scaling factor $\alpha = 32$, and dropout probability $p = 0.05$. We target all linear modules. Our optimizer is 8 bit paged AdamW \cite{dettmers20218_8bit,loshchilov2017adamw}. We find that a relatively large learning rate of $2 \times 10^{-4}$ works well, with a linear warm-up for 500 iterations followed by constant learning rate. We use a batch size of 32 examples. We find that accuracy typically plateaus after 2500 iterations (80 thousand examples) across all splits and prolonged training does not improve it, so we terminate training after 2500 iterations and evaluate the model. A maximum context length of 32 thousand tokens is enforced across all models. All models we evaluate have a native context window of at least 32 thousand. 

Evaluation is conducted by allowing the model to autoregressively generate the next token until it produces a \texttt{<success />} or \texttt{<failure />} tag. If the model does not generate one of these tags before reaching its max context length, we terminate generation and designate that trial as a failure. Otherwise, classification accuracy is evaluated based on the terminating tag. Next tokens are sampled greedily.

\subsection{TI task details}
Our transitive inference task is similar to the forward-chaining setups prevalent in the TI literature \cite{lippl_ti,kay_ti,stachenfeld_ti}. The key difference is that rather than use a single chain, we employ $B$ parallel chains of inferences, and the model must additionally learn cross-chain inferences.

We offer additional details on the sampling procedure for positive and negative examples. Suppose we have symbols $s_i^n$ indexed by depth $i$ and branch $n$. To construct an example with label $y = 1$, we first sample a symbol $s_i^n$ uniformly at random, then sample its partner $s_j^n$ where $i<j$, with $|j - i| \leq k$ during training and $|j - i| > k$ during testing, also uniformly at random from all eligible choices. To construct an example with a negative label, we perform the same procedure but ensure that the two symbols fall on different branches, resulting in a pair $s_i^n, s_j^m$ where $n \neq m$. Note, another valid way to construct a negative example would be to reverse a relation (e.g. $s_2^n \rightarrow s_1^n$). We omit this case for simplicity, finding it unnecessary to reproduce the generalization characteristics we observed in PITA, and leave its exploration for future work.

\subsection{TI model configurations}
We use single-layer, single-head, decoder-only Transformers with no positional encodings, no layer norm, no bias, and no residual connections. Embeddings are fixed at initialization and do not change over the course of training. Transformers are initialized using $\mu$P parameterization \cite{yang2021mup,bordelon2022mup}. Specifically, for the model given in Section \ref{sec:model}, we initialize the readout weights as $\vr{w}_r \sim \mathcal{N}(\vr{0}, \vr{I}_H)$. We initialize all other weight matrices ($\vr{K}, \vr{Q}, \vr{V}, \vr{W}_1$) with entries sampled iid from $\mathcal{N}(0, H)$. All activations are additionally scaled by a prefactor $\frac{1}{H}$ before passing to the next layer.

We use a learning rate of $1 \times 10^-2$ and train our models using AdamW for 100 thousand iterations. To ensure the model sees every symbol about the same number of times, we set our batch size to be $BD$, where $B$ is the number of branches and $D$ is the depth of the task. Unless otherwise stated, all evaluations are performed using 1024 samples.

\clearpage
\section{Additional figures and results} \label{app:additional_fig}

\subsection{Class proportions and class-balanced evaluation} \label{app:class_balance}
Table \ref{tab:pita_class_proportions} gives the label proportions before and after the proof-length cutoff. The largest shift occurs in \textsc{Imply}: the training distribution is overwhelmingly true, whereas the longer-proof test distribution is predominantly false.

\begin{table}[!ht]
    \centering
    \caption{\textbf{PITA train/test class proportions.} Entries are percentages of each split's training-length pools.}
    \label{tab:pita_class_proportions}
    \begin{tabular}{lrrrr}
        \toprule
        Split & Train true & Train false & Test true & Test false \\
        \midrule
        \textsc{Full}  & 35.6 & 64.4 & 33.6 & 66.4 \\
        \textsc{Imply} & 99.0 &  1.0 & 14.3 & 85.7 \\
        \textsc{Or}    & 64.6 & 35.4 & 38.7 & 61.3 \\
        \textsc{PHP}   & 39.6 & 60.4 & 13.2 & 86.8 \\
        \bottomrule
    \end{tabular}
\end{table}

We reanalyze the saved out-of-distribution evaluations. The saved confusion summaries record no-label outputs only in aggregate, without their gold class. Accordingly, the classwise recalls in Table \ref{tab:pita_balanced} condition on emitting a label: true recall is $\mathrm{TP}/(\mathrm{TP}+\mathrm{FN})$ and false recall is $\mathrm{TN}/(\mathrm{TN}+\mathrm{FP})$. These conditional recalls can be high even when many generations fail to terminate. Raw accuracy counts all no-label outputs as incorrect.

To include no-label errors in balanced accuracy, let $U$ be the no-label count and let $u\in[0,U]$ be the unknown number of these outputs whose gold label is true. Then
\begin{equation} \label{eq:balanced_bounds}
    \mathrm{BA}(u)=\frac{1}{2}\left[
    \frac{\mathrm{TP}}{\mathrm{TP}+\mathrm{FN}+u}
    +\frac{\mathrm{TN}}{\mathrm{TN}+\mathrm{FP}+U-u}
    \right].
\end{equation}
We report the minimum and maximum over possible allocations of $U$, averaged across runs. These are identification bounds, not confidence intervals. When $U=0$, the bounds coincide with exact balanced accuracy. This reweighting gives equal weight to the two classes; we have not generated a separate class-balanced test sample.

\begin{table}[!ht]
    \centering
    \caption{\textbf{Classwise accuracy and balanced-accuracy bounds.} Qwen2.5-Coder-32B is used for \textsc{Full}, \textsc{Imply}, and \textsc{Or}, and 7B for \textsc{PHP}. All values are percentages. Raw accuracy, conditional classwise recall ($\dagger$), and no-label rate are mean $\pm$ sample standard deviation across $n$ independently trained adapters; duplicate evaluations of an adapter are averaged first. The last column gives mean lower and upper bounds from Eq.~\eqref{eq:balanced_bounds}, including no-label errors.}
    \label{tab:pita_balanced}
    \small
    \setlength{\tabcolsep}{3pt}
    \begin{tabular}{llrrrrrl}
        \toprule
        Split & Target & $n$ & Raw acc. & True recall$^\dagger$ & False recall$^\dagger$ & No label & BA bounds \\
        \midrule
        \textsc{Full} & DP & 9 & $78.6\pm6.2$ & $85.0\pm10.7$ & $74.3\pm9.7$ & $0.0\pm0.0$ & $79.7$ \\
                      & RT & 9 & $88.3\pm5.2$ & $90.4\pm6.1$ & $90.6\pm5.7$ & $2.3\pm3.5$ & $[87.9,88.9]$ \\
        \textsc{Imply} & DP & 9 & $34.8\pm16.8$ & $91.0\pm20.5$ & $27.1\pm20.8$ & $0.0\pm0.0$ & $59.1$ \\
                       & RT & 9 & $76.4\pm25.1$ & $60.6\pm19.9$ & $80.3\pm31.6$ & $1.9\pm2.4$ & $[66.6,69.6]$ \\
        \textsc{Or} & DP & 10 & $100.0\pm0.0$ & $100.0\pm0.0$ & $100.0\pm0.0$ & $0.0\pm0.0$ & $100.0$ \\
                    & RT & 5 & $96.6\pm4.3$ & $99.0\pm1.4$ & $100.0\pm0.0$ & $3.0\pm4.0$ & $[96.2,97.1]$ \\
        \textsc{PHP} & DP & 10 & $83.4\pm4.4$ & $100.0\pm0.0$ & $80.9\pm4.9$ & $0.0\pm0.0$ & $90.5$ \\
                     & RT & 9 & $36.6\pm19.5$ & $88.3\pm24.7$ & $57.9\pm27.4$ & $49.4\pm18.4$ & $[33.8,59.4]$ \\
        \bottomrule
    \end{tabular}
\end{table}

For these selected models, RT's lower bound exceeds DP on \textsc{Full} and \textsc{Imply}, while RT's upper bound remains below DP on \textsc{Or} and \textsc{PHP}. Thus the direction of the reversal survives every allocation of the unrecorded no-label classes in these evaluations. Label balancing nevertheless changes the size of the gaps substantially, especially on \textsc{Imply}. This comparison does not establish the same ordering at every model size, or isolate topology from syntax, vocabulary, token length, and proof-search behavior.

\subsection{Lean validation of generated traces} \label{app:trace_validity}
We evaluate 100 randomly sampled longer-proof examples per split using one retrained Qwen2.5-Coder-7B RT checkpoint per split. The \textsc{Full}, \textsc{Imply}, and \textsc{Or} checkpoints use 2000 training updates; the \textsc{PHP} checkpoint uses 500. We distinguish the following metrics:
\begin{itemize}
    \item \textit{Executable trace}: a structurally complete output with a terminal label, accepted tactics, resolvable backtracks, and emitted goal states consistent with Lean throughout the trace, including abandoned branches. State comparisons normalize whitespace and pretty-printer differences; exact hypothesis-context equality is not required.
    \item \textit{Valid proof}: a structurally complete predicted-success output with resolvable backtracks whose final active tactic script closes the theorem in Lean without \texttt{sorry} or \texttt{admit}. An invalid abandoned branch can prevent whole-trace executability without invalidating the final proof.
    \item \textit{Correct executable trace}: an executable trace ending either in a valid proof with the correct success label or in a failure label for a gold-false statement. A correct failure label is not itself a Lean certificate of falsity.
    \item \textit{Exact match}: equality to the complete gold completion after removing outer whitespace only. \textit{Mean valid prefix}: the mean fraction of tactic/backtrack steps accepted from the start, stopping at the first invalid operation or mismatched emitted goal state.
\end{itemize}

\begin{table}[!ht]
    \centering
    \caption{\textbf{Generated trace quality.} All entries are percentages over 100 examples per split, except the last column, which conditions on a correct final label. Executability and proof validity are distinct; the last column is $\Pr(\text{correct executable trace}\mid\text{correct label})$.}
    \label{tab:pita_trace_quality}
    \small
    \setlength{\tabcolsep}{4pt}
    \begin{tabular}{lrrrrrr}
        \toprule
        Split & Label acc. & Exact match & \shortstack{Mean valid\\prefix} & Executable & \shortstack{Correct\\executable} & \shortstack{Correct executable\\/ correct label} \\
        \midrule
        \textsc{Full} & 86.0 & 30.0 & 86.5 & 76.0 & 72.0 & 83.7 \\
        \textsc{Imply} & 59.0 & 1.0 & 73.5 & 40.0 & 23.0 & 39.0 \\
        \textsc{Or} & 95.0 & 73.0 & 97.7 & 86.0 & 86.0 & 90.5 \\
        \textsc{PHP} & 48.0 & 0.0 & 82.3 & 37.0 & 30.0 & 62.5 \\
        \bottomrule
    \end{tabular}
\end{table}

\begin{table}[!ht]
    \centering
    \caption{\textbf{Proof validity and conditional final-label accuracy.} Valid proof rates use all 100 examples per split as the denominator. The next column conditions on a predicted success label and gives the corresponding counts. The final two columns condition on executability and proof validity, respectively. A dash indicates that no valid proof was generated, so the conditional accuracy is undefined.}
    \label{tab:pita_proof_validity}
    \small
    \begin{tabular}{lrrrr}
        \toprule
        Split & Valid proof & \shortstack{Valid proof /\\predicted success} & \shortstack{Label acc. /\\executable} & \shortstack{Label acc. /\\valid proof} \\
        \midrule
        \textsc{Full} & 19.0 & 65.5 (19/29) & 94.7 & 100.0 \\
        \textsc{Imply} & 1.0 & 2.7 (1/37) & 62.5 & 100.0 \\
        \textsc{Or} & 32.0 & 100.0 (32/32) & 100.0 & 100.0 \\
        \textsc{PHP} & 0.0 & 0.0 (0/15) & 91.9 & --- \\
        \bottomrule
    \end{tabular}
\end{table}

The valid prefixes in Table \ref{tab:pita_trace_quality} show substantial local agreement with the prover, but correct final labels frequently occur without correct executable traces, especially on \textsc{Imply} and \textsc{PHP}. Table \ref{tab:pita_proof_validity} further shows why executability alone should not be interpreted as successful theorem proving. All 52 generated Lean-valid proofs have the correct final label; eight differ from the gold completion. As a harness check, all 82 gold-success completions in these samples close their theorems in Lean. One contains an invalid speculative branch later discarded by backtracking, so whole-trace executability for the gold-success control is 98.8\%, despite 100\% final-proof validity. These results cover one checkpoint per split and do not constitute a full validation sweep over model sizes or seeds.

\subsection{Target-token-matched continuation} \label{app:continuation}
The main RT and DP comparisons match training examples but differ in supervised target-token count and inference-time generation. To test whether auxiliary proof supervision alone explains the RT advantage, we branch from the same trained RT adapter in each paired run. One branch continues under the DP objective for 500 updates, learning to emit the label immediately after already receiving proof supervision. The other continues under the RT objective for an integer number of updates chosen to approximately match the DP branch's additional supervised target tokens. The token budget accounts for completion masking, the end-of-sequence token, tokenizer boundaries, and context truncation; RT gradient accumulation is adjusted when needed to reduce the mismatch. Both branches use fresh optimizers.

\begin{table}[!ht]
    \centering
    \caption{\textbf{Out-of-distribution accuracy after target-token-matched continuation.} Entries are percentages, reported as median [first quartile, third quartile], with quartiles computed by linear interpolation. The last column gives the available DP/RT run counts (five repetitions were planned). Both branches share the same proof-supervised initialization. Source checkpoints use 2000 RT updates for \textsc{Full}, \textsc{Imply}, and \textsc{Or}, and 500 for \textsc{PHP}.}
    \label{tab:pita_continuation}
    \begin{tabular}{llrrc}
        \toprule
        Split & Qwen2.5-Coder & DP continuation & RT continuation & Runs (DP/RT) \\
        \midrule
        \textsc{Full} & 7B & 69.0 [69.0, 74.0] & 82.0 [81.0, 91.0] & 5/5 \\
        \textsc{Imply} & 7B & 30.0 [24.0, 36.0] & 57.0 [49.0, 77.0] & 5/5 \\
        \textsc{Or} & 7B & 100.0 [88.7, 100.0] & 96.5 [95.8, 97.2] & 4/4 \\
        \textsc{PHP} & 0.5B & 48.0 [16.0, 79.5] & 2.0 [0.0, 18.0] & 4/5 \\
        \bottomrule
    \end{tabular}
\end{table}

The median reversal persists: RT continuation is better on \textsc{Full} and \textsc{Imply}, while DP continuation is better on \textsc{Or} and \textsc{PHP}. Thus RT's advantage on the broad splits persists even when the direct-output model has received the same initial proof supervision. The short \textsc{PHP} initialization lowers performance relative to the main experiments, and the small sample has substantial variability.

\clearpage
\subsection{Additional figures}

\begin{figure}[!h]
    \centering
    \includegraphics[width=0.8\linewidth]{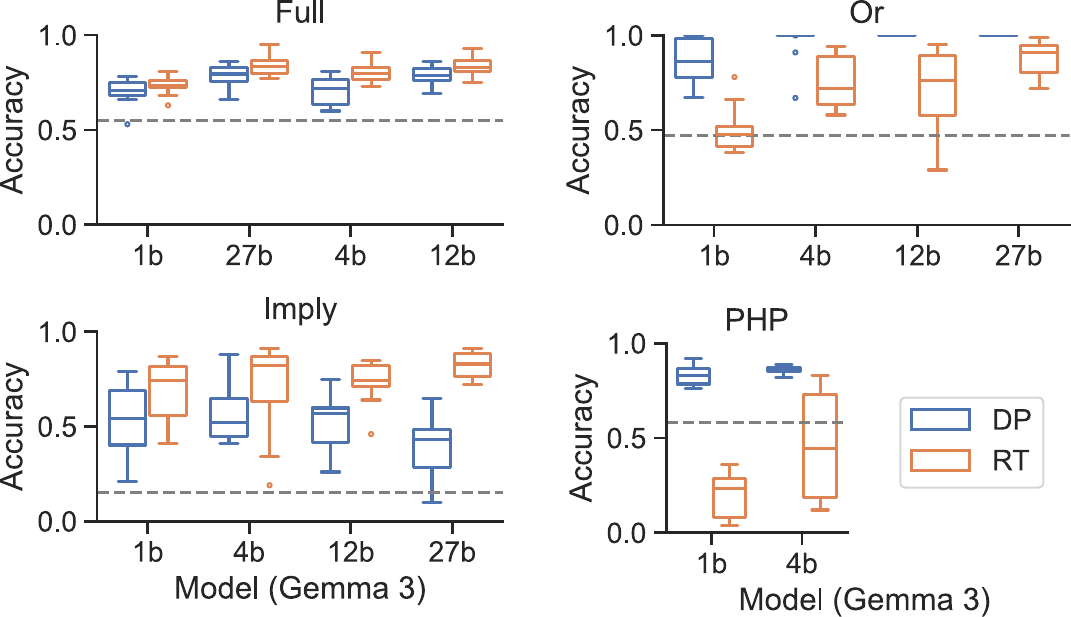}
    \caption{\textbf{Generalization accuracy in Gemma 3 on PITA splits.} The setting is the same as in Figure \ref{fig:pita_res}, but evaluated on a selection of the Gemma 3 family of models \cite{team2025gemma}. The trends are the same as observed for Qwen2.5-Coder, in which shallow and wide \textit{boule-shaped} splits favor RT models, while the reverse is true for deep and narrow \textit{baguette-shaped} splits. Computational constraints prevent training models larger than 10B on \textsc{PHP}. Boxplots are constructed from 10 runs, where each model is evaluated on 100 test samples. Box lines illustrate the median and quartiles. Outliers are determined from 1.5 times the inter-quartile range.}
    \label{fig:pita_gemma}
\end{figure}

\begin{figure}[!h]
    \centering
    \includegraphics[width=0.85\linewidth]{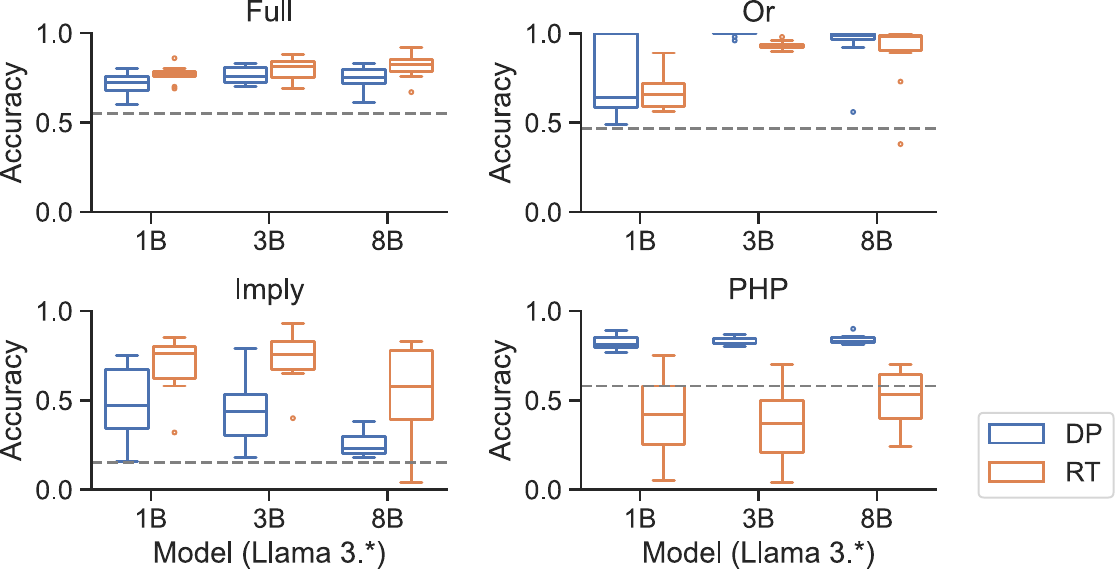}
    \caption{\textbf{Generalization accuracy in Llama 3 series on PITA splits.} The setting is the same as in Figure \ref{fig:pita_res}, but evaluated on a selection of the Llama 3 family of models \cite{llama3}. The trends are the same as observed for Qwen2.5-Coder, in which shallow and wide \textit{boule-shaped} splits favor RT models, while the reverse is true for deep and narrow \textit{baguette-shaped} splits. Boxplots are constructed from 10 runs, where each model is evaluated on 100 test samples. Box lines illustrate the median and quartiles. Outliers are determined from 1.5 times the inter-quartile range.}
    \label{fig:pita_llama}
\end{figure}

\begin{figure}
    \centering
    \subfigure[Qwen2.5-Coder]{
        \includegraphics[width=0.62\linewidth]{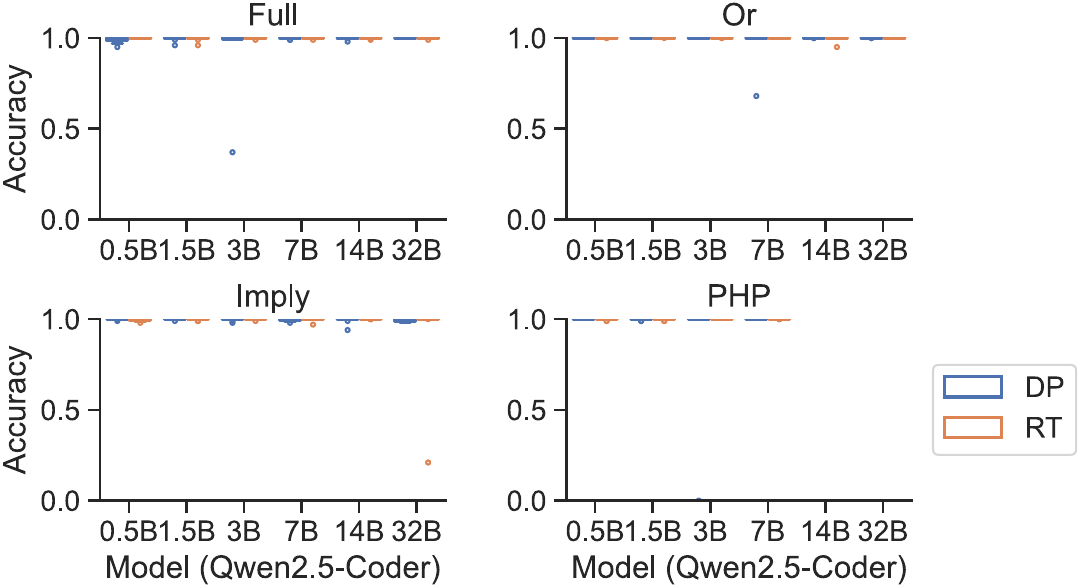}
    }
    \subfigure[Gemma 3]{
        \includegraphics[width=0.62\linewidth]{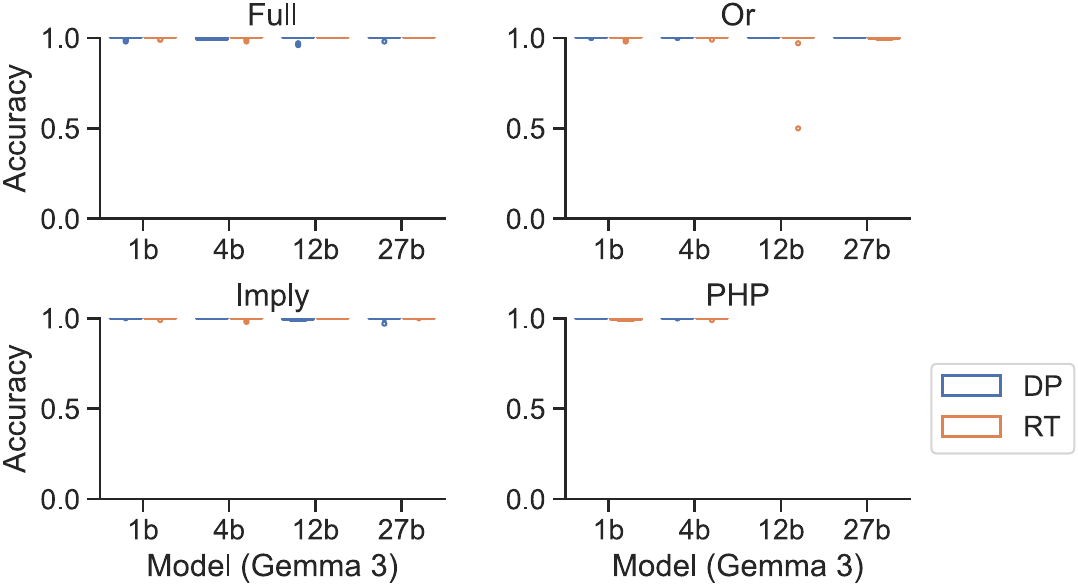}
    }
    \subfigure[Llama 3]{
        \includegraphics[width=0.62\linewidth]{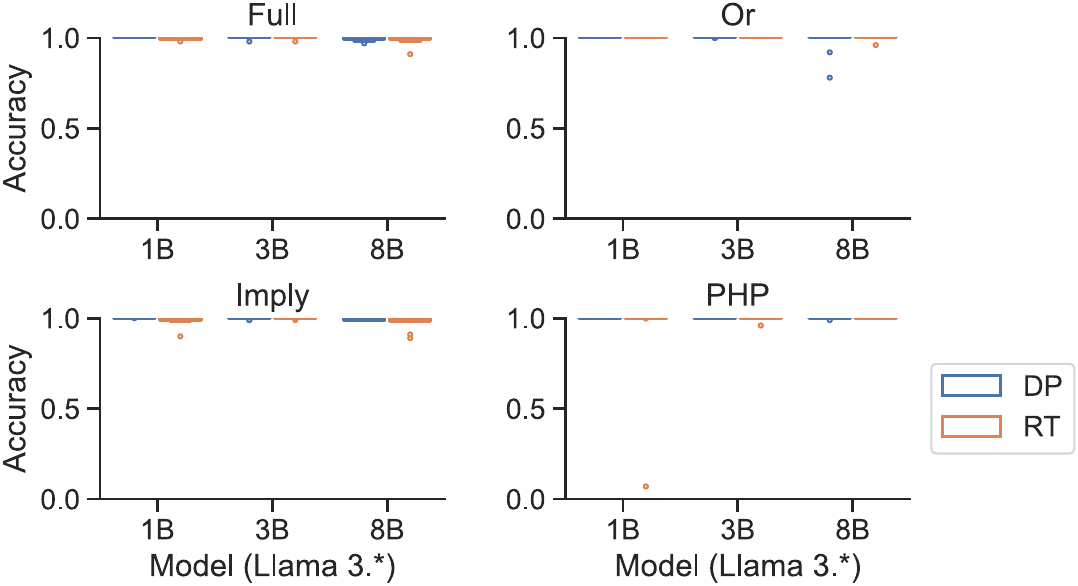}
    }
    \caption{\textbf{In-distribution accuracy on PITA splits.} The setting is the same as Figure \ref{fig:pita_res}, but models are evaluated on held-out examples whose proof lengths fall within the training range. Across model families and splits, both RT and DP models attain high in-distribution accuracy, indicating that the main challenge in Figure \ref{fig:pita_res} is length generalization rather than fitting the training-length regime.}
    \label{fig:pita_in_dist}
\end{figure}

\begin{figure}
    \centering
    \includegraphics[width=0.8\linewidth]{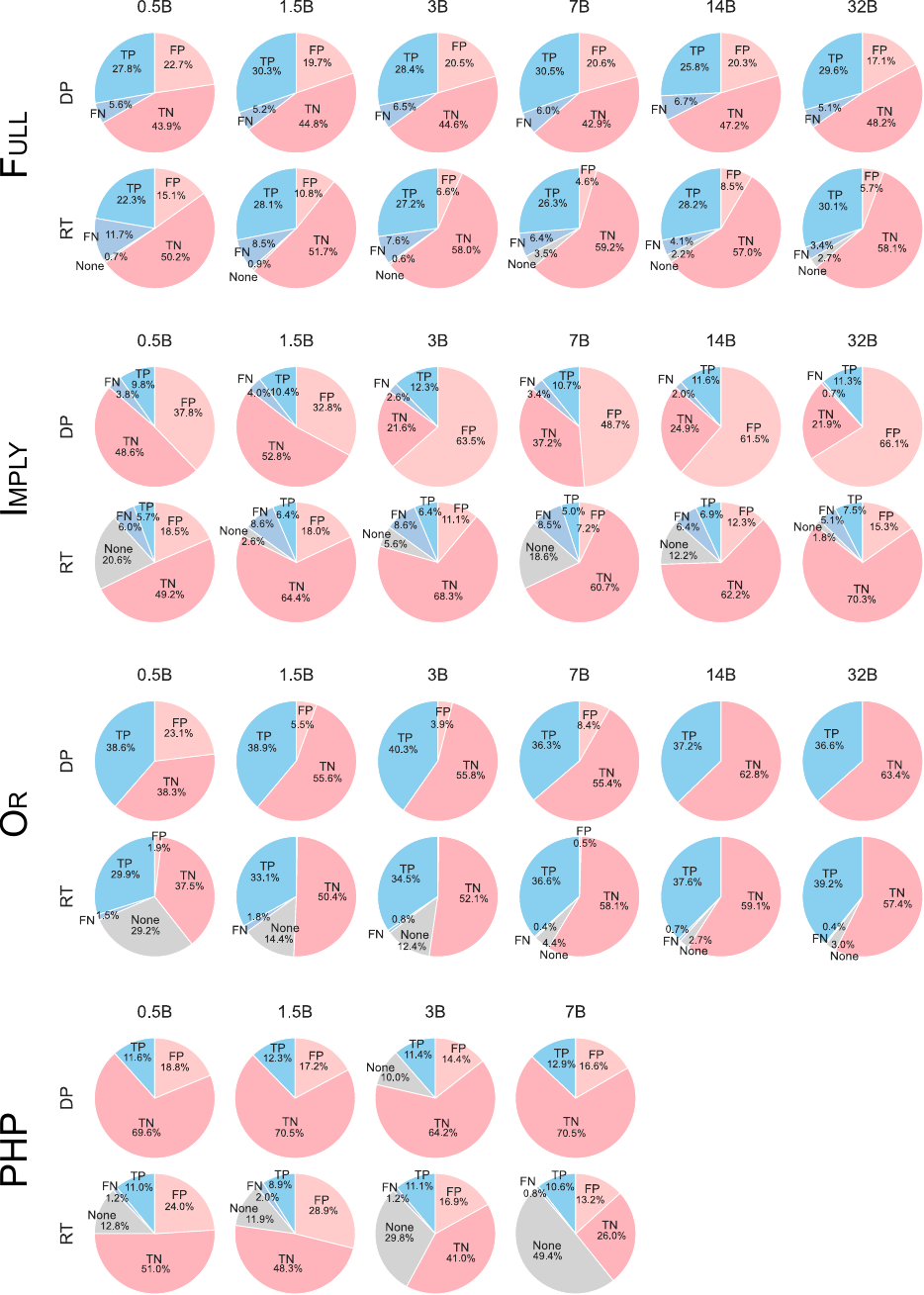}
    \caption{\textbf{Accuracy breakdown for Qwen2.5-Coder models.} For each combination of split, model size, and input format, we plot the generalization accuracy breakdown between true positives (TP), false positives (FP), true negatives (TN), false negatives (FN), and examples for which the model does not emit a label within its evaluation generation budget (None). Evaluations use randomly sampled examples. Varying TP/TN proportions across perfect classifiers reflect sampling differences (Appendix \ref{app:details}). The proportion of negative examples witnessed during generalization is typically higher than the proportion witnessed during training. Hence, models frequently commit false positives, but using a reasoning trace appears to lessen this bias. However, RT models tend not to terminate, and a large portion of their error appears to stem from this failure to terminate.}
    \label{fig:app_breakdown}
\end{figure}

\begin{figure}
    \centering
    \includegraphics[width=0.7\linewidth]{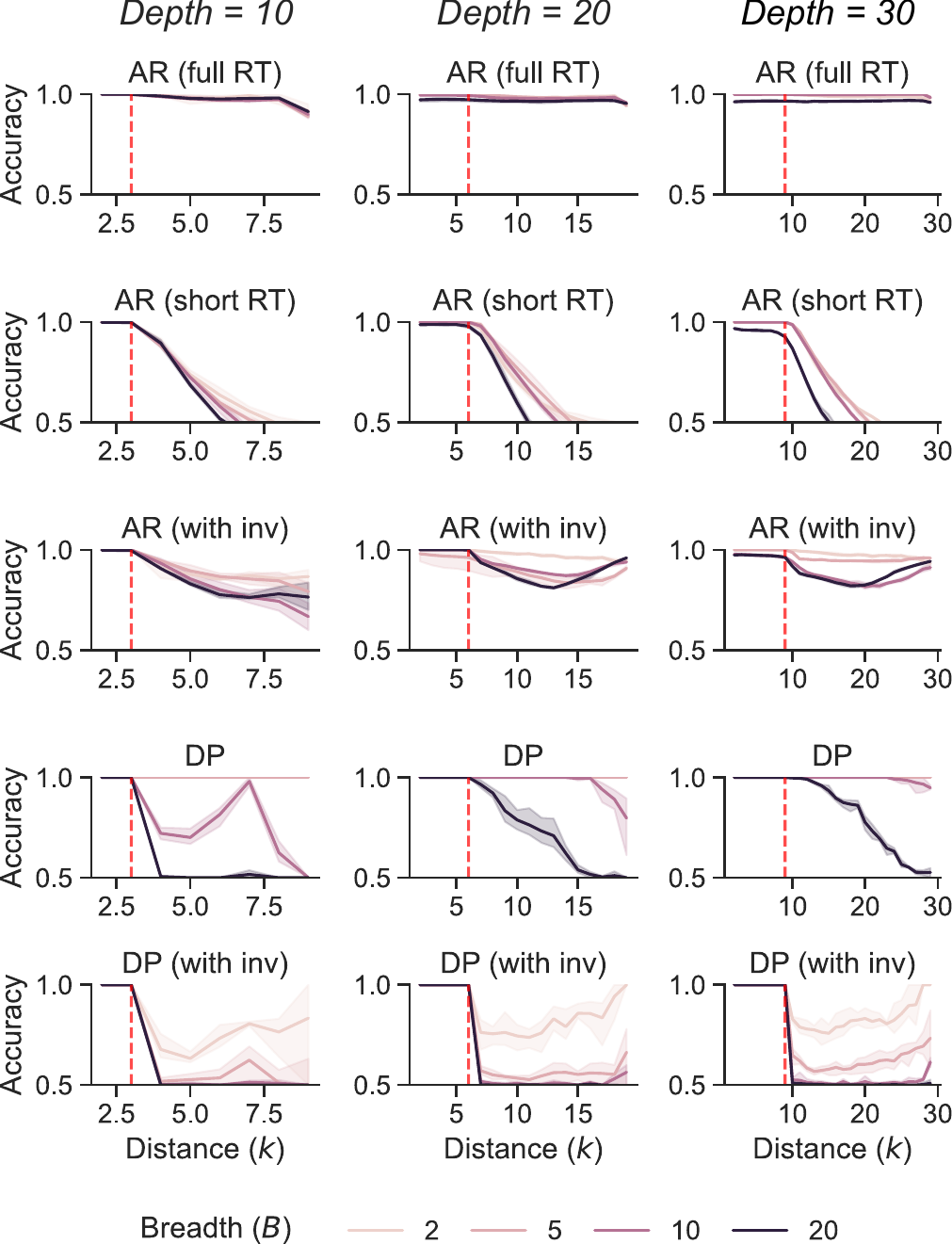}
    \caption{\textbf{Generalization performance on TI task for different input formats.} We compare generalization accuracy across several possible input formats in RT and DP models. \textit{Full RT} corresponds to the full enumeration over branch symbols used in the main text. \textit{Short RT} terminates the enumeration if the first input symbol is recovered. \textit{With inv} includes inverted negative examples. For example, if $s_1 \rightarrow s_2$ is a positive relation, then $s_2 \rightarrow s_1$ would be an inverted negative example. The results presented in the main text are with full RT and no inverted negative examples (first and fourth rows). Alternative input formats tend to do worse. In particular, truncating the reasoning trace early substantially worsens performance. Shaded error regions correspond to 95 percent intervals estimated from 3 seeds.}
    \label{fig:app_variation}
\end{figure}

\clearpage
\section{Example output} \label{app:ex_out}
We illustrate typical reasoning trace output from our LLMs finetuned on PITA across the four dataset splits. In each case, we use a Qwen2.5-Coder-7B model finetuned for 2000 steps, precisely the setting in our figures above. Whitespace has been added to the XML output to improve readability.

\subsection{\textsc{Full} split: example outputs}

Example of a successful output on the \textsc{Full} split. The model output is token-for-token identical to the ground truth reasoning trace. Note, the maximum depth observed during training is 4, whereas this example has a maximum depth of 7. Nonetheless, the output is exactly identical to the ground truth. This kind of exact match is typical across successful examples in all four splits. Even in unsuccessful cases, the output typically presents a superficially plausible and syntactic facsimile of correct Lean code, suggesting a high degree of precision in learning the basic syntax of the task.

\begin{minipage}[]{\linewidth}
    \scriptsize
    \begin{minipage}[t]{0.45\linewidth}
        {\footnotesize\itshape Ground truth}
        \begin{Verbatim}[commandchars=\\\{\}]
\PYG{n+nt}{\PYGZlt{}state} \PYG{n+na}{id=}\PYG{l+s}{\PYGZdq{}0\PYGZdq{}}\PYG{n+nt}{\PYGZgt{}}
  \PYG{n+nt}{\PYGZlt{}if}\PYG{n+nt}{\PYGZgt{}}p1 p2 p3 : Prop\PYG{n+nt}{\PYGZlt{}/if\PYGZgt{}}
  \PYG{n+nt}{\PYGZlt{}then}\PYG{n+nt}{\PYGZgt{}}⊢ (p2 ∨ False) ∨ (True ∧ False) ∧ False\PYG{n+nt}{\PYGZlt{}/then\PYGZgt{}}
\PYG{n+nt}{\PYGZlt{}/state\PYGZgt{}}
\PYG{n+nt}{\PYGZlt{}tactic}\PYG{n+nt}{\PYGZgt{}}apply Or.inl\PYG{n+nt}{\PYGZlt{}/tactic\PYGZgt{}}
\PYG{n+nt}{\PYGZlt{}state} \PYG{n+na}{id=}\PYG{l+s}{\PYGZdq{}1\PYGZdq{}}\PYG{n+nt}{\PYGZgt{}}
  \PYG{n+nt}{\PYGZlt{}if}\PYG{n+nt}{\PYGZgt{}}case h\PYG{n+nt}{\PYGZlt{}/if\PYGZgt{}}
  \PYG{n+nt}{\PYGZlt{}if}\PYG{n+nt}{\PYGZgt{}}p1 p2 p3 : Prop\PYG{n+nt}{\PYGZlt{}/if\PYGZgt{}}
  \PYG{n+nt}{\PYGZlt{}then}\PYG{n+nt}{\PYGZgt{}}⊢ p2 ∨ False\PYG{n+nt}{\PYGZlt{}/then\PYGZgt{}}
\PYG{n+nt}{\PYGZlt{}/state\PYGZgt{}}
\PYG{n+nt}{\PYGZlt{}tactic}\PYG{n+nt}{\PYGZgt{}}apply Or.inl\PYG{n+nt}{\PYGZlt{}/tactic\PYGZgt{}}
\PYG{n+nt}{\PYGZlt{}state} \PYG{n+na}{id=}\PYG{l+s}{\PYGZdq{}2\PYGZdq{}}\PYG{n+nt}{\PYGZgt{}}
  \PYG{n+nt}{\PYGZlt{}if}\PYG{n+nt}{\PYGZgt{}}case h.h\PYG{n+nt}{\PYGZlt{}/if\PYGZgt{}}
  \PYG{n+nt}{\PYGZlt{}if}\PYG{n+nt}{\PYGZgt{}}p1 p2 p3 : Prop\PYG{n+nt}{\PYGZlt{}/if\PYGZgt{}}
  \PYG{n+nt}{\PYGZlt{}then}\PYG{n+nt}{\PYGZgt{}}⊢ p2\PYG{n+nt}{\PYGZlt{}/then\PYGZgt{}}
\PYG{n+nt}{\PYGZlt{}/state\PYGZgt{}}
\PYG{n+nt}{\PYGZlt{}backtrack} \PYG{n+na}{to=}\PYG{l+s}{\PYGZdq{}1\PYGZdq{}} \PYG{n+nt}{/\PYGZgt{}}
\PYG{n+nt}{\PYGZlt{}state} \PYG{n+na}{id=}\PYG{l+s}{\PYGZdq{}1\PYGZdq{}}\PYG{n+nt}{\PYGZgt{}}
  \PYG{n+nt}{\PYGZlt{}if}\PYG{n+nt}{\PYGZgt{}}case h\PYG{n+nt}{\PYGZlt{}/if\PYGZgt{}}
  \PYG{n+nt}{\PYGZlt{}if}\PYG{n+nt}{\PYGZgt{}}p1 p2 p3 : Prop\PYG{n+nt}{\PYGZlt{}/if\PYGZgt{}}
  \PYG{n+nt}{\PYGZlt{}then}\PYG{n+nt}{\PYGZgt{}}⊢ p2 ∨ False\PYG{n+nt}{\PYGZlt{}/then\PYGZgt{}}
\PYG{n+nt}{\PYGZlt{}/state\PYGZgt{}}
\PYG{n+nt}{\PYGZlt{}tactic}\PYG{n+nt}{\PYGZgt{}}apply Or.inr\PYG{n+nt}{\PYGZlt{}/tactic\PYGZgt{}}
\PYG{n+nt}{\PYGZlt{}state} \PYG{n+na}{id=}\PYG{l+s}{\PYGZdq{}3\PYGZdq{}}\PYG{n+nt}{\PYGZgt{}}
  \PYG{n+nt}{\PYGZlt{}if}\PYG{n+nt}{\PYGZgt{}}case h.h\PYG{n+nt}{\PYGZlt{}/if\PYGZgt{}}
  \PYG{n+nt}{\PYGZlt{}if}\PYG{n+nt}{\PYGZgt{}}p1 p2 p3 : Prop\PYG{n+nt}{\PYGZlt{}/if\PYGZgt{}}
  \PYG{n+nt}{\PYGZlt{}then}\PYG{n+nt}{\PYGZgt{}}⊢ False\PYG{n+nt}{\PYGZlt{}/then\PYGZgt{}}
\PYG{n+nt}{\PYGZlt{}/state\PYGZgt{}}

\PYG{c+cm}{\PYGZlt{}!\PYGZhy{}\PYGZhy{} truncated \PYGZhy{}\PYGZhy{}\PYGZgt{}}

\PYG{n+nt}{\PYGZlt{}state} \PYG{n+na}{id=}\PYG{l+s}{\PYGZdq{}7\PYGZdq{}}\PYG{n+nt}{\PYGZgt{}}
  \PYG{n+nt}{\PYGZlt{}if}\PYG{n+nt}{\PYGZgt{}}case h.left.right\PYG{n+nt}{\PYGZlt{}/if\PYGZgt{}}
  \PYG{n+nt}{\PYGZlt{}if}\PYG{n+nt}{\PYGZgt{}}p1 p2 p3 : Prop\PYG{n+nt}{\PYGZlt{}/if\PYGZgt{}}
  \PYG{n+nt}{\PYGZlt{}then}\PYG{n+nt}{\PYGZgt{}}⊢ False\PYG{n+nt}{\PYGZlt{}/then\PYGZgt{}}
\PYG{n+nt}{\PYGZlt{}/state\PYGZgt{}}
\PYG{n+nt}{\PYGZlt{}backtrack} \PYG{n+na}{to=}\PYG{l+s}{\PYGZdq{}0\PYGZdq{}} \PYG{n+nt}{/\PYGZgt{}}
\PYG{n+nt}{\PYGZlt{}state} \PYG{n+na}{id=}\PYG{l+s}{\PYGZdq{}0\PYGZdq{}}\PYG{n+nt}{\PYGZgt{}}
  \PYG{n+nt}{\PYGZlt{}if}\PYG{n+nt}{\PYGZgt{}}p1 p2 p3 : Prop\PYG{n+nt}{\PYGZlt{}/if\PYGZgt{}}
  \PYG{n+nt}{\PYGZlt{}then}\PYG{n+nt}{\PYGZgt{}}⊢ (p2 ∨ False) ∨ (True ∧ False) ∧ False\PYG{n+nt}{\PYGZlt{}/then\PYGZgt{}}
\PYG{n+nt}{\PYGZlt{}/state\PYGZgt{}}
\PYG{n+nt}{\PYGZlt{}failure} \PYG{n+nt}{/\PYGZgt{}}
\end{Verbatim}

    \end{minipage}
    \begin{minipage}[t]{0.5\linewidth}
        {\footnotesize\itshape Model output}
        \begin{Verbatim}[commandchars=\\\{\}]
\PYG{n+nt}{\PYGZlt{}state} \PYG{n+na}{id=}\PYG{l+s}{\PYGZdq{}0\PYGZdq{}}\PYG{n+nt}{\PYGZgt{}}
  \PYG{n+nt}{\PYGZlt{}if}\PYG{n+nt}{\PYGZgt{}}p1 p2 p3 : Prop\PYG{n+nt}{\PYGZlt{}/if\PYGZgt{}}
  \PYG{n+nt}{\PYGZlt{}then}\PYG{n+nt}{\PYGZgt{}}⊢ (p2 ∨ False) ∨ (True ∧ False) ∧ False\PYG{n+nt}{\PYGZlt{}/then\PYGZgt{}}
\PYG{n+nt}{\PYGZlt{}/state\PYGZgt{}}
\PYG{n+nt}{\PYGZlt{}tactic}\PYG{n+nt}{\PYGZgt{}}apply Or.inl\PYG{n+nt}{\PYGZlt{}/tactic\PYGZgt{}}
\PYG{n+nt}{\PYGZlt{}state} \PYG{n+na}{id=}\PYG{l+s}{\PYGZdq{}1\PYGZdq{}}\PYG{n+nt}{\PYGZgt{}}
  \PYG{n+nt}{\PYGZlt{}if}\PYG{n+nt}{\PYGZgt{}}case h\PYG{n+nt}{\PYGZlt{}/if\PYGZgt{}}
  \PYG{n+nt}{\PYGZlt{}if}\PYG{n+nt}{\PYGZgt{}}p1 p2 p3 : Prop\PYG{n+nt}{\PYGZlt{}/if\PYGZgt{}}
  \PYG{n+nt}{\PYGZlt{}then}\PYG{n+nt}{\PYGZgt{}}⊢ p2 ∨ False\PYG{n+nt}{\PYGZlt{}/then\PYGZgt{}}
\PYG{n+nt}{\PYGZlt{}/state\PYGZgt{}}
\PYG{n+nt}{\PYGZlt{}tactic}\PYG{n+nt}{\PYGZgt{}}apply Or.inl\PYG{n+nt}{\PYGZlt{}/tactic\PYGZgt{}}
\PYG{n+nt}{\PYGZlt{}state} \PYG{n+na}{id=}\PYG{l+s}{\PYGZdq{}2\PYGZdq{}}\PYG{n+nt}{\PYGZgt{}}
  \PYG{n+nt}{\PYGZlt{}if}\PYG{n+nt}{\PYGZgt{}}case h.h\PYG{n+nt}{\PYGZlt{}/if\PYGZgt{}}
  \PYG{n+nt}{\PYGZlt{}if}\PYG{n+nt}{\PYGZgt{}}p1 p2 p3 : Prop\PYG{n+nt}{\PYGZlt{}/if\PYGZgt{}}
  \PYG{n+nt}{\PYGZlt{}then}\PYG{n+nt}{\PYGZgt{}}⊢ p2\PYG{n+nt}{\PYGZlt{}/then\PYGZgt{}}
\PYG{n+nt}{\PYGZlt{}/state\PYGZgt{}}
\PYG{n+nt}{\PYGZlt{}backtrack} \PYG{n+na}{to=}\PYG{l+s}{\PYGZdq{}1\PYGZdq{}} \PYG{n+nt}{/\PYGZgt{}}
\PYG{n+nt}{\PYGZlt{}state} \PYG{n+na}{id=}\PYG{l+s}{\PYGZdq{}1\PYGZdq{}}\PYG{n+nt}{\PYGZgt{}}
  \PYG{n+nt}{\PYGZlt{}if}\PYG{n+nt}{\PYGZgt{}}case h\PYG{n+nt}{\PYGZlt{}/if\PYGZgt{}}
  \PYG{n+nt}{\PYGZlt{}if}\PYG{n+nt}{\PYGZgt{}}p1 p2 p3 : Prop\PYG{n+nt}{\PYGZlt{}/if\PYGZgt{}}
  \PYG{n+nt}{\PYGZlt{}then}\PYG{n+nt}{\PYGZgt{}}⊢ p2 ∨ False\PYG{n+nt}{\PYGZlt{}/then\PYGZgt{}}
\PYG{n+nt}{\PYGZlt{}/state\PYGZgt{}}
\PYG{n+nt}{\PYGZlt{}tactic}\PYG{n+nt}{\PYGZgt{}}apply Or.inr\PYG{n+nt}{\PYGZlt{}/tactic\PYGZgt{}}
\PYG{n+nt}{\PYGZlt{}state} \PYG{n+na}{id=}\PYG{l+s}{\PYGZdq{}3\PYGZdq{}}\PYG{n+nt}{\PYGZgt{}}
  \PYG{n+nt}{\PYGZlt{}if}\PYG{n+nt}{\PYGZgt{}}case h.h\PYG{n+nt}{\PYGZlt{}/if\PYGZgt{}}
  \PYG{n+nt}{\PYGZlt{}if}\PYG{n+nt}{\PYGZgt{}}p1 p2 p3 : Prop\PYG{n+nt}{\PYGZlt{}/if\PYGZgt{}}
  \PYG{n+nt}{\PYGZlt{}then}\PYG{n+nt}{\PYGZgt{}}⊢ False\PYG{n+nt}{\PYGZlt{}/then\PYGZgt{}}
\PYG{n+nt}{\PYGZlt{}/state\PYGZgt{}}

\PYG{c+cm}{\PYGZlt{}!\PYGZhy{}\PYGZhy{} truncated \PYGZhy{}\PYGZhy{}\PYGZgt{}}

\PYG{n+nt}{\PYGZlt{}state} \PYG{n+na}{id=}\PYG{l+s}{\PYGZdq{}7\PYGZdq{}}\PYG{n+nt}{\PYGZgt{}}
  \PYG{n+nt}{\PYGZlt{}if}\PYG{n+nt}{\PYGZgt{}}case h.left.right\PYG{n+nt}{\PYGZlt{}/if\PYGZgt{}}
  \PYG{n+nt}{\PYGZlt{}if}\PYG{n+nt}{\PYGZgt{}}p1 p2 p3 : Prop\PYG{n+nt}{\PYGZlt{}/if\PYGZgt{}}
  \PYG{n+nt}{\PYGZlt{}then}\PYG{n+nt}{\PYGZgt{}}⊢ False\PYG{n+nt}{\PYGZlt{}/then\PYGZgt{}}
\PYG{n+nt}{\PYGZlt{}/state\PYGZgt{}}
\PYG{n+nt}{\PYGZlt{}backtrack} \PYG{n+na}{to=}\PYG{l+s}{\PYGZdq{}0\PYGZdq{}} \PYG{n+nt}{/\PYGZgt{}}
\PYG{n+nt}{\PYGZlt{}state} \PYG{n+na}{id=}\PYG{l+s}{\PYGZdq{}0\PYGZdq{}}\PYG{n+nt}{\PYGZgt{}}
  \PYG{n+nt}{\PYGZlt{}if}\PYG{n+nt}{\PYGZgt{}}p1 p2 p3 : Prop\PYG{n+nt}{\PYGZlt{}/if\PYGZgt{}}
  \PYG{n+nt}{\PYGZlt{}then}\PYG{n+nt}{\PYGZgt{}}⊢ (p2 ∨ False) ∨ (True ∧ False) ∧ False\PYG{n+nt}{\PYGZlt{}/then\PYGZgt{}}
\PYG{n+nt}{\PYGZlt{}/state\PYGZgt{}}
\PYG{n+nt}{\PYGZlt{}failure} \PYG{n+nt}{/\PYGZgt{}}
\end{Verbatim}

    \end{minipage}
\end{minipage}

\clearpage
Example of an unsuccessful output on the \textsc{Full} split. The output matches the ground truth exactly until the tactic following state 3. The correct tactic in this case is \texttt{cases h2}, but the model erroneously produces \texttt{apply False.elim h2}. The tactic application is obviously incorrect (\texttt{h2} does not even contain a \texttt{False} atom), but the following output is syntactically well-formed. The model terminates prematurely following this error.

\begin{minipage}[]{\linewidth}
    \scriptsize
    \begin{minipage}[t]{0.45\linewidth}
        {\footnotesize\itshape Ground truth}
        \begin{Verbatim}[commandchars=\\\{\}]
\PYG{n+nt}{\PYGZlt{}state} \PYG{n+na}{id=}\PYG{l+s}{\PYGZdq{}0\PYGZdq{}}\PYG{n+nt}{\PYGZgt{}}
  \PYG{n+nt}{\PYGZlt{}if}\PYG{n+nt}{\PYGZgt{}}p1 p2 p3 : Prop\PYG{n+nt}{\PYGZlt{}/if\PYGZgt{}}
  \PYG{n+nt}{\PYGZlt{}then}\PYG{n+nt}{\PYGZgt{}}⊢ (True ∨ p1) ∧ (p2 → p2) → p3\PYG{n+nt}{\PYGZlt{}/then\PYGZgt{}}
\PYG{n+nt}{\PYGZlt{}/state\PYGZgt{}}
\PYG{n+nt}{\PYGZlt{}tactic}\PYG{n+nt}{\PYGZgt{}}intro h1\PYG{n+nt}{\PYGZlt{}/tactic\PYGZgt{}}
\PYG{n+nt}{\PYGZlt{}state} \PYG{n+na}{id=}\PYG{l+s}{\PYGZdq{}1\PYGZdq{}}\PYG{n+nt}{\PYGZgt{}}
  \PYG{n+nt}{\PYGZlt{}if}\PYG{n+nt}{\PYGZgt{}}p1 p2 p3 : Prop\PYG{n+nt}{\PYGZlt{}/if\PYGZgt{}}
  \PYG{n+nt}{\PYGZlt{}if}\PYG{n+nt}{\PYGZgt{}}h1 : (True ∨ p1) ∧ (p2 → p2)\PYG{n+nt}{\PYGZlt{}/if\PYGZgt{}}
  \PYG{n+nt}{\PYGZlt{}then}\PYG{n+nt}{\PYGZgt{}}⊢ p3\PYG{n+nt}{\PYGZlt{}/then\PYGZgt{}}
\PYG{n+nt}{\PYGZlt{}/state\PYGZgt{}}
\PYG{n+nt}{\PYGZlt{}tactic}\PYG{n+nt}{\PYGZgt{}}have h2 := h1.left\PYG{n+nt}{\PYGZlt{}/tactic\PYGZgt{}}
\PYG{n+nt}{\PYGZlt{}state} \PYG{n+na}{id=}\PYG{l+s}{\PYGZdq{}2\PYGZdq{}}\PYG{n+nt}{\PYGZgt{}}
  \PYG{n+nt}{\PYGZlt{}if}\PYG{n+nt}{\PYGZgt{}}p1 p2 p3 : Prop\PYG{n+nt}{\PYGZlt{}/if\PYGZgt{}}
  \PYG{n+nt}{\PYGZlt{}if}\PYG{n+nt}{\PYGZgt{}}h1 : (True ∨ p1) ∧ (p2 → p2)\PYG{n+nt}{\PYGZlt{}/if\PYGZgt{}}
  \PYG{n+nt}{\PYGZlt{}if}\PYG{n+nt}{\PYGZgt{}}h2 : True ∨ p1\PYG{n+nt}{\PYGZlt{}/if\PYGZgt{}}
  \PYG{n+nt}{\PYGZlt{}then}\PYG{n+nt}{\PYGZgt{}}⊢ p3\PYG{n+nt}{\PYGZlt{}/then\PYGZgt{}}
\PYG{n+nt}{\PYGZlt{}/state\PYGZgt{}}
\PYG{n+nt}{\PYGZlt{}tactic}\PYG{n+nt}{\PYGZgt{}}have h3 := h1.right\PYG{n+nt}{\PYGZlt{}/tactic\PYGZgt{}}
\PYG{n+nt}{\PYGZlt{}state} \PYG{n+na}{id=}\PYG{l+s}{\PYGZdq{}3\PYGZdq{}}\PYG{n+nt}{\PYGZgt{}}
  \PYG{n+nt}{\PYGZlt{}if}\PYG{n+nt}{\PYGZgt{}}p1 p2 p3 : Prop\PYG{n+nt}{\PYGZlt{}/if\PYGZgt{}}
  \PYG{n+nt}{\PYGZlt{}if}\PYG{n+nt}{\PYGZgt{}}h1 : (True ∨ p1) ∧ (p2 → p2)\PYG{n+nt}{\PYGZlt{}/if\PYGZgt{}}
  \PYG{n+nt}{\PYGZlt{}if}\PYG{n+nt}{\PYGZgt{}}h2 : True ∨ p1\PYG{n+nt}{\PYGZlt{}/if\PYGZgt{}}
  \PYG{n+nt}{\PYGZlt{}if}\PYG{n+nt}{\PYGZgt{}}h3 : p2 → p2\PYG{n+nt}{\PYGZlt{}/if\PYGZgt{}}
  \PYG{n+nt}{\PYGZlt{}then}\PYG{n+nt}{\PYGZgt{}}⊢ p3\PYG{n+nt}{\PYGZlt{}/then\PYGZgt{}}
\PYG{n+nt}{\PYGZlt{}/state\PYGZgt{}}
\PYG{n+nt}{\PYGZlt{}tactic}\PYG{n+nt}{\PYGZgt{}}cases h2\PYG{n+nt}{\PYGZlt{}/tactic\PYGZgt{}}
\PYG{n+nt}{\PYGZlt{}state} \PYG{n+na}{id=}\PYG{l+s}{\PYGZdq{}4\PYGZdq{}}\PYG{n+nt}{\PYGZgt{}}
  \PYG{n+nt}{\PYGZlt{}if}\PYG{n+nt}{\PYGZgt{}}case inl\PYG{n+nt}{\PYGZlt{}/if\PYGZgt{}}
  \PYG{n+nt}{\PYGZlt{}if}\PYG{n+nt}{\PYGZgt{}}p1 p2 p3 : Prop\PYG{n+nt}{\PYGZlt{}/if\PYGZgt{}}
  \PYG{n+nt}{\PYGZlt{}if}\PYG{n+nt}{\PYGZgt{}}h1 : (True ∨ p1) ∧ (p2 → p2)\PYG{n+nt}{\PYGZlt{}/if\PYGZgt{}}
  \PYG{n+nt}{\PYGZlt{}if}\PYG{n+nt}{\PYGZgt{}}h3 : p2 → p2\PYG{n+nt}{\PYGZlt{}/if\PYGZgt{}}
  \PYG{n+nt}{\PYGZlt{}if}\PYG{n+nt}{\PYGZgt{}}h✝ : True\PYG{n+nt}{\PYGZlt{}/if\PYGZgt{}}
  \PYG{n+nt}{\PYGZlt{}then}\PYG{n+nt}{\PYGZgt{}}⊢ p3\PYG{n+nt}{\PYGZlt{}/then\PYGZgt{}}
\PYG{n+nt}{\PYGZlt{}/state\PYGZgt{}}

\PYG{c+cm}{\PYGZlt{}!\PYGZhy{}\PYGZhy{} truncated \PYGZhy{}\PYGZhy{}\PYGZgt{}}

\PYG{n+nt}{\PYGZlt{}failure} \PYG{n+nt}{/\PYGZgt{}}
\end{Verbatim}

    \end{minipage}
    \begin{minipage}[t]{0.5\linewidth}
        {\footnotesize\itshape Model output}
        \begin{Verbatim}[commandchars=\\\{\}]
\PYG{n+nt}{\PYGZlt{}state} \PYG{n+na}{id=}\PYG{l+s}{\PYGZdq{}0\PYGZdq{}}\PYG{n+nt}{\PYGZgt{}}
  \PYG{n+nt}{\PYGZlt{}if}\PYG{n+nt}{\PYGZgt{}}p1 p2 p3 : Prop\PYG{n+nt}{\PYGZlt{}/if\PYGZgt{}}
  \PYG{n+nt}{\PYGZlt{}then}\PYG{n+nt}{\PYGZgt{}}⊢ (True ∨ p1) ∧ (p2 → p2) → p3\PYG{n+nt}{\PYGZlt{}/then\PYGZgt{}}
\PYG{n+nt}{\PYGZlt{}/state\PYGZgt{}}
\PYG{n+nt}{\PYGZlt{}tactic}\PYG{n+nt}{\PYGZgt{}}intro h1\PYG{n+nt}{\PYGZlt{}/tactic\PYGZgt{}}
\PYG{n+nt}{\PYGZlt{}state} \PYG{n+na}{id=}\PYG{l+s}{\PYGZdq{}1\PYGZdq{}}\PYG{n+nt}{\PYGZgt{}}
  \PYG{n+nt}{\PYGZlt{}if}\PYG{n+nt}{\PYGZgt{}}p1 p2 p3 : Prop\PYG{n+nt}{\PYGZlt{}/if\PYGZgt{}}
  \PYG{n+nt}{\PYGZlt{}if}\PYG{n+nt}{\PYGZgt{}}h1 : (True ∨ p1) ∧ (p2 → p2)\PYG{n+nt}{\PYGZlt{}/if\PYGZgt{}}
  \PYG{n+nt}{\PYGZlt{}then}\PYG{n+nt}{\PYGZgt{}}⊢ p3\PYG{n+nt}{\PYGZlt{}/then\PYGZgt{}}
\PYG{n+nt}{\PYGZlt{}/state\PYGZgt{}}
\PYG{n+nt}{\PYGZlt{}tactic}\PYG{n+nt}{\PYGZgt{}}have h2 := h1.left\PYG{n+nt}{\PYGZlt{}/tactic\PYGZgt{}}
\PYG{n+nt}{\PYGZlt{}state} \PYG{n+na}{id=}\PYG{l+s}{\PYGZdq{}2\PYGZdq{}}\PYG{n+nt}{\PYGZgt{}}
  \PYG{n+nt}{\PYGZlt{}if}\PYG{n+nt}{\PYGZgt{}}p1 p2 p3 : Prop\PYG{n+nt}{\PYGZlt{}/if\PYGZgt{}}
  \PYG{n+nt}{\PYGZlt{}if}\PYG{n+nt}{\PYGZgt{}}h1 : (True ∨ p1) ∧ (p2 → p2)\PYG{n+nt}{\PYGZlt{}/if\PYGZgt{}}
  \PYG{n+nt}{\PYGZlt{}if}\PYG{n+nt}{\PYGZgt{}}h2 : True ∨ p1\PYG{n+nt}{\PYGZlt{}/if\PYGZgt{}}
  \PYG{n+nt}{\PYGZlt{}then}\PYG{n+nt}{\PYGZgt{}}⊢ p3\PYG{n+nt}{\PYGZlt{}/then\PYGZgt{}}
\PYG{n+nt}{\PYGZlt{}/state\PYGZgt{}}
\PYG{n+nt}{\PYGZlt{}tactic}\PYG{n+nt}{\PYGZgt{}}have h3 := h1.right\PYG{n+nt}{\PYGZlt{}/tactic\PYGZgt{}}
\PYG{n+nt}{\PYGZlt{}state} \PYG{n+na}{id=}\PYG{l+s}{\PYGZdq{}3\PYGZdq{}}\PYG{n+nt}{\PYGZgt{}}
  \PYG{n+nt}{\PYGZlt{}if}\PYG{n+nt}{\PYGZgt{}}p1 p2 p3 : Prop\PYG{n+nt}{\PYGZlt{}/if\PYGZgt{}}
  \PYG{n+nt}{\PYGZlt{}if}\PYG{n+nt}{\PYGZgt{}}h1 : (True ∨ p1) ∧ (p2 → p2)\PYG{n+nt}{\PYGZlt{}/if\PYGZgt{}}
  \PYG{n+nt}{\PYGZlt{}if}\PYG{n+nt}{\PYGZgt{}}h2 : True ∨ p1\PYG{n+nt}{\PYGZlt{}/if\PYGZgt{}}
  \PYG{n+nt}{\PYGZlt{}if}\PYG{n+nt}{\PYGZgt{}}h3 : p2 → p2\PYG{n+nt}{\PYGZlt{}/if\PYGZgt{}}
  \PYG{n+nt}{\PYGZlt{}then}\PYG{n+nt}{\PYGZgt{}}⊢ p3\PYG{n+nt}{\PYGZlt{}/then\PYGZgt{}}
\PYG{n+nt}{\PYGZlt{}/state\PYGZgt{}}
\PYG{n+nt}{\PYGZlt{}tactic}\PYG{n+nt}{\PYGZgt{}}apply False.elim h2\PYG{n+nt}{\PYGZlt{}/tactic\PYGZgt{}}   \PYG{c+cm}{\PYGZlt{}!\PYGZhy{}\PYGZhy{} wrong tactic chosen \PYGZhy{}\PYGZhy{}\PYGZgt{}}
\PYG{n+nt}{\PYGZlt{}state} \PYG{n+na}{id=}\PYG{l+s}{\PYGZdq{}4\PYGZdq{}}\PYG{n+nt}{\PYGZgt{}}
  \PYG{n+nt}{\PYGZlt{}complete} \PYG{n+nt}{/\PYGZgt{}}
\PYG{n+nt}{\PYGZlt{}/state\PYGZgt{}}
\PYG{n+nt}{\PYGZlt{}success} \PYG{n+nt}{/\PYGZgt{}}
\end{Verbatim}

    \end{minipage}
\end{minipage}

\clearpage
\subsection{\textsc{Imply} split: example outputs}

Example of a successful output on the \textsc{Imply} split. As before, the model output is a token-for-token exact match to the ground truth.

\begin{minipage}[]{\linewidth}
    \scriptsize
    \begin{minipage}[t]{0.45\linewidth}
        {\footnotesize\itshape Ground truth}
        \begin{Verbatim}[commandchars=\\\{\}]
\PYG{n+nt}{\PYGZlt{}state} \PYG{n+na}{id=}\PYG{l+s}{\PYGZdq{}0\PYGZdq{}}\PYG{n+nt}{\PYGZgt{}}
  \PYG{n+nt}{\PYGZlt{}if}\PYG{n+nt}{\PYGZgt{}}p1 p2 p3 : Prop\PYG{n+nt}{\PYGZlt{}/if\PYGZgt{}}
  \PYG{n+nt}{\PYGZlt{}then}\PYG{n+nt}{\PYGZgt{}}⊢ p1 → ((False → True) → p3) → 
          p1 → p1 → p1\PYG{n+nt}{\PYGZlt{}/then\PYGZgt{}}
\PYG{n+nt}{\PYGZlt{}/state\PYGZgt{}}
\PYG{n+nt}{\PYGZlt{}tactic}\PYG{n+nt}{\PYGZgt{}}intro h1\PYG{n+nt}{\PYGZlt{}/tactic\PYGZgt{}}
\PYG{n+nt}{\PYGZlt{}state} \PYG{n+na}{id=}\PYG{l+s}{\PYGZdq{}1\PYGZdq{}}\PYG{n+nt}{\PYGZgt{}}
  \PYG{n+nt}{\PYGZlt{}if}\PYG{n+nt}{\PYGZgt{}}p1 p2 p3 : Prop\PYG{n+nt}{\PYGZlt{}/if\PYGZgt{}}
  \PYG{n+nt}{\PYGZlt{}if}\PYG{n+nt}{\PYGZgt{}}h1 : p1\PYG{n+nt}{\PYGZlt{}/if\PYGZgt{}}
  \PYG{n+nt}{\PYGZlt{}then}\PYG{n+nt}{\PYGZgt{}}⊢ ((False → True) → p3) → p1 → p1 → p1\PYG{n+nt}{\PYGZlt{}/then\PYGZgt{}}
\PYG{n+nt}{\PYGZlt{}/state\PYGZgt{}}
\PYG{n+nt}{\PYGZlt{}tactic}\PYG{n+nt}{\PYGZgt{}}intro h2\PYG{n+nt}{\PYGZlt{}/tactic\PYGZgt{}}
\PYG{n+nt}{\PYGZlt{}state} \PYG{n+na}{id=}\PYG{l+s}{\PYGZdq{}2\PYGZdq{}}\PYG{n+nt}{\PYGZgt{}}
  \PYG{n+nt}{\PYGZlt{}if}\PYG{n+nt}{\PYGZgt{}}p1 p2 p3 : Prop\PYG{n+nt}{\PYGZlt{}/if\PYGZgt{}}
  \PYG{n+nt}{\PYGZlt{}if}\PYG{n+nt}{\PYGZgt{}}h1 : p1\PYG{n+nt}{\PYGZlt{}/if\PYGZgt{}}
  \PYG{n+nt}{\PYGZlt{}if}\PYG{n+nt}{\PYGZgt{}}h2 : (False → True) → p3\PYG{n+nt}{\PYGZlt{}/if\PYGZgt{}}
  \PYG{n+nt}{\PYGZlt{}then}\PYG{n+nt}{\PYGZgt{}}⊢ p1 → p1 → p1\PYG{n+nt}{\PYGZlt{}/then\PYGZgt{}}
\PYG{n+nt}{\PYGZlt{}/state\PYGZgt{}}
\PYG{n+nt}{\PYGZlt{}tactic}\PYG{n+nt}{\PYGZgt{}}intro h3\PYG{n+nt}{\PYGZlt{}/tactic\PYGZgt{}}
\PYG{n+nt}{\PYGZlt{}state} \PYG{n+na}{id=}\PYG{l+s}{\PYGZdq{}3\PYGZdq{}}\PYG{n+nt}{\PYGZgt{}}
  \PYG{n+nt}{\PYGZlt{}if}\PYG{n+nt}{\PYGZgt{}}p1 p2 p3 : Prop\PYG{n+nt}{\PYGZlt{}/if\PYGZgt{}}
  \PYG{n+nt}{\PYGZlt{}if}\PYG{n+nt}{\PYGZgt{}}h1 : p1\PYG{n+nt}{\PYGZlt{}/if\PYGZgt{}}
  \PYG{n+nt}{\PYGZlt{}if}\PYG{n+nt}{\PYGZgt{}}h2 : (False → True) → p3\PYG{n+nt}{\PYGZlt{}/if\PYGZgt{}}
  \PYG{n+nt}{\PYGZlt{}if}\PYG{n+nt}{\PYGZgt{}}h3 : p1\PYG{n+nt}{\PYGZlt{}/if\PYGZgt{}}
  \PYG{n+nt}{\PYGZlt{}then}\PYG{n+nt}{\PYGZgt{}}⊢ p1 → p1\PYG{n+nt}{\PYGZlt{}/then\PYGZgt{}}
\PYG{n+nt}{\PYGZlt{}/state\PYGZgt{}}
\PYG{n+nt}{\PYGZlt{}tactic}\PYG{n+nt}{\PYGZgt{}}intro h4\PYG{n+nt}{\PYGZlt{}/tactic\PYGZgt{}}
\PYG{n+nt}{\PYGZlt{}state} \PYG{n+na}{id=}\PYG{l+s}{\PYGZdq{}4\PYGZdq{}}\PYG{n+nt}{\PYGZgt{}}
  \PYG{n+nt}{\PYGZlt{}if}\PYG{n+nt}{\PYGZgt{}}p1 p2 p3 : Prop\PYG{n+nt}{\PYGZlt{}/if\PYGZgt{}}
  \PYG{n+nt}{\PYGZlt{}if}\PYG{n+nt}{\PYGZgt{}}h1 : p1\PYG{n+nt}{\PYGZlt{}/if\PYGZgt{}}
  \PYG{n+nt}{\PYGZlt{}if}\PYG{n+nt}{\PYGZgt{}}h2 : (False → True) → p3\PYG{n+nt}{\PYGZlt{}/if\PYGZgt{}}
  \PYG{n+nt}{\PYGZlt{}if}\PYG{n+nt}{\PYGZgt{}}h3 h4 : p1\PYG{n+nt}{\PYGZlt{}/if\PYGZgt{}}
  \PYG{n+nt}{\PYGZlt{}then}\PYG{n+nt}{\PYGZgt{}}⊢ p1\PYG{n+nt}{\PYGZlt{}/then\PYGZgt{}}
\PYG{n+nt}{\PYGZlt{}/state\PYGZgt{}}
\PYG{n+nt}{\PYGZlt{}tactic}\PYG{n+nt}{\PYGZgt{}}exact h1\PYG{n+nt}{\PYGZlt{}/tactic\PYGZgt{}}
\PYG{n+nt}{\PYGZlt{}state} \PYG{n+na}{id=}\PYG{l+s}{\PYGZdq{}5\PYGZdq{}}\PYG{n+nt}{\PYGZgt{}}
  \PYG{n+nt}{\PYGZlt{}complete} \PYG{n+nt}{/\PYGZgt{}}
\PYG{n+nt}{\PYGZlt{}/state\PYGZgt{}}
\PYG{n+nt}{\PYGZlt{}success} \PYG{n+nt}{/\PYGZgt{}}
\end{Verbatim}

    \end{minipage}
    \begin{minipage}[t]{0.5\linewidth}
        {\footnotesize\itshape Model output}
        \begin{Verbatim}[commandchars=\\\{\}]
\PYG{n+nt}{\PYGZlt{}state} \PYG{n+na}{id=}\PYG{l+s}{\PYGZdq{}0\PYGZdq{}}\PYG{n+nt}{\PYGZgt{}}
  \PYG{n+nt}{\PYGZlt{}if}\PYG{n+nt}{\PYGZgt{}}p1 p2 p3 : Prop\PYG{n+nt}{\PYGZlt{}/if\PYGZgt{}}
  \PYG{n+nt}{\PYGZlt{}then}\PYG{n+nt}{\PYGZgt{}}⊢ p1 → ((False → True) → p3) → 
          p1 → p1 → p1\PYG{n+nt}{\PYGZlt{}/then\PYGZgt{}}
\PYG{n+nt}{\PYGZlt{}/state\PYGZgt{}}
\PYG{n+nt}{\PYGZlt{}tactic}\PYG{n+nt}{\PYGZgt{}}intro h1\PYG{n+nt}{\PYGZlt{}/tactic\PYGZgt{}}
\PYG{n+nt}{\PYGZlt{}state} \PYG{n+na}{id=}\PYG{l+s}{\PYGZdq{}1\PYGZdq{}}\PYG{n+nt}{\PYGZgt{}}
  \PYG{n+nt}{\PYGZlt{}if}\PYG{n+nt}{\PYGZgt{}}p1 p2 p3 : Prop\PYG{n+nt}{\PYGZlt{}/if\PYGZgt{}}
  \PYG{n+nt}{\PYGZlt{}if}\PYG{n+nt}{\PYGZgt{}}h1 : p1\PYG{n+nt}{\PYGZlt{}/if\PYGZgt{}}
  \PYG{n+nt}{\PYGZlt{}then}\PYG{n+nt}{\PYGZgt{}}⊢ ((False → True) → p3) → p1 → p1 → p1\PYG{n+nt}{\PYGZlt{}/then\PYGZgt{}}
\PYG{n+nt}{\PYGZlt{}/state\PYGZgt{}}
\PYG{n+nt}{\PYGZlt{}tactic}\PYG{n+nt}{\PYGZgt{}}intro h2\PYG{n+nt}{\PYGZlt{}/tactic\PYGZgt{}}
\PYG{n+nt}{\PYGZlt{}state} \PYG{n+na}{id=}\PYG{l+s}{\PYGZdq{}2\PYGZdq{}}\PYG{n+nt}{\PYGZgt{}}
  \PYG{n+nt}{\PYGZlt{}if}\PYG{n+nt}{\PYGZgt{}}p1 p2 p3 : Prop\PYG{n+nt}{\PYGZlt{}/if\PYGZgt{}}
  \PYG{n+nt}{\PYGZlt{}if}\PYG{n+nt}{\PYGZgt{}}h1 : p1\PYG{n+nt}{\PYGZlt{}/if\PYGZgt{}}
  \PYG{n+nt}{\PYGZlt{}if}\PYG{n+nt}{\PYGZgt{}}h2 : (False → True) → p3\PYG{n+nt}{\PYGZlt{}/if\PYGZgt{}}
  \PYG{n+nt}{\PYGZlt{}then}\PYG{n+nt}{\PYGZgt{}}⊢ p1 → p1 → p1\PYG{n+nt}{\PYGZlt{}/then\PYGZgt{}}
\PYG{n+nt}{\PYGZlt{}/state\PYGZgt{}}
\PYG{n+nt}{\PYGZlt{}tactic}\PYG{n+nt}{\PYGZgt{}}intro h3\PYG{n+nt}{\PYGZlt{}/tactic\PYGZgt{}}
\PYG{n+nt}{\PYGZlt{}state} \PYG{n+na}{id=}\PYG{l+s}{\PYGZdq{}3\PYGZdq{}}\PYG{n+nt}{\PYGZgt{}}
  \PYG{n+nt}{\PYGZlt{}if}\PYG{n+nt}{\PYGZgt{}}p1 p2 p3 : Prop\PYG{n+nt}{\PYGZlt{}/if\PYGZgt{}}
  \PYG{n+nt}{\PYGZlt{}if}\PYG{n+nt}{\PYGZgt{}}h1 : p1\PYG{n+nt}{\PYGZlt{}/if\PYGZgt{}}
  \PYG{n+nt}{\PYGZlt{}if}\PYG{n+nt}{\PYGZgt{}}h2 : (False → True) → p3\PYG{n+nt}{\PYGZlt{}/if\PYGZgt{}}
  \PYG{n+nt}{\PYGZlt{}if}\PYG{n+nt}{\PYGZgt{}}h3 : p1\PYG{n+nt}{\PYGZlt{}/if\PYGZgt{}}
  \PYG{n+nt}{\PYGZlt{}then}\PYG{n+nt}{\PYGZgt{}}⊢ p1 → p1\PYG{n+nt}{\PYGZlt{}/then\PYGZgt{}}
\PYG{n+nt}{\PYGZlt{}/state\PYGZgt{}}
\PYG{n+nt}{\PYGZlt{}tactic}\PYG{n+nt}{\PYGZgt{}}intro h4\PYG{n+nt}{\PYGZlt{}/tactic\PYGZgt{}}
\PYG{n+nt}{\PYGZlt{}state} \PYG{n+na}{id=}\PYG{l+s}{\PYGZdq{}4\PYGZdq{}}\PYG{n+nt}{\PYGZgt{}}
  \PYG{n+nt}{\PYGZlt{}if}\PYG{n+nt}{\PYGZgt{}}p1 p2 p3 : Prop\PYG{n+nt}{\PYGZlt{}/if\PYGZgt{}}
  \PYG{n+nt}{\PYGZlt{}if}\PYG{n+nt}{\PYGZgt{}}h1 : p1\PYG{n+nt}{\PYGZlt{}/if\PYGZgt{}}
  \PYG{n+nt}{\PYGZlt{}if}\PYG{n+nt}{\PYGZgt{}}h2 : (False → True) → p3\PYG{n+nt}{\PYGZlt{}/if\PYGZgt{}}
  \PYG{n+nt}{\PYGZlt{}if}\PYG{n+nt}{\PYGZgt{}}h3 h4 : p1\PYG{n+nt}{\PYGZlt{}/if\PYGZgt{}}
  \PYG{n+nt}{\PYGZlt{}then}\PYG{n+nt}{\PYGZgt{}}⊢ p1\PYG{n+nt}{\PYGZlt{}/then\PYGZgt{}}
\PYG{n+nt}{\PYGZlt{}/state\PYGZgt{}}
\PYG{n+nt}{\PYGZlt{}tactic}\PYG{n+nt}{\PYGZgt{}}exact h1\PYG{n+nt}{\PYGZlt{}/tactic\PYGZgt{}}
\PYG{n+nt}{\PYGZlt{}state} \PYG{n+na}{id=}\PYG{l+s}{\PYGZdq{}5\PYGZdq{}}\PYG{n+nt}{\PYGZgt{}}
  \PYG{n+nt}{\PYGZlt{}complete} \PYG{n+nt}{/\PYGZgt{}}
\PYG{n+nt}{\PYGZlt{}/state\PYGZgt{}}
\PYG{n+nt}{\PYGZlt{}success} \PYG{n+nt}{/\PYGZgt{}}
\end{Verbatim}

    \end{minipage}
\end{minipage}

\clearpage
Example of an unsuccessful output on the \textsc{Imply} split. The model chooses an erroneous tactic, which leads it to produce syntactically well-formed but meaningless states that ultimately terminate prematurely with the wrong classification.

\begin{minipage}[]{\linewidth}
    \scriptsize
    \begin{minipage}[t]{0.45\linewidth}
        {\footnotesize\itshape Ground truth}
        \begin{Verbatim}[commandchars=\\\{\}]
\PYG{n+nt}{\PYGZlt{}state} \PYG{n+na}{id=}\PYG{l+s}{\PYGZdq{}0\PYGZdq{}}\PYG{n+nt}{\PYGZgt{}}
  \PYG{n+nt}{\PYGZlt{}if}\PYG{n+nt}{\PYGZgt{}}p1 p2 p3 : Prop\PYG{n+nt}{\PYGZlt{}/if\PYGZgt{}}
  \PYG{n+nt}{\PYGZlt{}then}\PYG{n+nt}{\PYGZgt{}}⊢ (True → False) → (False → True) → 
          (p3 → False) → False\PYG{n+nt}{\PYGZlt{}/then\PYGZgt{}}
\PYG{n+nt}{\PYGZlt{}/state\PYGZgt{}}
\PYG{n+nt}{\PYGZlt{}tactic}\PYG{n+nt}{\PYGZgt{}}intro h1\PYG{n+nt}{\PYGZlt{}/tactic\PYGZgt{}}
\PYG{n+nt}{\PYGZlt{}state} \PYG{n+na}{id=}\PYG{l+s}{\PYGZdq{}1\PYGZdq{}}\PYG{n+nt}{\PYGZgt{}}
  \PYG{n+nt}{\PYGZlt{}if}\PYG{n+nt}{\PYGZgt{}}p1 p2 p3 : Prop\PYG{n+nt}{\PYGZlt{}/if\PYGZgt{}}
  \PYG{n+nt}{\PYGZlt{}if}\PYG{n+nt}{\PYGZgt{}}h1 : True → False\PYG{n+nt}{\PYGZlt{}/if\PYGZgt{}}
  \PYG{n+nt}{\PYGZlt{}then}\PYG{n+nt}{\PYGZgt{}}⊢ (False → True) → (p3 → False) → False\PYG{n+nt}{\PYGZlt{}/then\PYGZgt{}}
\PYG{n+nt}{\PYGZlt{}/state\PYGZgt{}}
\PYG{n+nt}{\PYGZlt{}tactic}\PYG{n+nt}{\PYGZgt{}}intro h2\PYG{n+nt}{\PYGZlt{}/tactic\PYGZgt{}}
\PYG{n+nt}{\PYGZlt{}state} \PYG{n+na}{id=}\PYG{l+s}{\PYGZdq{}2\PYGZdq{}}\PYG{n+nt}{\PYGZgt{}}
  \PYG{n+nt}{\PYGZlt{}if}\PYG{n+nt}{\PYGZgt{}}p1 p2 p3 : Prop\PYG{n+nt}{\PYGZlt{}/if\PYGZgt{}}
  \PYG{n+nt}{\PYGZlt{}if}\PYG{n+nt}{\PYGZgt{}}h1 : True → False\PYG{n+nt}{\PYGZlt{}/if\PYGZgt{}}
  \PYG{n+nt}{\PYGZlt{}if}\PYG{n+nt}{\PYGZgt{}}h2 : False → True\PYG{n+nt}{\PYGZlt{}/if\PYGZgt{}}
  \PYG{n+nt}{\PYGZlt{}then}\PYG{n+nt}{\PYGZgt{}}⊢ (p3 → False) → False\PYG{n+nt}{\PYGZlt{}/then\PYGZgt{}}
\PYG{n+nt}{\PYGZlt{}/state\PYGZgt{}}
\PYG{n+nt}{\PYGZlt{}tactic}\PYG{n+nt}{\PYGZgt{}}intro h3\PYG{n+nt}{\PYGZlt{}/tactic\PYGZgt{}}
\PYG{n+nt}{\PYGZlt{}state} \PYG{n+na}{id=}\PYG{l+s}{\PYGZdq{}3\PYGZdq{}}\PYG{n+nt}{\PYGZgt{}}
  \PYG{n+nt}{\PYGZlt{}if}\PYG{n+nt}{\PYGZgt{}}p1 p2 p3 : Prop\PYG{n+nt}{\PYGZlt{}/if\PYGZgt{}}
  \PYG{n+nt}{\PYGZlt{}if}\PYG{n+nt}{\PYGZgt{}}h1 : True → False\PYG{n+nt}{\PYGZlt{}/if\PYGZgt{}}
  \PYG{n+nt}{\PYGZlt{}if}\PYG{n+nt}{\PYGZgt{}}h2 : False → True\PYG{n+nt}{\PYGZlt{}/if\PYGZgt{}}
  \PYG{n+nt}{\PYGZlt{}if}\PYG{n+nt}{\PYGZgt{}}h3 : p3 → False\PYG{n+nt}{\PYGZlt{}/if\PYGZgt{}}
  \PYG{n+nt}{\PYGZlt{}then}\PYG{n+nt}{\PYGZgt{}}⊢ False\PYG{n+nt}{\PYGZlt{}/then\PYGZgt{}}
\PYG{n+nt}{\PYGZlt{}/state\PYGZgt{}}
\PYG{n+nt}{\PYGZlt{}tactic}\PYG{n+nt}{\PYGZgt{}}have h4 : True := by\PYG{n+nt}{\PYGZlt{}/tactic\PYGZgt{}}
\PYG{n+nt}{\PYGZlt{}state} \PYG{n+na}{id=}\PYG{l+s}{\PYGZdq{}4\PYGZdq{}}\PYG{n+nt}{\PYGZgt{}}
  \PYG{n+nt}{\PYGZlt{}if}\PYG{n+nt}{\PYGZgt{}}p1 p2 p3 : Prop\PYG{n+nt}{\PYGZlt{}/if\PYGZgt{}}
  \PYG{n+nt}{\PYGZlt{}if}\PYG{n+nt}{\PYGZgt{}}h1 : True → False\PYG{n+nt}{\PYGZlt{}/if\PYGZgt{}}
  \PYG{n+nt}{\PYGZlt{}if}\PYG{n+nt}{\PYGZgt{}}h2 : False → True\PYG{n+nt}{\PYGZlt{}/if\PYGZgt{}}
  \PYG{n+nt}{\PYGZlt{}if}\PYG{n+nt}{\PYGZgt{}}h3 : p3 → False\PYG{n+nt}{\PYGZlt{}/if\PYGZgt{}}
  \PYG{n+nt}{\PYGZlt{}then}\PYG{n+nt}{\PYGZgt{}}⊢ True\PYG{n+nt}{\PYGZlt{}/then\PYGZgt{}}
\PYG{n+nt}{\PYGZlt{}/state\PYGZgt{}}

\PYG{c+cm}{\PYGZlt{}!\PYGZhy{}\PYGZhy{} truncated \PYGZhy{}\PYGZhy{}\PYGZgt{}}

\PYG{n+nt}{\PYGZlt{}success} \PYG{n+nt}{/\PYGZgt{}}
\end{Verbatim}

    \end{minipage}
    \begin{minipage}[t]{0.5\linewidth}
        {\footnotesize\itshape Model output}
        \begin{Verbatim}[commandchars=\\\{\}]
\PYG{n+nt}{\PYGZlt{}state} \PYG{n+na}{id=}\PYG{l+s}{\PYGZdq{}0\PYGZdq{}}\PYG{n+nt}{\PYGZgt{}}
  \PYG{n+nt}{\PYGZlt{}if}\PYG{n+nt}{\PYGZgt{}}p1 p2 p3 : Prop\PYG{n+nt}{\PYGZlt{}/if\PYGZgt{}}
  \PYG{n+nt}{\PYGZlt{}then}\PYG{n+nt}{\PYGZgt{}}⊢ (True → False) → (False → True) → 
          (p3 → False) → False\PYG{n+nt}{\PYGZlt{}/then\PYGZgt{}}
\PYG{n+nt}{\PYGZlt{}/state\PYGZgt{}}
\PYG{n+nt}{\PYGZlt{}tactic}\PYG{n+nt}{\PYGZgt{}}intro h1\PYG{n+nt}{\PYGZlt{}/tactic\PYGZgt{}}
\PYG{n+nt}{\PYGZlt{}state} \PYG{n+na}{id=}\PYG{l+s}{\PYGZdq{}1\PYGZdq{}}\PYG{n+nt}{\PYGZgt{}}
  \PYG{n+nt}{\PYGZlt{}if}\PYG{n+nt}{\PYGZgt{}}p1 p2 p3 : Prop\PYG{n+nt}{\PYGZlt{}/if\PYGZgt{}}
  \PYG{n+nt}{\PYGZlt{}if}\PYG{n+nt}{\PYGZgt{}}h1 : True → False\PYG{n+nt}{\PYGZlt{}/if\PYGZgt{}}
  \PYG{n+nt}{\PYGZlt{}then}\PYG{n+nt}{\PYGZgt{}}⊢ (False → True) → (p3 → False) → False\PYG{n+nt}{\PYGZlt{}/then\PYGZgt{}}
\PYG{n+nt}{\PYGZlt{}/state\PYGZgt{}}
\PYG{n+nt}{\PYGZlt{}tactic}\PYG{n+nt}{\PYGZgt{}}intro h2\PYG{n+nt}{\PYGZlt{}/tactic\PYGZgt{}}
\PYG{n+nt}{\PYGZlt{}state} \PYG{n+na}{id=}\PYG{l+s}{\PYGZdq{}2\PYGZdq{}}\PYG{n+nt}{\PYGZgt{}}
  \PYG{n+nt}{\PYGZlt{}if}\PYG{n+nt}{\PYGZgt{}}p1 p2 p3 : Prop\PYG{n+nt}{\PYGZlt{}/if\PYGZgt{}}
  \PYG{n+nt}{\PYGZlt{}if}\PYG{n+nt}{\PYGZgt{}}h1 : True → False\PYG{n+nt}{\PYGZlt{}/if\PYGZgt{}}
  \PYG{n+nt}{\PYGZlt{}if}\PYG{n+nt}{\PYGZgt{}}h2 : False → True\PYG{n+nt}{\PYGZlt{}/if\PYGZgt{}}
  \PYG{n+nt}{\PYGZlt{}then}\PYG{n+nt}{\PYGZgt{}}⊢ (p3 → False) → False\PYG{n+nt}{\PYGZlt{}/then\PYGZgt{}}
\PYG{n+nt}{\PYGZlt{}/state\PYGZgt{}}
\PYG{n+nt}{\PYGZlt{}tactic}\PYG{n+nt}{\PYGZgt{}}intro h3\PYG{n+nt}{\PYGZlt{}/tactic\PYGZgt{}}
\PYG{n+nt}{\PYGZlt{}state} \PYG{n+na}{id=}\PYG{l+s}{\PYGZdq{}3\PYGZdq{}}\PYG{n+nt}{\PYGZgt{}}
  \PYG{n+nt}{\PYGZlt{}if}\PYG{n+nt}{\PYGZgt{}}p1 p2 p3 : Prop\PYG{n+nt}{\PYGZlt{}/if\PYGZgt{}}
  \PYG{n+nt}{\PYGZlt{}if}\PYG{n+nt}{\PYGZgt{}}h1 : True → False\PYG{n+nt}{\PYGZlt{}/if\PYGZgt{}}
  \PYG{n+nt}{\PYGZlt{}if}\PYG{n+nt}{\PYGZgt{}}h2 : False → True\PYG{n+nt}{\PYGZlt{}/if\PYGZgt{}}
  \PYG{n+nt}{\PYGZlt{}if}\PYG{n+nt}{\PYGZgt{}}h3 : p3 → False\PYG{n+nt}{\PYGZlt{}/if\PYGZgt{}}
  \PYG{n+nt}{\PYGZlt{}then}\PYG{n+nt}{\PYGZgt{}}⊢ False\PYG{n+nt}{\PYGZlt{}/then\PYGZgt{}}
\PYG{n+nt}{\PYGZlt{}/state\PYGZgt{}}
\PYG{n+nt}{\PYGZlt{}tactic}\PYG{n+nt}{\PYGZgt{}}have h4 : p3 := by\PYG{n+nt}{\PYGZlt{}/tactic\PYGZgt{}}    \PYG{c+cm}{\PYGZlt{}!\PYGZhy{}\PYGZhy{} wrong tactic chosen \PYGZhy{}\PYGZhy{}\PYGZgt{}}
\PYG{n+nt}{\PYGZlt{}state} \PYG{n+na}{id=}\PYG{l+s}{\PYGZdq{}4\PYGZdq{}}\PYG{n+nt}{\PYGZgt{}}
  \PYG{n+nt}{\PYGZlt{}if}\PYG{n+nt}{\PYGZgt{}}p1 p2 p3 : Prop\PYG{n+nt}{\PYGZlt{}/if\PYGZgt{}}
  \PYG{n+nt}{\PYGZlt{}if}\PYG{n+nt}{\PYGZgt{}}h1 : True → False\PYG{n+nt}{\PYGZlt{}/if\PYGZgt{}}
  \PYG{n+nt}{\PYGZlt{}if}\PYG{n+nt}{\PYGZgt{}}h2 : False → True\PYG{n+nt}{\PYGZlt{}/if\PYGZgt{}}
  \PYG{n+nt}{\PYGZlt{}if}\PYG{n+nt}{\PYGZgt{}}h3 : p3 → False\PYG{n+nt}{\PYGZlt{}/if\PYGZgt{}}
  \PYG{n+nt}{\PYGZlt{}then}\PYG{n+nt}{\PYGZgt{}}⊢ p3\PYG{n+nt}{\PYGZlt{}/then\PYGZgt{}}
\PYG{n+nt}{\PYGZlt{}/state\PYGZgt{}}
\PYG{n+nt}{\PYGZlt{}state} \PYG{n+na}{id=}\PYG{l+s}{\PYGZdq{}5\PYGZdq{}}\PYG{n+nt}{\PYGZgt{}}
  \PYG{n+nt}{\PYGZlt{}complete} \PYG{n+nt}{/\PYGZgt{}}
\PYG{n+nt}{\PYGZlt{}/state\PYGZgt{}}
\PYG{n+nt}{\PYGZlt{}tactic}\PYG{n+nt}{\PYGZgt{}}have h5 := h3 h4\PYG{n+nt}{\PYGZlt{}/tactic\PYGZgt{}}
\PYG{n+nt}{\PYGZlt{}state} \PYG{n+na}{id=}\PYG{l+s}{\PYGZdq{}6\PYGZdq{}}\PYG{n+nt}{\PYGZgt{}}
  \PYG{n+nt}{\PYGZlt{}complete} \PYG{n+nt}{/\PYGZgt{}}
\PYG{n+nt}{\PYGZlt{}/state\PYGZgt{}}
\PYG{n+nt}{\PYGZlt{}backtrack} \PYG{n+na}{to=}\PYG{l+s}{\PYGZdq{}3\PYGZdq{}} \PYG{n+nt}{/\PYGZgt{}}
\PYG{n+nt}{\PYGZlt{}state} \PYG{n+na}{id=}\PYG{l+s}{\PYGZdq{}3\PYGZdq{}}\PYG{n+nt}{\PYGZgt{}}
  \PYG{n+nt}{\PYGZlt{}if}\PYG{n+nt}{\PYGZgt{}}p1 p2 p3 : Prop\PYG{n+nt}{\PYGZlt{}/if\PYGZgt{}}
  \PYG{n+nt}{\PYGZlt{}if}\PYG{n+nt}{\PYGZgt{}}h1 : True → False\PYG{n+nt}{\PYGZlt{}/if\PYGZgt{}}
  \PYG{n+nt}{\PYGZlt{}if}\PYG{n+nt}{\PYGZgt{}}h2 : False → True\PYG{n+nt}{\PYGZlt{}/if\PYGZgt{}}
  \PYG{n+nt}{\PYGZlt{}if}\PYG{n+nt}{\PYGZgt{}}h3 : p3 → False\PYG{n+nt}{\PYGZlt{}/if\PYGZgt{}}
  \PYG{n+nt}{\PYGZlt{}then}\PYG{n+nt}{\PYGZgt{}}⊢ False\PYG{n+nt}{\PYGZlt{}/then\PYGZgt{}}
\PYG{n+nt}{\PYGZlt{}/state\PYGZgt{}}
\PYG{n+nt}{\PYGZlt{}failure} \PYG{n+nt}{/\PYGZgt{}}
\end{Verbatim}

    \end{minipage}
\end{minipage}

\clearpage
\subsection{\textsc{Or} split: example output}

Example of an successful output on the \textsc{Or} split. Despite being substantially longer than any training example and numbering several thousand tokens, the model attains a token-for-token perfect match to the ground truth. The output has been condensed slightly to improve readability.

\begin{minipage}[]{\linewidth}
    \scriptsize
    \begin{minipage}[t]{0.45\linewidth}
        {\footnotesize\itshape Ground truth}
        \begin{Verbatim}[commandchars=\\\{\}]
\PYG{n+nt}{\PYGZlt{}state} \PYG{n+na}{id=}\PYG{l+s}{\PYGZdq{}0\PYGZdq{}}\PYG{n+nt}{\PYGZgt{}}
  \PYG{n+nt}{\PYGZlt{}then}\PYG{n+nt}{\PYGZgt{}}p88 → (((p10 ∨ p57 ∨ ... ∨ p89) ∨ p15\PYG{n+nt}{\PYGZlt{}/then\PYGZgt{}}
\PYG{n+nt}{\PYGZlt{}/state\PYGZgt{}}
\PYG{n+nt}{\PYGZlt{}tactic}\PYG{n+nt}{\PYGZgt{}}intro h1\PYG{n+nt}{\PYGZlt{}/tactic\PYGZgt{}}
\PYG{n+nt}{\PYGZlt{}state} \PYG{n+na}{id=}\PYG{l+s}{\PYGZdq{}1\PYGZdq{}}\PYG{n+nt}{\PYGZgt{}}
  \PYG{n+nt}{\PYGZlt{}if}\PYG{n+nt}{\PYGZgt{}}h1 : p88\PYG{n+nt}{\PYGZlt{}/if\PYGZgt{}}
  \PYG{n+nt}{\PYGZlt{}then}\PYG{n+nt}{\PYGZgt{}}(((p10 ∨ p57 ∨ ... ∨ p89) ∨ p15\PYG{n+nt}{\PYGZlt{}/then\PYGZgt{}}
\PYG{n+nt}{\PYGZlt{}/state\PYGZgt{}}
\PYG{n+nt}{\PYGZlt{}tactic}\PYG{n+nt}{\PYGZgt{}}apply Or.inl\PYG{n+nt}{\PYGZlt{}/tactic\PYGZgt{}}
\PYG{n+nt}{\PYGZlt{}state} \PYG{n+na}{id=}\PYG{l+s}{\PYGZdq{}2\PYGZdq{}}\PYG{n+nt}{\PYGZgt{}}
  \PYG{n+nt}{\PYGZlt{}if}\PYG{n+nt}{\PYGZgt{}}h1 : p88\PYG{n+nt}{\PYGZlt{}/if\PYGZgt{}}
  \PYG{n+nt}{\PYGZlt{}then}\PYG{n+nt}{\PYGZgt{}}((p10 ∨ p57 ∨ ... ∨ p89\PYG{n+nt}{\PYGZlt{}/then\PYGZgt{}}
\PYG{n+nt}{\PYGZlt{}/state\PYGZgt{}}
\PYG{n+nt}{\PYGZlt{}tactic}\PYG{n+nt}{\PYGZgt{}}apply Or.inl\PYG{n+nt}{\PYGZlt{}/tactic\PYGZgt{}}
\PYG{n+nt}{\PYGZlt{}state} \PYG{n+na}{id=}\PYG{l+s}{\PYGZdq{}3\PYGZdq{}}\PYG{n+nt}{\PYGZgt{}}
  \PYG{n+nt}{\PYGZlt{}if}\PYG{n+nt}{\PYGZgt{}}h1 : p88\PYG{n+nt}{\PYGZlt{}/if\PYGZgt{}}
  \PYG{n+nt}{\PYGZlt{}then}\PYG{n+nt}{\PYGZgt{}}(p10 ∨ p57 ∨ ... ∨ p87\PYG{n+nt}{\PYGZlt{}/then\PYGZgt{}}
\PYG{n+nt}{\PYGZlt{}/state\PYGZgt{}}

\PYG{c+cm}{\PYGZlt{}!\PYGZhy{}\PYGZhy{} truncated \PYGZhy{}\PYGZhy{}\PYGZgt{}}

\PYG{n+nt}{\PYGZlt{}state} \PYG{n+na}{id=}\PYG{l+s}{\PYGZdq{}22\PYGZdq{}}\PYG{n+nt}{\PYGZgt{}}
  \PYG{n+nt}{\PYGZlt{}if}\PYG{n+nt}{\PYGZgt{}}h1 : p88\PYG{n+nt}{\PYGZlt{}/if\PYGZgt{}}
  \PYG{n+nt}{\PYGZlt{}then}\PYG{n+nt}{\PYGZgt{}}p88\PYG{n+nt}{\PYGZlt{}/then\PYGZgt{}}
\PYG{n+nt}{\PYGZlt{}/state\PYGZgt{}}
\PYG{n+nt}{\PYGZlt{}tactic}\PYG{n+nt}{\PYGZgt{}}exact h1\PYG{n+nt}{\PYGZlt{}/tactic\PYGZgt{}}
\PYG{n+nt}{\PYGZlt{}state} \PYG{n+na}{id=}\PYG{l+s}{\PYGZdq{}23\PYGZdq{}}\PYG{n+nt}{\PYGZgt{}}
  \PYG{n+nt}{\PYGZlt{}complete} \PYG{n+nt}{/\PYGZgt{}}
\PYG{n+nt}{\PYGZlt{}/state\PYGZgt{}}
\PYG{n+nt}{\PYGZlt{}success} \PYG{n+nt}{/\PYGZgt{}}
\end{Verbatim}

    \end{minipage}
    \begin{minipage}[t]{0.5\linewidth}
        {\footnotesize\itshape Model output}
        \begin{Verbatim}[commandchars=\\\{\}]
\PYG{n+nt}{\PYGZlt{}state} \PYG{n+na}{id=}\PYG{l+s}{\PYGZdq{}0\PYGZdq{}}\PYG{n+nt}{\PYGZgt{}}
  \PYG{n+nt}{\PYGZlt{}then}\PYG{n+nt}{\PYGZgt{}}p88 → (((p10 ∨ p57 ∨ ... ∨ p89) ∨ p15\PYG{n+nt}{\PYGZlt{}/then\PYGZgt{}}
\PYG{n+nt}{\PYGZlt{}/state\PYGZgt{}}
\PYG{n+nt}{\PYGZlt{}tactic}\PYG{n+nt}{\PYGZgt{}}intro h1\PYG{n+nt}{\PYGZlt{}/tactic\PYGZgt{}}
\PYG{n+nt}{\PYGZlt{}state} \PYG{n+na}{id=}\PYG{l+s}{\PYGZdq{}1\PYGZdq{}}\PYG{n+nt}{\PYGZgt{}}
  \PYG{n+nt}{\PYGZlt{}if}\PYG{n+nt}{\PYGZgt{}}h1 : p88\PYG{n+nt}{\PYGZlt{}/if\PYGZgt{}}
  \PYG{n+nt}{\PYGZlt{}then}\PYG{n+nt}{\PYGZgt{}}(((p10 ∨ p57 ∨ ... ∨ p89) ∨ p15\PYG{n+nt}{\PYGZlt{}/then\PYGZgt{}}
\PYG{n+nt}{\PYGZlt{}/state\PYGZgt{}}
\PYG{n+nt}{\PYGZlt{}tactic}\PYG{n+nt}{\PYGZgt{}}apply Or.inl\PYG{n+nt}{\PYGZlt{}/tactic\PYGZgt{}}
\PYG{n+nt}{\PYGZlt{}state} \PYG{n+na}{id=}\PYG{l+s}{\PYGZdq{}2\PYGZdq{}}\PYG{n+nt}{\PYGZgt{}}
  \PYG{n+nt}{\PYGZlt{}if}\PYG{n+nt}{\PYGZgt{}}h1 : p88\PYG{n+nt}{\PYGZlt{}/if\PYGZgt{}}
  \PYG{n+nt}{\PYGZlt{}then}\PYG{n+nt}{\PYGZgt{}}((p10 ∨ p57 ∨ ... ∨ p89\PYG{n+nt}{\PYGZlt{}/then\PYGZgt{}}
\PYG{n+nt}{\PYGZlt{}/state\PYGZgt{}}
\PYG{n+nt}{\PYGZlt{}tactic}\PYG{n+nt}{\PYGZgt{}}apply Or.inl\PYG{n+nt}{\PYGZlt{}/tactic\PYGZgt{}}
\PYG{n+nt}{\PYGZlt{}state} \PYG{n+na}{id=}\PYG{l+s}{\PYGZdq{}3\PYGZdq{}}\PYG{n+nt}{\PYGZgt{}}
  \PYG{n+nt}{\PYGZlt{}if}\PYG{n+nt}{\PYGZgt{}}h1 : p88\PYG{n+nt}{\PYGZlt{}/if\PYGZgt{}}
  \PYG{n+nt}{\PYGZlt{}then}\PYG{n+nt}{\PYGZgt{}}(p10 ∨ p57 ∨ ... ∨ p87\PYG{n+nt}{\PYGZlt{}/then\PYGZgt{}}
\PYG{n+nt}{\PYGZlt{}/state\PYGZgt{}}

\PYG{c+cm}{\PYGZlt{}!\PYGZhy{}\PYGZhy{} truncated \PYGZhy{}\PYGZhy{}\PYGZgt{}}

\PYG{n+nt}{\PYGZlt{}state} \PYG{n+na}{id=}\PYG{l+s}{\PYGZdq{}22\PYGZdq{}}\PYG{n+nt}{\PYGZgt{}}
  \PYG{n+nt}{\PYGZlt{}if}\PYG{n+nt}{\PYGZgt{}}h1 : p88\PYG{n+nt}{\PYGZlt{}/if\PYGZgt{}}
  \PYG{n+nt}{\PYGZlt{}then}\PYG{n+nt}{\PYGZgt{}}p88\PYG{n+nt}{\PYGZlt{}/then\PYGZgt{}}
\PYG{n+nt}{\PYGZlt{}/state\PYGZgt{}}
\PYG{n+nt}{\PYGZlt{}tactic}\PYG{n+nt}{\PYGZgt{}}exact h1\PYG{n+nt}{\PYGZlt{}/tactic\PYGZgt{}}
\PYG{n+nt}{\PYGZlt{}state} \PYG{n+na}{id=}\PYG{l+s}{\PYGZdq{}23\PYGZdq{}}\PYG{n+nt}{\PYGZgt{}}
  \PYG{n+nt}{\PYGZlt{}complete} \PYG{n+nt}{/\PYGZgt{}}
\PYG{n+nt}{\PYGZlt{}/state\PYGZgt{}}
\PYG{n+nt}{\PYGZlt{}success} \PYG{n+nt}{/\PYGZgt{}}
\end{Verbatim}

    \end{minipage}
\end{minipage}

\clearpage
\subsection{\textsc{PHP} split: example output}

Example of an unsuccessful, non-terminating output on the \textsc{PHP} split. The model matches the ground truth for several thousand tokens before erroneously labeling a state with id 100, rather than 59. This appears to lead the model down an alternative computation path that does not terminate within the allotted max context length of 32 thousand tokens, despite the ground truth terminating well within this range. The output has been condensed slightly to improve readability.

\begin{minipage}[]{\linewidth}
    \scriptsize
    \begin{minipage}[t]{0.45\linewidth}
        {\footnotesize\itshape Ground truth}
        \begin{Verbatim}[commandchars=\\\{\}]
\PYG{n+nt}{\PYGZlt{}state} \PYG{n+na}{id=}\PYG{l+s}{\PYGZdq{}0\PYGZdq{}}\PYG{n+nt}{\PYGZgt{}}
  \PYG{n+nt}{\PYGZlt{}then}\PYG{n+nt}{\PYGZgt{}}p02 ∧ p11 ∧ (p20 ∨ p22) ∧ p32 → 
  p00 ∧ p10 ∨ p01 ∧ p11 ... ∨ p31 ∧ p21\PYG{n+nt}{\PYGZlt{}/then\PYGZgt{}}
\PYG{n+nt}{\PYGZlt{}/state\PYGZgt{}}
\PYG{n+nt}{\PYGZlt{}tactic}\PYG{n+nt}{\PYGZgt{}}intro h1\PYG{n+nt}{\PYGZlt{}/tactic\PYGZgt{}}
\PYG{n+nt}{\PYGZlt{}state} \PYG{n+na}{id=}\PYG{l+s}{\PYGZdq{}1\PYGZdq{}}\PYG{n+nt}{\PYGZgt{}}
  \PYG{n+nt}{\PYGZlt{}if}\PYG{n+nt}{\PYGZgt{}}h1 : p02 ∧ p11 ∧ (p20 ∨ p22) ∧ p32\PYG{n+nt}{\PYGZlt{}/if\PYGZgt{}}
  \PYG{n+nt}{\PYGZlt{}then}\PYG{n+nt}{\PYGZgt{}}p00 ∧ p10 ∨ p01 ∧ p11 ... p31 ∧ p21\PYG{n+nt}{\PYGZlt{}/then\PYGZgt{}}
\PYG{n+nt}{\PYGZlt{}/state\PYGZgt{}}
\PYG{n+nt}{\PYGZlt{}tactic}\PYG{n+nt}{\PYGZgt{}}have h2 := h1.left\PYG{n+nt}{\PYGZlt{}/tactic\PYGZgt{}}
\PYG{n+nt}{\PYGZlt{}state} \PYG{n+na}{id=}\PYG{l+s}{\PYGZdq{}2\PYGZdq{}}\PYG{n+nt}{\PYGZgt{}}
  \PYG{n+nt}{\PYGZlt{}if}\PYG{n+nt}{\PYGZgt{}}h1 : p02 ∧ p11 ∧ (p20 ∨ p22) ∧ p32\PYG{n+nt}{\PYGZlt{}/if\PYGZgt{}}
  \PYG{n+nt}{\PYGZlt{}if}\PYG{n+nt}{\PYGZgt{}}h2 : p02\PYG{n+nt}{\PYGZlt{}/if\PYGZgt{}}
  \PYG{n+nt}{\PYGZlt{}then}\PYG{n+nt}{\PYGZgt{}}p00 ∧ p10 ∨ p01 ∧ p11 ... p31 ∧ p21\PYG{n+nt}{\PYGZlt{}/then\PYGZgt{}}
\PYG{n+nt}{\PYGZlt{}/state\PYGZgt{}}

\PYG{c+cm}{\PYGZlt{}!\PYGZhy{}\PYGZhy{} truncated \PYGZhy{}\PYGZhy{}\PYGZgt{}}

\PYG{n+nt}{\PYGZlt{}state} \PYG{n+na}{id=}\PYG{l+s}{\PYGZdq{}59\PYGZdq{}}\PYG{n+nt}{\PYGZgt{}}
  \PYG{n+nt}{\PYGZlt{}if}\PYG{n+nt}{\PYGZgt{}}h1 : p02 ∧ p11 ∧ (p20 ∨ p22) ∧ p32\PYG{n+nt}{\PYGZlt{}/if\PYGZgt{}}
  \PYG{n+nt}{\PYGZlt{}if}\PYG{n+nt}{\PYGZgt{}}h2 : p02\PYG{n+nt}{\PYGZlt{}/if\PYGZgt{}}
  \PYG{n+nt}{\PYGZlt{}if}\PYG{n+nt}{\PYGZgt{}}h3 : p11 ∧ (p20 ∨ p22) ∧ p32\PYG{n+nt}{\PYGZlt{}/if\PYGZgt{}}
  \PYG{n+nt}{\PYGZlt{}if}\PYG{n+nt}{\PYGZgt{}}h4 : p11\PYG{n+nt}{\PYGZlt{}/if\PYGZgt{}}
  \PYG{n+nt}{\PYGZlt{}if}\PYG{n+nt}{\PYGZgt{}}h5 : (p20 ∨ p22) ∧ p32\PYG{n+nt}{\PYGZlt{}/if\PYGZgt{}}
  \PYG{n+nt}{\PYGZlt{}if}\PYG{n+nt}{\PYGZgt{}}h7 : p32\PYG{n+nt}{\PYGZlt{}/if\PYGZgt{}}
  \PYG{n+nt}{\PYGZlt{}if}\PYG{n+nt}{\PYGZgt{}}h8 : p20\PYG{n+nt}{\PYGZlt{}/if\PYGZgt{}}
  \PYG{n+nt}{\PYGZlt{}then}\PYG{n+nt}{\PYGZgt{}}p00 ∧ p10 ∨ p01 ∧ p11 ... p31 ∧ p21\PYG{n+nt}{\PYGZlt{}/then\PYGZgt{}}
\PYG{n+nt}{\PYGZlt{}/state\PYGZgt{}}

\PYG{c+cm}{\PYGZlt{}!\PYGZhy{}\PYGZhy{} truncated \PYGZhy{}\PYGZhy{}\PYGZgt{}}

\PYG{n+nt}{\PYGZlt{}failure} \PYG{n+nt}{/\PYGZgt{}}
\end{Verbatim}

    \end{minipage}
    \begin{minipage}[t]{0.5\linewidth}
        {\footnotesize\itshape Model output}
        \begin{Verbatim}[commandchars=\\\{\}]
\PYG{n+nt}{\PYGZlt{}state} \PYG{n+na}{id=}\PYG{l+s}{\PYGZdq{}0\PYGZdq{}}\PYG{n+nt}{\PYGZgt{}}
  \PYG{n+nt}{\PYGZlt{}then}\PYG{n+nt}{\PYGZgt{}}p02 ∧ p11 ∧ (p20 ∨ p22) ∧ p32 → 
  p00 ∧ p10 ∨ p01 ∧ p11 ... ∨ p31 ∧ p21\PYG{n+nt}{\PYGZlt{}/then\PYGZgt{}}
\PYG{n+nt}{\PYGZlt{}/state\PYGZgt{}}
\PYG{n+nt}{\PYGZlt{}tactic}\PYG{n+nt}{\PYGZgt{}}intro h1\PYG{n+nt}{\PYGZlt{}/tactic\PYGZgt{}}
\PYG{n+nt}{\PYGZlt{}state} \PYG{n+na}{id=}\PYG{l+s}{\PYGZdq{}1\PYGZdq{}}\PYG{n+nt}{\PYGZgt{}}
  \PYG{n+nt}{\PYGZlt{}if}\PYG{n+nt}{\PYGZgt{}}h1 : p02 ∧ p11 ∧ (p20 ∨ p22) ∧ p32\PYG{n+nt}{\PYGZlt{}/if\PYGZgt{}}
  \PYG{n+nt}{\PYGZlt{}then}\PYG{n+nt}{\PYGZgt{}}p00 ∧ p10 ∨ p01 ∧ p11 ... p31 ∧ p21\PYG{n+nt}{\PYGZlt{}/then\PYGZgt{}}
\PYG{n+nt}{\PYGZlt{}/state\PYGZgt{}}
\PYG{n+nt}{\PYGZlt{}tactic}\PYG{n+nt}{\PYGZgt{}}have h2 := h1.left\PYG{n+nt}{\PYGZlt{}/tactic\PYGZgt{}}
\PYG{n+nt}{\PYGZlt{}state} \PYG{n+na}{id=}\PYG{l+s}{\PYGZdq{}2\PYGZdq{}}\PYG{n+nt}{\PYGZgt{}}
  \PYG{n+nt}{\PYGZlt{}if}\PYG{n+nt}{\PYGZgt{}}h1 : p02 ∧ p11 ∧ (p20 ∨ p22) ∧ p32\PYG{n+nt}{\PYGZlt{}/if\PYGZgt{}}
  \PYG{n+nt}{\PYGZlt{}if}\PYG{n+nt}{\PYGZgt{}}h2 : p02\PYG{n+nt}{\PYGZlt{}/if\PYGZgt{}}
  \PYG{n+nt}{\PYGZlt{}then}\PYG{n+nt}{\PYGZgt{}}p00 ∧ p10 ∨ p01 ∧ p11 ... p31 ∧ p21\PYG{n+nt}{\PYGZlt{}/then\PYGZgt{}}
\PYG{n+nt}{\PYGZlt{}/state\PYGZgt{}}

\PYG{c+cm}{\PYGZlt{}!\PYGZhy{}\PYGZhy{} truncated \PYGZhy{}\PYGZhy{}\PYGZgt{}}

\PYG{n+nt}{\PYGZlt{}state} \PYG{n+na}{id=}\PYG{l+s}{\PYGZdq{}100\PYGZdq{}}\PYG{n+nt}{\PYGZgt{}}   \PYG{c+cm}{\PYGZlt{}!\PYGZhy{}\PYGZhy{} incorrect state number \PYGZhy{}\PYGZhy{}\PYGZgt{}}
  \PYG{n+nt}{\PYGZlt{}if}\PYG{n+nt}{\PYGZgt{}}h1 : p02 ∧ p11 ∧ (p20 ∨ p22) ∧ p32\PYG{n+nt}{\PYGZlt{}/if\PYGZgt{}}
  \PYG{n+nt}{\PYGZlt{}if}\PYG{n+nt}{\PYGZgt{}}h2 : p02\PYG{n+nt}{\PYGZlt{}/if\PYGZgt{}}
  \PYG{n+nt}{\PYGZlt{}if}\PYG{n+nt}{\PYGZgt{}}h3 : p11 ∧ (p20 ∨ p22) ∧ p32\PYG{n+nt}{\PYGZlt{}/if\PYGZgt{}}
  \PYG{n+nt}{\PYGZlt{}if}\PYG{n+nt}{\PYGZgt{}}h4 : p11\PYG{n+nt}{\PYGZlt{}/if\PYGZgt{}}
  \PYG{n+nt}{\PYGZlt{}if}\PYG{n+nt}{\PYGZgt{}}h5 : (p20 ∨ p22) ∧ p32\PYG{n+nt}{\PYGZlt{}/if\PYGZgt{}}
  \PYG{n+nt}{\PYGZlt{}if}\PYG{n+nt}{\PYGZgt{}}h7 : p32\PYG{n+nt}{\PYGZlt{}/if\PYGZgt{}}
  \PYG{n+nt}{\PYGZlt{}if}\PYG{n+nt}{\PYGZgt{}}h8 : p20\PYG{n+nt}{\PYGZlt{}/if\PYGZgt{}}
  \PYG{n+nt}{\PYGZlt{}then}\PYG{n+nt}{\PYGZgt{}}p00 ∧ p10 ∨ p01 ∧ p11 ... ∨ p31 ∧ p21\PYG{n+nt}{\PYGZlt{}/then\PYGZgt{}}
\PYG{n+nt}{\PYGZlt{}/state\PYGZgt{}}

\PYG{c+cm}{\PYGZlt{}!\PYGZhy{}\PYGZhy{} does not terminate \PYGZhy{}\PYGZhy{}\PYGZgt{}}

\end{Verbatim}

    \end{minipage}
\end{minipage}

\end{document}